\documentclass{article}

\usepackage{microtype}
\usepackage{graphicx}
\usepackage{subcaption}
\usepackage{booktabs} 

\usepackage{hyperref}

\usepackage[accepted]{icml2026}

\usepackage{amsmath}
\usepackage{amssymb}
\usepackage{mathtools}
\usepackage{amsthm}
\usepackage{mathabx}
\usepackage{bbm}
\usepackage{mathrsfs}
\usepackage{enumitem}

\usepackage{relsize}
\usepackage{exscale}

\usepackage{accents}

\usepackage{algorithm}

\usepackage[capitalize,noabbrev]{cleveref}

\theoremstyle{plain}
\newtheorem{theorem}{Theorem}[section]
\newtheorem{proposition}[theorem]{Proposition}
\newtheorem{lemma}[theorem]{Lemma}
\newtheorem{corollary}[theorem]{Corollary}
\theoremstyle{definition}

\newtheorem{assumption}[theorem]{Assumption}
\theoremstyle{remark}
\newtheorem{remark}[theorem]{Remark}

\usepackage[textsize=tiny]{todonotes}

\usepackage{caption}
\usepackage{graphicx}
\usepackage{subcaption}
\usepackage{multirow}

\usepackage{makecell}

\usepackage{graphicx}
\usepackage{mwe}

\usepackage{tikz}
\usetikzlibrary{shapes,arrows,positioning,fit,decorations.pathreplacing,calc}

\usepackage{blkarray}
\usepackage{bm}
\usepackage{todonotes}

\usepackage{abbreviation_definitions}

\definecolor{myblue}{RGB}{70, 160, 230}

\icmltitlerunning{Optimal Design for Multinomial Logit Model with
Applications to Best Assortment Identification}

\begin{document}

\twocolumn[
  \icmltitle{Optimal Design for Multinomial Logit Model
  with
  \\Applications to Best Assortment Identification}



  \icmlsetsymbol{equal}{*}

  \begin{icmlauthorlist}
    \icmlauthor{Joongkyu Lee}{sch}
    \icmlauthor{Min-hwan Oh}{sch}
  \end{icmlauthorlist}

    \icmlaffiliation{sch}{Seoul National University, Seoul, Korea}

    \icmlcorrespondingauthor{Min-hwan Oh}{minoh@snu.ac.kr}

  \icmlkeywords{Machine Learning, ICML}

  \vskip 0.3in
]



\printAffiliationsAndNotice{}  

\begin{abstract}
We study optimal experimental design for multinomial logit (MNL) bandits, where an agent repeatedly selects a subset of $K$ items from a ground set of size $N$ and observes single-choice feedback. Unlike linear or generalized linear bandits, MNL bandits have a \textit{combinatorial} action space, which makes classical optimal design approaches and naive optimization over all subsets computationally intractable. We propose a computationally efficient optimal design framework for MNL models that achieves both statistical efficiency and scalability through two complementary approaches: (i) an exact or certified-approximate reformulation of the design oracle as a $0$--$1$ mixed-integer linear program (MILP) with solver-certified early stopping, and (ii) a fully polynomial-time lifted design that replaces the nonlinear objective with a tractable surrogate. Using the Kiefer--Wolfowitz equivalence theorem, we establish near G-optimality guarantees and characterize the induced statistical--computational trade-offs. As an application, we develop a best assortment identification algorithm for MNL bandits with linear utilities and non-uniform revenues, and prove an instance-dependent sample complexity of $\tilde{\mathcal{O}}\big(\frac{d \log N}{\Delta^2}\big)$, where $d$ is the feature dimension, $N$ is the number of arms, and $\Delta$ is the minimum revenue gap. 
\end{abstract}

\section{Introduction}
\label{sec:intro}
We study optimal experimental design in multinomial logit (MNL) models~\citep{mcfadden1977modelling} with linear utilities.
The MNL model is a fundamental discrete choice model widely used in economics, marketing, and machine learning, where outcomes are generated from subset-dependent multinomial distributions parameterized by latent utilities.
For instance, in online advertising, a user selects (e.g., clicks) one item from a subset of \(K\) ads presented on a page.
Each ad can be represented by high-dimensional features, such as image or text embeddings produced by an off-the-shelf neural network.
This structured feedback setting naturally arises in large-scale applications, where the total number of available items \(N\) can be very large.

From an experimental design perspective, this setting poses significant challenges.
Unlike linear or generalized linear models—where the design space consists of individual arms—
the design space in MNL models is \emph{combinatorial}: each experiment corresponds to selecting a subset of arms, and the number of feasible experiments grows on the order of $N^{K}$.
Moreover, even for a fixed subset, the information induced by the MNL model is fundamentally different from that in linear or generalized linear models.
In particular, the Fisher information is matrix-valued, depends nonlinearly on the underlying parameter, and is determined by the entire subset through MNL choice probabilities, rather than decomposing additively across individual arms.
As a result, classical optimal design theory~\citep{kiefer1960equivalence, pukelsheim2006optimal}—which relies on additive, per-arm information contributions—does not directly apply.

In this paper, we develop an optimal design framework for MNL models with combinatorial experiments that addresses the structural and computational challenges described above.
Our approach builds on \textit{locally} optimal design principles to handle the inherent nonlinearity of the MNL model, and leads to a Frank--Wolfe algorithm~\citep{frank1956algorithm, dunn1978conditional} whose main computational challenge is a linear maximization oracle (LMO) over a \emph{combinatorial} collection of feasible subsets.
To address this bottleneck, we propose two complementary approaches:
(i) an \emph{exact} $0$--$1$ mixed-integer linear program (MILP) reformulation of the LMO, allowing \textit{early stopping} with guaranteed accuracy up to a user-specified tolerance, and
(ii) a \textit{polynomial-time} relaxation based on a Schur-complement lifting of the MNL Fisher information with a bounded approximation error.

As an application of our efficient optimal design framework, we study \textit{best assortment identification} in MNL bandits with linear utilities.
In this fixed-confidence setting, an agent sequentially offers assortments (i.e., subsets of arms) and observes a single MNL choice outcome, with the goal of identifying the unique revenue-maximizing assortment $S^\star$ with high probability while minimizing the total number of interactions.
Leveraging a $G$-optimal design computed at a nominal parameter, we propose an algorithm, \AlgName{}, that combines an initial exploration phase (which ensures the nominal parameter is uniformly close to the true parameter), design-based sampling, and a principled stopping rule.
We prove an instance-dependent sample complexity bound of
$\widetilde{\mathcal O}\!\left(d \log(N/\delta)\big(\Gap^{-2}+(\kappa\Gap)^{-1}\big)\right)$,
which in the small-gap regime simplifies to the leading-order guarantee
$\widetilde{\mathcal O}\!\big(\frac{d\log N}{\Gap^2}\big)$.
To the best of our knowledge, this yields the first sample complexity guarantee for best assortment identification in MNL bandits with linear utilities, while remaining computationally efficient via either MILP-based early stopping or the polynomial-time lifted relaxation.

Our main contributions are summarized as follows:
\begin{itemize}

    \item \textbf{Exact MILP reformulation:} 
    To address the primary computational bottleneck of the Frank--Wolfe (FW) algorithm—the linear maximization oracle (LMO) defined in~\eqref{eq:LMO}—we show that the optimization problem admits an \emph{exact} reformulation as a $0$--$1$ mixed-integer linear program (Theorem~\ref{thm:LMO_MILP}).
    This reformulation allows the LMO to be solved using off-the-shelf MILP solvers, which either return an exact solution or terminate early with a user-specified, solver-certified optimality gap.
    Moreover, we show that the FW algorithm retains its $\epsilon$-optimality guarantee even with certified inexact LMO solutions, at the cost of a small additional number of FW iterations (Proposition~\ref{prop:stopping_time_inexact_LMO_informal}).

    \item \textbf{Polynomial-time lifted surrogate:}
    As an alternative approach, we propose a novel lifted surrogate based on a Schur-complement construction.
    This relaxation yields a ratio-of-sums objective that can be optimized \emph{exactly} in polynomial time.
    The resulting design trades statistical efficiency for computational scalability, while incurring only a bounded and explicitly characterized approximation error (Theorem~\ref{thm:liftedFW}).

    \item \textbf{Best assortment identification in MNL bandits:}
    Leveraging the optimal design, we propose \AlgName{} for best assortment identification in MNL bandits, and establish an instance-dependent sample complexity of
    $\BigOTilde\!\big( \frac{d \log N}{\Gap^2} \big)$,
    where $d$ is the feature dimension, $N$ is the number of arms, and $\Gap$ is the minimum revenue gap (Theorem~\ref{thm:BSI_MNL}).
    To the best of our knowledge, this is the first sample complexity guarantee for best assortment identification in MNL bandits with linear utilities and non-uniform revenue parameters.
    Moreover, the algorithm remains computationally efficient, either via MILP-based early stopping with certified gaps or via a fully polynomial-time lifted relaxation.

\end{itemize}


\section{Related Work}
\label{sec:related}
\textbf{MNL model in Bandits.}\,\,
The MNL feedback model has been extensively studied in the bandit literature~\citep{agrawal2019mnl, agrawal2017thompson, ou2018multinomial, chen2020dynamic, oh2019thompson, oh2021multinomial, perivier2022dynamic, agrawal2023tractable, lee2024nearly}.
Compared to linear, generalized linear, and logistic bandits~\citep{abbasi2011improved, faury2020improved, abeille2021instance, faury2022jointly, lee2024unified}, MNL bandits present additional challenges due to their combinatorial action space and non-uniform revenue parameters.
Most prior work primarily focuses on regret minimization~\citep{chen2020dynamic, oh2019thompson, oh2021multinomial, perivier2022dynamic, lee2024nearly, lee2025improved}.
Although a limited number of studies address best-arm or best-assortment identification~\citep{rusmevichientong2010dynamic, saure2013optimal, yang2021fully, karpov2022instance}, they are largely restricted to context-free MNL settings.
Consequently, best assortment identification in MNL bandits with linear utilities remains underexplored.

\textbf{Pure Exploration in Linear Bandits.}\,\,
The pure exploration problem has been extensively investigated in linear bandits~\citep{soare2014best, degenne2020gamification, li2022instance, jedra2020optimal, fiez2019sequential, komiyama2022minimax, yang2022minimax}.
In this line of work, G-optimal design plays a central role in achieving near-optimal sample complexity.
These approaches aim to identify the best arm or to certify near-optimality with high confidence, yielding strong theoretical and practical guarantees.
However, existing design-based methods rely on enumerating individual arms and therefore do not extend to MNL bandits, where the action space is inherently combinatorial.
Consequently, prior work provides no sample-complexity guarantees 
for pure exploration in MNL bandits.

\section{Optimal Design for MNL Model}
\label{sec:opt_design_MNL}
\subsection{Notations.}
\label{subsec:notation}
For a positive integer $n$, we define $[n]$ as the set $\{1, 2, \ldots, n \}$. 
The $\ell_2$--norm of a vector $x$ is denoted by $\|x\|_2$.
For a positive semi-definite matrix $A$ and a vector $x$, we use $\|x \|_A$ to represent $\sqrt{x^\top A x}$.
For any two symmetric matrices $A$ and $B$ of the same dimensions, $A \succeq B$ indicates that $A-B$  is a positive semi-definite matrix.
Finally, we define $\mathcal{S}$ as the set of candidate assortments with a size constraint of at most $K$, i.e., $ \mathcal{S} = \{S \subseteq [N]: |S| \leq K \}$.

\subsection{MNL Model and Locally Optimal Design}
\label{subsec:mnl_local_opt}
\textbf{MNL model and loss function. }\,
We consider the problem of optimal experimental design under the multinomial logit (MNL) model~\citep{mcfadden1977modelling}.
Let $\Acal = \{\ab_i\}_{i=1}^N$ denote a finite set of $N$ arms, where each arm $i \in [N]$ is represented by a $d$-dimensional feature vector $\ab_i \in \RR^d$.
Let $\Scal$ denote the collection of feasible subsets of arms with maximum size $K$.
At each step $t$, the learner selects a subset $S_t \in \Scal$ with $2 \le |S_t| \le K$.\footnote{For a singleton subset, the choice probability is identically $1$; hence, we exclude this trivial case in this section.}
Conditional on $S_t$, the probability of choosing arm $i \in S_t$ follows the MNL model:
\begin{align}
    p(i | S_t, \thetab^\star)
    := \frac{\exp\!\big( \ab_i^\top \thetab^\star \big)}{\sum_{j \in S_t} \exp\!\big( \ab_j^\top \thetab^\star \big)},
    \label{eq:MNL_prob}
\end{align}
where $\thetab^\star \in \RR^d$ is an \emph{unknown} underlying parameter.
Note that when $K=2$, the MNL model reduces to the Bradley--Terry--Luce (BTL) model~\cite{bradley1952rank}.

At step $t$, let $\yb_t := (y_{ti})_{i \in S_t} \in \{0,1\}^{|S_t|}$ denote the one-hot choice feedback vector, where $y_{ti}=1$ indicates that arm $i \in S_t$ is selected. Then
$
\Pr(y_{ti}=1 | S_t)=p(i | S_t,\thetab^\star),$ 
where $ i \in S_t,
$
and
$
\sum_{i \in S_t} y_{ti}=1.
$

Now, we define the MNL loss function at step $t$ as:
\begin{align*}
    \ell_t(\thetab) := - \sum_{i \in S_t} y_{ti} \log p(i | S_t, \thetab).
\end{align*}
Then, the Fisher information matrix~\citep{fisher1925theory} is defined as the expected Hessian of the negative log-likelihood:
\begin{align*}
    \Ib_{\thetab}(S_t)
    &:= \EE_{\yb_t}\!\left[ \nabla^2 \ell_t(\thetab) \,\middle|\, S_t \right]
    = \nabla^2 \ell_t(\thetab)
    \\
    &= \sum_{i \in S_t} p(i | S_t, \thetab)\,
        \big( \ab_i - \bar{\ab}_{\thetab}(S_t) \big)
        \big( \ab_i - \bar{\ab}_{\thetab}(S_t) \big)^\top\!,
\end{align*}
where $\bar{\ab}_{\thetab}(S_t)
    \!:=\! \sum_{j \in S_t} p(j | S_t, \thetab)\, \ab_j$
denotes the mean feature vector under the MNL choice probabilities.
The second equality holds because, under the MNL model, the Hessian of the negative log-likelihood,
$\nabla^2 \ell_t(\thetab)$, does not depend on the \textit{realized} choice feedback.

Throughout the paper, we assume that the pairwise feature differences span the
$d$-dimensional space:
\begin{assumption}[Spanning condition]
\label{assm:span}
We assume that $\operatorname{span}\big\{\ab_i - \ab_j : i,j \in [N],\, i \neq j\big\} = \RR^{d}.$
\end{assumption}
This assumption holds \textit{without loss of generality}: if the span were a strict subspace of $\RR^d$,
we could restrict $\thetab$ to the identifiable subspace (and redefine $d$ accordingly),
so that Assumption~\ref{assm:span} holds throughout, without changing the MNL probabilities (see Appendix~\ref{app:wlog_span} for more details).

\textbf{Locally G-optimal design. }\,
We focus on solving the \textit{G-optimal design} problem~\citep{kiefer1960equivalence, pukelsheim2006optimal} under the MNL model, in which the set of feasible assortments is combinatorial and has cardinality on the order of $\BigO(N^K)$.
Let $\pi: \Scal \rightarrow [0,1]$ denote a probability distribution over feasible assortments. 
Given any fixed parameter $\thetab_0 \in \RR^d$—hence the term \emph{locally} optimal design—our goal is to find 
$\pi^\star_{\thetab_0} \in \argmin_{\pi} \, g_{\thetab_0}(\pi)$, where
\begin{align*}
    g_{\thetab_0}(\pi) = \max_{S \in \Scal }\,
    \operatorname{tr} \! \left(
        \Mb_{\thetab_0}(\pi)^{-1} \Ib_{\thetab_0} (S)
    \right)\!.
    \numberthis \label{eq:g_opt}
\end{align*}
Here, $\Mb_{\thetab_0}(\pi) := \EE_{S \sim \pi}[\Ib_{\thetab_0} (S)]$ denotes the expected Fisher information matrix induced by the design $\pi$.
Note that, since the objective in~\eqref{eq:g_opt} does not depend on the realized feedback, algorithms based on the optimal design $\pi^\star_{\thetab_0}$ are non-adaptive.

It is well known that solutions to the optimal design problem are characterized by the Kiefer–Wolfowitz (KW) equivalence theorem~\citep{kiefer1960equivalence, lattimore2020bandit}.
However, whereas the classical KW theorem is formulated in terms of feature vectors, we consider a matrix-valued information form $\Ib_{\thetab_0}(S) \in \RR^{d \times d}$.
Therefore, we establish the KW equivalence theorem for the MNL model, building on the matrix-valued KW theorem of~\citealt[Theorem~1]{mukherjee2024optimal}.
\begin{proposition}[KW equivalence theorem for MNL]
\label{prop:KW_MNL}
    Under Assumption~\ref{assm:span},
    the following statements are equivalent:
    \begin{enumerate}[label=(\alph*)]
        \item $\pi^\star_{\thetab_0}$ is a minimizer of $g_{\thetab_0}(\pi)$, defined in~\eqref{eq:g_opt}.
    
        \item $\pi^\star_{\thetab_0}$ is a maximizer of $f_{\thetab_0}(\pi) = \log \operatorname{det}\!\left(\Mb_{\thetab_0}(\pi)\right)$.

        \item  $g_{\thetab_0}(\pi^\star_{\thetab_0}) = d$.
    \end{enumerate}
    Furthermore, there exists a minimizer $\pi_{\thetab_0}^\star$ of $g_{\thetab_0}(\pi)$ such that $|\operatorname{supp}(\pi_{\thetab_0}^\star)|\leq d(d+1)/2$.
\end{proposition}
Proposition~\ref{prop:KW_MNL} shows the equivalence between G-optimality and D-optimality under the MNL model, and further guarantees the existence of an optimal design with finite support.
The proof of Proposition~\ref{prop:KW_MNL} is provided in Appendix~\ref{subsec:proof_prop:KW_MNL}.
\subsection{Frank-Wolfe Algorithm for MNL model}
\label{subsec:opt_MNL}
To solve~\eqref{eq:g_opt},
we employ a Frank--Wolfe algorithm~\citep{frank1956algorithm,nocedal2006numerical,jaggi2013revisiting},
tailored to the MNL model
(see Algorithm~\ref{alg:FW_MNL}).
Rather than directly minimizing the G-optimality objective $g_{\thetab_0}(\pi)$,
we exploit its equivalent D-optimal formulation at the nominal parameter $\thetab_0$,
and instead maximize
$f_{\thetab_0}(\pi) := \log\det\!\big(\Mb_{\thetab_0}(\pi)\big)$ over $\pi\in\Delta(\Scal)$ (see Proposition~\ref{prop:KW_MNL}).
\begin{algorithm}[t]
   \caption{Frank–Wolfe Algorithm for MNL at $\thetab_0$}
   \label{alg:FW_MNL}
    \begin{algorithmic}[1]
       \STATE {\bfseries Input:}
       nominal parameter $\thetab_0$,
       initial distribution $\pi_0\in\Delta(\Scal)$,
       FW precision $\varepsilon$.
       \STATE {\bfseries Initialize:} 
       $\Mb_0 = \Mb_{\thetab_0}(\pi_0)$, 
       $m=0$.
       \WHILE{$g_{\thetab_0}(\pi_m) > (1+\varepsilon)d $}
            \STATE $S_m \leftarrow \argmax_{S\in\Scal}\ \mathrm{tr} \big(\Mb_m^{-1}\Ib_{\thetab_0}(S) 
            \big)$. 

            \STATE $\gamma_m \leftarrow \argmax_{\gamma \in [0,1]}
            f_{\thetab_0}
            \big(
                    (1-\gamma) \pi_m + \gamma \mathbbm{1}_{S_m}
            \big)
            $.

            \STATE $\pi_{m+1} = (1-\gamma_m)\pi_m  + \gamma_m \mathbbm{1}_{S_m}$.

            \STATE Update $\Mb_{m+1}$ and $m \leftarrow m+1$.
            
       \ENDWHILE

       \STATE Output the estimated policy $\widehat{\pi}_{\thetab_0} = \pi_m$.
    \end{algorithmic}
\end{algorithm}

In Algorithm~\ref{alg:FW_MNL}, we fix a nominal parameter $\thetab_0$ and 
 initialize $\pi_0$ such that $\Mb_{\thetab_0}(\pi_0)\!\succ\! 0$.
At iteration $m$, let $\pi_m$ denote the current design and define $\Mb_m:=\Mb_{\thetab_0}(\pi_m)$.
Using the matrix derivative $\nabla_{\Mb} \log\det(\Mb)=\Mb^{-1}$ together with the chain rule, the partial derivative of $ f_{\thetab_0}(\pi_m)$
with respect to the coordinate $\pi(S)$ is given by
\[
\big[\nabla_{\pi} f_{\thetab_0}(\pi_m)\big]_S
= \left\langle \Mb_m^{-1},\, \Ib_{\thetab_0}(S)\right\rangle
= \mathrm{tr}\big(\Mb_m^{-1} \Ib_{\thetab_0}(S)   \big).
\]
The \textit{linear maximization oracle} (LMO) therefore, selects an extreme point
$\mathbbm{1}_{S_m}\in\Delta(\Scal)$, where $\mathbbm{1}_{S}$ denotes the Dirac distribution at $S$,
corresponding to
\begin{equation}
\label{eq:LMO}
    S_m\in
    \argmax_{S\in\Scal}\ \mathrm{tr} \!\left(\Mb_m^{-1} \Ib_{\thetab_0}(S) \right)\!.
\end{equation}
We then update the design by taking a convex combination of the current iterate
$\pi_m$ and the extreme point $\mathbbm{1}_{S_m}$, i.e.,
$\pi_{m+1} = (1-\gamma_m)\pi_m + \gamma_m \mathbbm{1}_{S_m}$, 
where the stepsize $\gamma_m\in[0,1]$ is selected via a line search that maximizes
$f_{\thetab_0}\!\left((1-\gamma)\pi_m + \gamma \mathbbm{1}_{S_m}\right)$.
It is known that this procedure converges to a maximizer of $f_{\thetab_0}(\pi)$~\citep{zhao2023analysis}.\footnote{Since inexact line search is standard in Frank--Wolfe methods and the resulting error is typically negligible in practice, we omitted this approximation error from the presentation for simplicity.}
We refer to the returned policy $\widehat{\pi}_{\thetab_0}$ as a $(1+\varepsilon)$-approximation of the optimal policy $\pi^\star_{\thetab_0}$.

\textbf{Computational challenge in LMO. }\,
However, a naive implementation of Algorithm~\ref{alg:FW_MNL} is computationally infeasible, as the LMO in~\eqref{eq:LMO} entails exhaustive enumeration over the set of all assortments, whose cardinality is
$|\Scal| = \BigO( N^K )$.

In the following two subsections, we address this computational bottleneck in two complementary ways.
\textbf{First}, we reformulate~\eqref{eq:LMO} as a $0$–$1$ \textit{mixed-integer linear program (MILP)}, which enables the LMO to be solved using standard MILP solvers with early stopping and solver-certified optimality gaps (Subsection~\ref{subsec:lmo_milp}).
\textbf{Second}, we relax the original problem in~\eqref{eq:LMO} via a \textit{Schur-complement lifting}, resulting in a polynomial-time surrogate LMO at the cost of approximation error; importantly, the resulting worst-case error admits a provable upper bound (Subsection~\ref{subsec:lifted_lmo_mnl}).

\subsection{LMO via $0$--$1$ MILP with Certified Optimality Gaps}
\label{subsec:lmo_milp}

In this subsection, we show that the linear minimization oracle (LMO) in~\eqref{eq:LMO} can be reformulated as a $0$–$1$ \textit{mixed-integer linear program (MILP)}.
\begin{theorem}[MILP reformulation of the LMO]
\label{thm:LMO_MILP}
The optimization problem in~\eqref{eq:LMO} admits an \emph{exact} reformulation as a compact $0$--$1$ mixed-integer linear program (MILP) whose number of variables and linear constraints is polynomial in $N$. As a consequence, the LMO can be solved using standard off-the-shelf MILP solvers either to global optimality or, via early termination, with a solver-certified optimality gap at most $\ErrorLMO$.
\end{theorem}
The proof of Theorem~\ref{thm:LMO_MILP} is deferred to
Appendix~\ref{app_subsec:proof_thm:LMO_MILP}.

\textbf{Early stopping with $\ErrorLMO$-optimality. }\,
Although the resulting MILP is NP-hard in general, modern MILP solvers based on
branch-and-bound maintain both (i) the best feasible solution found so far and
(ii) a provable bound on how good any solution can possibly be.
From these two quantities, the solver reports a \emph{certified optimality gap} at runtime,
which upper-bounds the remaining suboptimality of the current incumbent solution
(the best feasible solution).
Therefore, the LMO computation can be safely terminated once this solver-certified gap
drops below a user-specified tolerance $\ErrorLMO$.
This yields an \textit{instance-dependent} and \textit{verifiable} \textit{early-stopping} rule:
the returned LMO solution comes with an explicit certificate of near-optimality,
providing a principled trade-off between computational cost and user-specified
optimization accuracy.

\textbf{Frank--Wolfe stopping under $\ErrorLMO$-approximate LMO. }\,
When the MILP-based LMO is solved approximately, the solver returns a candidate set $\hat S_m$
together with a certified optimality gap.
To account for this inexactness, we tighten the Frank--Wolfe stopping rule.
Specifically, if the solver is terminated once the certified gap is at most $\ErrorLMO$, we define
\[
    \tilde{\epsilon} \;\coloneqq\; \epsilon - \ErrorLMO/d,
\]
and require $0\le\ErrorLMO\le\epsilon d/2$, so that $\tilde\epsilon\ge\epsilon/2>0$.
The FW algorithm is then run with tolerance $\tilde{\epsilon}$.
See Algorithm~\ref{alg:FW_MNL_approx} for a complete description.

\begin{proposition}[FW stopping guarantee under an inexact LMO]
\label{prop:stopping_time_inexact_LMO_informal}
Let $\tilde{\epsilon}\in(0,1]$ and $\ErrorLMO\le d\tilde\epsilon$.
Then the Frank--Wolfe algorithm with tolerance $\tilde{\epsilon}$ terminates after at most
$ \tau = \tilde{\mathcal O}\!\left( d/\tilde{\epsilon} \right)$
iterations, and returns a design $\widehat{\pi}_{\thetab_0}$ that is $\epsilon$-accurate.
\end{proposition}
Proposition~\ref{prop:stopping_time_inexact_LMO_informal} shows that, even with an $\ErrorLMO$-approximate LMO, the FW algorithm terminates in at most $\tilde{\mathcal O}(d/\tilde\epsilon)$ iterations and returns an $\epsilon$-accurate design.
The proof of Proposition~\ref{prop:stopping_time_inexact_LMO_informal} is deferred to
Appendix~\ref{app_subsec:proof_prop_stopping}.

\textbf{Comparison with \citealt{thekumparampil2024comparing}. }\,
Recently, \citet{thekumparampil2024comparing} proposed a randomized FW algorithm
(\texttt{DopeWolfe}) that approximates the LMO by sampling a subset of candidate
assortments $\Rcal_m \subseteq \Scal$ at each iteration $m$.
However, to guarantee a small Frank--Wolfe approximation error, the required
sample size still scales as $|\Rcal_m| = \BigO(N^K)$
(see Corollary~4 in~\citealp{thekumparampil2024comparing}),
and thus the combinatorial nature of the LMO is not avoided.
Moreover, unlike our MILP-based approach, \texttt{DopeWolfe} requires a \emph{pre-specified} iteration budget and does not admit a solver-certified early-stopping rule.

\subsection{Polynomial-Time Surrogate LMO via Schur-Complement Lifting}
\label{subsec:lifted_lmo_mnl}

In this subsection, we develop a lifting-based reformulation of the MNL Fisher information that removes the set-dependent centering term $\bar{\ab}_S(\thetab)$ and enables an efficient surrogate linear maximization oracle (LMO).
Throughout, we fix a nominal parameter $\thetab_0\in\RR^d$ and write $p_{\thetab_0}\!(i|S):=p(i|S,\thetab_0)$ for brevity.

\textbf{Schur-complement lifting. }\,
For each arm $i\in[N]$, define the lifted feature vector $\tilde{\ab}_i \!:= \!(\ab_i^\top\!,1)^\top\!\in\!\RR^{d+1}$ and, f
or each assortment $S\in\Scal$, define the \textit{lifted} (uncentered) second-moment matrix
\begin{align*}
    \widetilde{\Ib}_{\thetab_0}(S)
    := \sum_{i\in S} p_{\thetab_0}(i|S)\, \tilde{\ab}_i \tilde{\ab}_i^\top
    \in \RR^{(d+1)\times(d+1)}.
    \numberthis
    \label{eq:lifted_I}
\end{align*}
Let
$
\bar{\Ab}_{\thetab_0}(S) := \sum_{i\in S} p_{\thetab_0}(i|S)\,\ab_i\ab_i^\top
$.
Then $\widetilde{\Ib}_{\thetab_0}(S)$ defined in~\eqref{eq:lifted_I} admits the block decomposition
\begin{align}
    \widetilde{\Ib}_{\thetab_0}(S)
    =
    \begin{bmatrix}
        \bar{\Ab}_{\thetab_0}(S) & \bar{\ab}_{\thetab_0}(S)\\
        \bar{\ab}_{\thetab_0}(S)^\top & 1
    \end{bmatrix},
    \label{eq:lifted_I_block_mnl}
\end{align}
and the (centered) MNL Fisher information satisfies
    $\Ib_{\thetab_0}(S)
    \!=\!
    \bar{\Ab}_{\thetab_0}(S)\!-\!\bar{\ab}_{\thetab_0}(S) \bar{\ab}_{\thetab_0}(S)^\top$\!,
which is precisely the Schur complement of the bottom-right block of~\eqref{eq:lifted_I_block_mnl}.\footnote{For a block matrix
$\left(\begin{smallmatrix} A & B \\ B^\top & C \end{smallmatrix}\right)$
with $C$ invertible,
the \emph{Schur complement} of the block $C$ is defined as
$A - B C^{-1} B^\top$.}
In particular, the dependence of $\Ib_{\thetab_0}(S)$ on the centering term $\bar{\ab}_{\thetab_0}(S)$ is captured linearly in the lifted matrix $\widetilde{\Ib}_{\thetab_0}(S)$.

Given a design $\pi\in\Delta(\Scal)$, we define the lifted design matrix
\[
\widetilde{\Mb}_{\thetab_0}(\pi)
:= \EE_{S \sim \pi} \left[ \widetilde{\Ib}_{\thetab_0}(S) \right]
=
\begin{bmatrix}
\bar{\Bb}_{\thetab_0}(\pi) & \bar{\bb}_{\thetab_0}(\pi)\\
\bar{\bb}_{\thetab_0}(\pi)^\top & 1
\end{bmatrix},
\]
where
$\bar{\Bb}_{\thetab_0}\!(\pi)\!:=\!\EE_{\pi}\big[\bar{\Ab}_{\thetab_0}(S)\big]$ and
$\bar{\bb}_{\thetab_0}\!(\pi)\!:=\!\EE_{\pi}\big[\bar{\ab}_{\thetab_0}(S)\big].$
Moreover, we define its Schur-complement matrix
\[
\widehat{\Mb}_{\thetab_0}(\pi)
:=
\bar{\Bb}_{\thetab_0}(\pi)-\bar{\bb}_{\thetab_0}(\pi)\bar{\bb}_{\thetab_0}(\pi)^\top
\in\RR^{d\times d},
\]
and the associated \emph{design-mismatch} matrix
\[\Delta_{\thetab_0}(\pi)
:=
\widehat{\Mb}_{\thetab_0}(\pi)-\Mb_{\thetab_0}(\pi).
\]
We call $\Delta_{\thetab_0}(\pi)$ a \emph{mismatch} matrix because $\widehat{\Mb}_{\thetab_0}(\pi)$ is the Schur complement of the lifted design matrix $\widetilde{\Mb}_{\thetab_0}(\pi)$, whereas $\Mb_{\thetab_0}(\pi)$ is the \emph{true} (centered) Fisher design matrix in the original space. 
Thus, $\Delta_{\thetab_0}(\pi)$ quantifies the discrepancy between the Schur-complement surrogate induced by lifting and the original design matrix.

\textbf{Lifted G-optimal design. }\,
The lifting in~\eqref{eq:lifted_I} naturally induces a \emph{lifted} G-optimal design criterion.
For any design $\pi\in\Delta(\Scal)$, define
\begin{equation*}
\widetilde g_{\thetab_0}(\pi)
\;:=\;
\max_{S\in\Scal}\ \tr\!\Big(\widetilde{\Mb}_{\thetab_0}(\pi)^{-1}\,\widetilde{\Ib}_{\thetab_0}(S)\Big).
\end{equation*}
We then define the \emph{lifted optimal design} as any minimizer
$\widetilde{\pi}^\star_{\thetab_0}
\in \argmin_{\pi\in\Delta(\Scal)} \widetilde g_{\thetab_0}(\pi).$
We solve this lifted design problem
by running Frank--Wolfe with $\widetilde g_{\thetab_0}$ as the objective. 
See Algorithm~\ref{alg:FW_MNL_lifted} for a full description.

\textbf{Polynomial-time lifted LMO. }\,
At iteration $m$ of Algorithm~\ref{alg:FW_MNL_lifted}, given the current design $\pi_m$ and
$\widetilde{\Mb}_m:=\widetilde{\Mb}_{\thetab_0}(\pi_m)$, we compute the next FW atom by solving
\begin{equation}
\label{eq:LMO_lifted_theta}
S_m\in
\argmax_{S\in\Scal}\ 
\tr \!\left(\widetilde{\Mb}_m^{-1}\, \widetilde{\Ib}_{\thetab_0}(S)\right).
\end{equation}
This lifted oracle admits a ratio-of-sums representation and can be solved in polynomial time.
Indeed, letting $s_i := \tilde{\ab}_i^\top \widetilde{\Mb}_m^{-1}\tilde{\ab}_i$
and $w_i := \exp(\ab_i^\top\thetab_0)$, for every $S\in\Scal$ we have
\[
    \tr \!\left(\widetilde{\Mb}_m^{-1}\widetilde{\Ib}_{\thetab_0}(S)\right)
    = \sum_{i\in S} p_{\thetab_0}(i\mid S)\, s_i
    = \frac{\sum_{i \in S} w_i\, s_i}{\sum_{j \in S} w_j}.
\]
Consequently,~\eqref{eq:LMO_lifted_theta} is a classical MNL ratio-of-sums assortment optimization
problem and can be reduced to a linear program (LP) and solved in time $\mathcal{O}(\mathrm{poly}(N))$
(e.g.,~\citealp{rusmevichientong2010dynamic, davis2014assortment}).

Let $\pi^\star_{\thetab_0} \in\argmin_{\pi\in\Delta(\Scal)} g_{\thetab_0}(\pi)$ be an original
G-optimal design.
Under the Kiefer--Wolfowitz equivalence (Proposition~\ref{prop:KW_MNL}), $g_{\thetab_0}(\pi^\star(\thetab_0))=d$.
Similarly, under the lifted criterion, one typically has
$ g_{\thetab_0}( \widetilde{\pi}^\star_{\thetab_0} )=d+1$ (reflecting the lifted dimension).

\begin{theorem}[Lifted FW guarantee]
\label{thm:liftedFW}
Fix $\thetab_0$ and run Frank--Wolfe (Algorithm~\ref{alg:FW_MNL}) to minimize
$\widetilde g_{\thetab_0}(\pi)$ using the polynomial-time lifted LMO~\eqref{eq:LMO_lifted_theta}.
Let $\hat\pi$ be the design returned at the stopping time
$\widetilde g_{\thetab_0}(\hat\pi)\ \le\ (1+\epsilon)(d+1)$, and define 
$\ErrorLift
:=
\inf\big\{\varepsilon\ge 0:\ \Delta_{\thetab_0}(\hat\pi)\preceq \varepsilon\,\Mb_{\thetab_0}(\hat\pi)\big\}$.
Then
\[
g_{\thetab_0}(\hat\pi)
\ \le\  2(1+\ErrorLift)(1+\epsilon)d.
\]
\end{theorem}
\textbf{Computational--statistical trade-off. }\,
Using the lifted LMO yields a polynomial-time Frank--Wolfe update.
The trade-off is a potential mismatch between the lifted surrogate and the true Fisher information, quantified by $\ErrorLift$.
Accordingly, the $G$-optimality guarantee degrades by at most a factor $(1+\ErrorLift)$.
A crude upper bound is $\BigO\big(1/\lambda_{\min}( \Mb_{\thetab_0}(\widetilde{\pi}^\star_{\thetab_0})) \big)$ (Proposition~\ref{prop:how_large_ErrorLift}).
In the standard MNL bandit setting with an outside option (see Equation~\eqref{eq:MNL_prob_bandit}),
we also obtain the bound 
$\ErrorLift \!=\! \BigO\big(
Ke^B \!\wedge\!
1/\lambda_{\min}( \Mb_{\thetab_0}(\widetilde{\pi}^\star_{\thetab_0}))
\big),$ 
when $\|\thetab_0\|\le B$ (Proposition~\ref{prop:outside_option_bound_eps_by_ratio}).
The proof of Theorem~\ref{thm:liftedFW} is provided in Appendix~\ref{app_subsec:proof_thm:liftedFW}.

\section{Best Assortment Identification in MNL Bandits with Linear Utilities}
\label{sec:BAI_MNL}
In this section, we consider \textit{\underline{b}est
a\underline{s}sortment \underline{i}dentification (BSI)}
in MNL bandits, where the abbreviation BSI is used throughout
to avoid confusion with best arm identification (BAI).
Following the previous literature~\citep{agrawal2019mnl,
agrawal2017thompson, ou2018multinomial, chen2020dynamic,
oh2019thompson, oh2021multinomial, perivier2022dynamic,
agrawal2023tractable, lee2024nearly, lee2025improved},
we adopt the MNL choice model with an \textit{outside option},
which represents the event that no item is selected from the
offered assortment $S$.
The choice probabilities are given by
\begin{align}
    p(i\mid S,\thetab^\star)
    :=\frac{\exp(\ab_i^\top\thetab^\star)}
    {1+\sum_{j\in S}\exp(\ab_j^\top\thetab^\star)},
    \qquad i\in S,
    \label{eq:MNL_prob_bandit}
\end{align}
and
$p(0\mid S,\thetab^\star)
:=1-\sum_{i\in S}p(i\mid S,\thetab^\star)$,
where $0$ denotes the outside option with
$\ab_0=0$ and $r_0=0$.

Throughout this section, we consider all assortments
$S\subseteq[N]$ with $|S|\le K$, including singleton
and empty assortments.
Unlike the MNL model without an outside option,
singleton assortments can provide nonzero Fisher information.
The corresponding negative log-likelihood is
\[
    \ell_s(\thetab)
    =\log\left(1+\sum_{j\in S_s}
    e^{\ab_j^\top\thetab}\right)
    -\ab_{i_s}^\top\thetab,
\]
and the Fisher information is the covariance of the
feature vectors over $S_s\cup\{0\}$.
Since the outside option provides the identifiable contrasts
$\ab_i-\ab_0=\ab_i$, we assume without loss of generality
that $\operatorname{span}\{\ab_i:i\in[N]\}=\RR^d$.

We define the expected revenue of an assortment $S$ as:
\[
R(S, \thetab)
=\sum_{i\in S} p(i| S, \thetab) r_i
= \frac{\sum_{i \in S} \exp(\ab_i^\top \thetab) r_i}{1+\sum_{j\in S}\exp(\ab_j^\top \thetab)}
,
\]
where $r_i\in[0,1]$ is a given revenue parameter.
Let $S^\star\in\argmax_{S \in \Scal}R(S,\thetab^\star)$ denote the optimal assortment and assume the maximizer is unique.
Given a failure probability $\delta>0$, 
our goal is to interact with the environment until a (possibly random) stopping time $\tau$ and return an assortment $\hat S$ such that
\[
    \PP(\hat{S} = S^\star \wedge \tau < \infty ) \geq 1-\delta.
\]
This setting is known as the fixed-confidence setting~\citep{garivier2016optimal},
where the objective is to identify the optimal assortment with as few samples as possible.

We work under the standard boundedness assumption.
\begin{assumption}
\label{assum:bounded_assumption}
For all $i \in [N]$, $\|\ab_i\|_2 \!\le\! 1$, and $\|\thetab^\star\|_2 \! \le\! \BoundParam$.
\end{assumption}
Moreover, following the previous MNL bandit literature~\citep{oh2021multinomial, perivier2022dynamic, lee2024nearly}, we  introduce the problem-dependent constant:
\begin{equation} \label{eq:kappa}
    \kappa
    := \min_{\varnothing\ne S \in \Scal}
        \min_{i \in S}
       p(i |S, \thetab^\star)\, p(0 |S, \thetab^\star),
\end{equation}
which measures the degree of deviation from the linear model.
Since $1/\kappa$ can grow exponentially in $B$,
it is crucial to derive bounds that do not depend on $1/\kappa$.

\textbf{Confidence bound. }\,
Let $\Dcal_t\! =\!  \{ (i_s, S_s) \}_{s=1}^t$ denote the data collected up to $t$.
We define the cumulative loss as:
\[
    \Lcal_{\Dcal_t}(\thetab)
    :=\sum_{s=1}^{t} \ell_s(\thetab) + \frac{\lambda}{2} \|\thetab \|_2^2,
    \quad
    \text{where }\,
    \lambda>0.
\]
We also define Hessian matrix at $\thetab$ as
    $\Hb_t(\thetab) 
    = \Hb_{\Dcal_t}(\thetab)
    := \sum_{s=1}^{t} \nabla^2 \ell_s(\thetab) + \lambda \Ib_d.$
The following lemma provides a confidence bound of order $\sqrt{\log N}$, independent of $d$.
\begin{lemma}[Fixed-design confidence bound]
    \label{lemma:estimation_error_main}
    Given $\Dcal=\{(i_s,S_s)\}_{s=1}^T$, let
    $\widehat{\thetab}_{\Dcal}$ be the unconstrained regularized MLE.
    Conditional on the assortments, assume that the choices are
    independent MNL observations and that
    Assumption~\ref{assum:bounded_assumption} holds.
    Write $L_\delta:=\log(4(N+1)^2/\delta)$.
    If
    \begin{equation}
        64\max_{i\in[N]}\|\ab_i\|_{\Hb_{\Dcal}(\thetab^\star)^{-1}}
        \le\frac1{\sqrt{d\log9+L_\delta+\lambda B^2}},
        \label{eq:warmup_condition_main}
    \end{equation}
    then, with probability at least $1-\delta$, for every $i\in[N]$,
    \[
        |\ab_i^\top(\widehat{\thetab}_{\Dcal}-\thetab^\star)|
        \le\beta(\delta)\|\ab_i\|_{\Hb_{\Dcal}(\thetab^\star)^{-1}},
    \]
    where $\beta(\delta):=36\sqrt{L_\delta}+64\sqrt\lambda B$.
\end{lemma}
The corresponding pairwise confidence bound and its proof are given
in Appendix~\ref{app:pairwise_confidence}.
This fixed-design argument is related to the confidence analysis
of~\citet{han2026improved}; the finite-direction form needed here
is proved explicitly.

\begin{algorithm*}[t]
   \caption{\AlgName (\textbf{B}est A\textbf{s}sortment \textbf{I}denficiation in \textbf{MNL} bandits)}
   \label{alg:MNL_BAI}
    \begin{algorithmic}[1]
       \STATE {\bfseries Input:}
       failure level $\delta$,
       warm-up threshold $\WarmupThreshold$,
       regularization $\lambda$,
       problem-dependent constant $\kappa$,
       precision $\varepsilon$ and $\ErrorLMO$.
       \STATE {\bfseries Initialize:} 
       $t=1$,
       $\Dcal_w = \Dcal_w'= \emptyset$,
       $\Vb_0 = \lambda \Ib_d$.
       \WHILE{$ \max_{i \in [N]} \|\ab_i \|_{\Vb_{t-1}^{-1}} > \WarmupThreshold$}
            \STATE Select $S_t$ using only the features and independent randomization,
            containing an item $i$ with $\|\ab_i\|_{\Vb_{t-1}^{-1}}>\WarmupThreshold$.
            \STATE Offer $S_t$ twice and observe two independent choice feedbacks $i_t, i_t'$.
            \STATE Update $\Dcal_w \leftarrow \Dcal_w \cup \{(S_t, i_t ) \}\,$,
            $\Dcal_w' \leftarrow \Dcal_w' \cup \{(S_t, i_t' ) \}\,$,
              $\Vb_{t} \leftarrow \Vb_{t-1} + \sum_{i \in S_t} \ab_i \ab_i^\top$, and $t \leftarrow t+1$.
       \ENDWHILE
        \STATE Compute $\thetab_0 \leftarrow \argmin_{\thetab \in \RR^d} \Lcal_{\Dcal_w'}(\thetab)$
        and 
        $\widehat{\thetab}_{t-1} \leftarrow \argmin_{\thetab \in \RR^d} \Lcal_{\Dcal_w}(\thetab)$.

        \STATE Construct a locally optimal design $\widehat{\pi}_{\thetab_0}$ at $\thetab_0$ by running
        Algorithm~\ref{alg:FW_MNL_approx} (or Algorithm~\ref{alg:FW_MNL_lifted}).
        
        \STATE Set
        $\Dcal_{t-1} \leftarrow \Dcal_w$ and
        initialize $\Hb_{t-1} (\thetab_0) \leftarrow \Hb_{\Dcal_w}\! (\thetab_0) $.
        
        \STATE Set $t\leftarrow t-1$ and $\WarmupRounds\leftarrow t$.
        \WHILE{either \eqref{eq:confidence_check} or \eqref{eq:stop_rule} is not true at $t$}
            \STATE Sample $S_{t+1}\sim\widehat{\pi}_{\thetab_0}$ and observe choice $i_{t+1}$.
            \STATE Update $\Dcal_{t+1}\leftarrow\Dcal_t\cup\{(S_{t+1},i_{t+1})\}$
            and $\Hb_{t+1}(\thetab_0)\leftarrow\Hb_t(\thetab_0)+\nabla^2\ell_{t+1}(\thetab_0)$.
            \STATE Compute $\widehat{\thetab}_{t+1}\leftarrow
            \argmin_{\thetab\in\RR^d}\Lcal_{\Dcal_{t+1}}(\thetab)$,
            and set $t\leftarrow t+1$.
        \ENDWHILE

        \STATE {\bfseries Return:} $\BestS_t$
    \end{algorithmic}
\end{algorithm*}
\subsection{Algorithm: \AlgName}
\label{subsec:algorithm}
We now describe our algorithm for BSI (Algorithm~\ref{alg:MNL_BAI}).

\textbf{Step 1: Initial exploration. }\,
In Algorithm~\ref{alg:MNL_BAI}, we begin with a pure exploration phase that continues until the burn-in condition
of Lemma~\ref{lemma:estimation_error_main} is met.
The purpose of this phase is to ensure that the uncertainty term is uniformly controlled over all arms.
The warm-up queries and stopping time depend only on the features and independent randomization, not on either response stream.
Since the true parameter $\thetab^\star$ is unknown, we cannot directly verify the condition
in~\eqref{eq:warmup_condition_main}.
Instead, we use the fact that
$\Hb_t(\thetab^\star) \succeq \kappa\,\Vb_t$, where
 $\Vb_t := \sum_{s=1}^{t} 
    \sum_{i \in S_s}
        \ab_{i} 
        \ab_{i}^\top
    + \lambda \Ib_d$.
Therefore, it suffices to ensure that $\kappa^{-1/2}
        \max_{i \in [N]} \|\ab_i \|_{\Vb_t^{-1}}$ is uniformly bounded:
\begin{align*}
        \max_{i\in[N]}\|\ab_i\|_{\Vb_t^{-1}}
        &\le\WarmupThreshold,
\end{align*}
where $\WarmupThreshold
        :=\frac{\sqrt\kappa}{1024}\Big(
        \frac1{\sqrt{d\log\!16(N+1)^2/\delta}}
        \wedge\frac1{\sqrt\lambda B}\Big).$

\textbf{Step 2: Fixing nominal parameter. }\,
After the initial exploration phase, we estimate a nominal parameter $\thetab_0$ using an
\emph{independent} feedback dataset $\Dcal_w'$.
During the initial exploration, each selected assortment $S_t$ is offered twice, yielding two independent
choice outcomes $i_t$ and $i_t'$.
These outcomes are stored in two separate datasets,
$\Dcal_w := \{(S_s,i_s)\}_{s=1}^{t_w}$ and
$\Dcal_w' := \{(S_s,i_s')\}_{s=1}^{t_w}$, which are conditionally independent given the offered assortments.
We then compute the nominal parameter as the regularized MLE, i.e., 
$
\thetab_0 \in \argmin_{\thetab\in\RR^d} \Lcal_{\Dcal_w'}(\thetab).
$
Conditional on the warm-up query sequence and its independent randomization,
$\thetab_0$ depends only on $\Dcal_w'$ and is independent of the unprimed
responses in $\Dcal_w$.
This independence is essential for satisfying the conditional independence requirement in
Lemma~\ref{lemma:estimation_error_main} when we later sample assortments according to the design computed at $\thetab_0$.

Finally, by running the warm-up phase for a sufficient number of rounds, we ensure that the nominal parameter $\thetab_0$ is uniformly close to $\thetab^\star$, which in turn implies that $p(i |S, \thetab_0) \approx p(i |S, \thetab^\star)$ for all $S$ and $i$ (see Equation~\ref{eq:p_star_bound}).

\begin{figure*}[t]
    \centering
    \begin{subfigure}{0.245\textwidth}
        \centering
        \includegraphics[width=\linewidth]{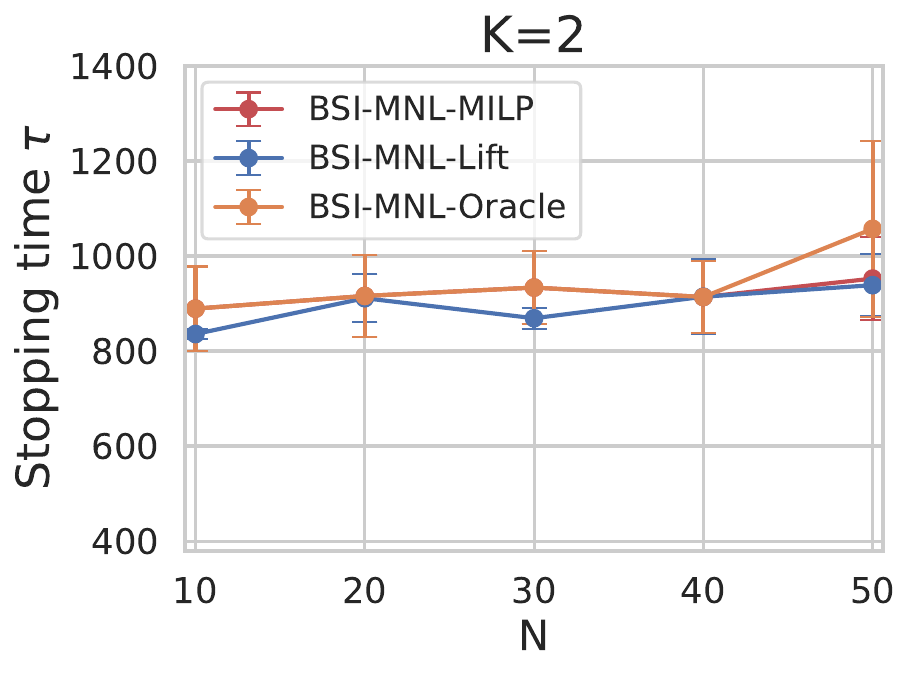}
    \end{subfigure}
    \begin{subfigure}{0.245\textwidth}
        \centering
        \includegraphics[width=\linewidth]{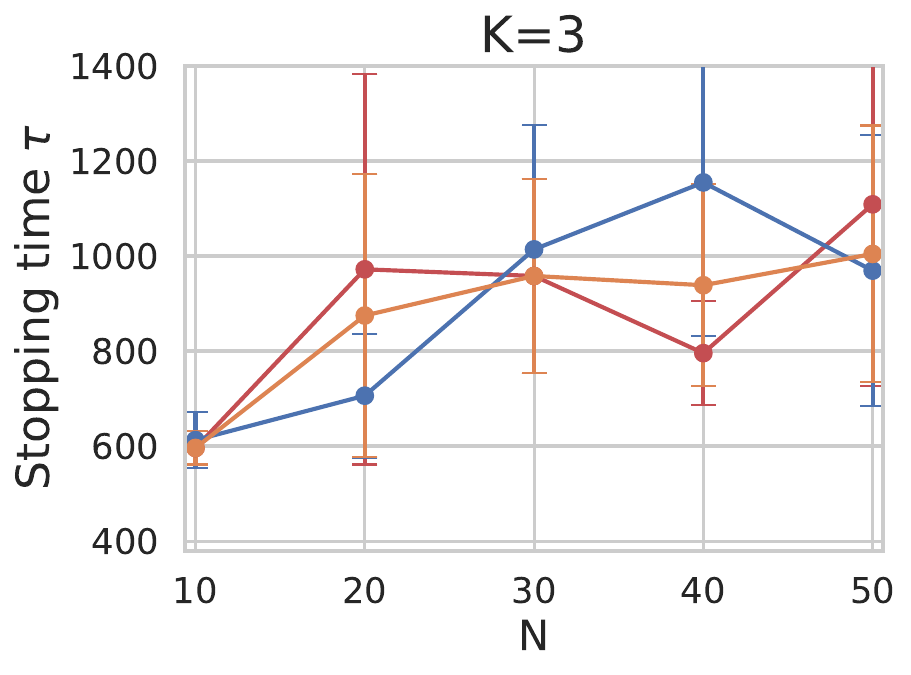}
    \end{subfigure}
    \begin{subfigure}{0.245\textwidth}
        \centering
        \includegraphics[width=\linewidth]{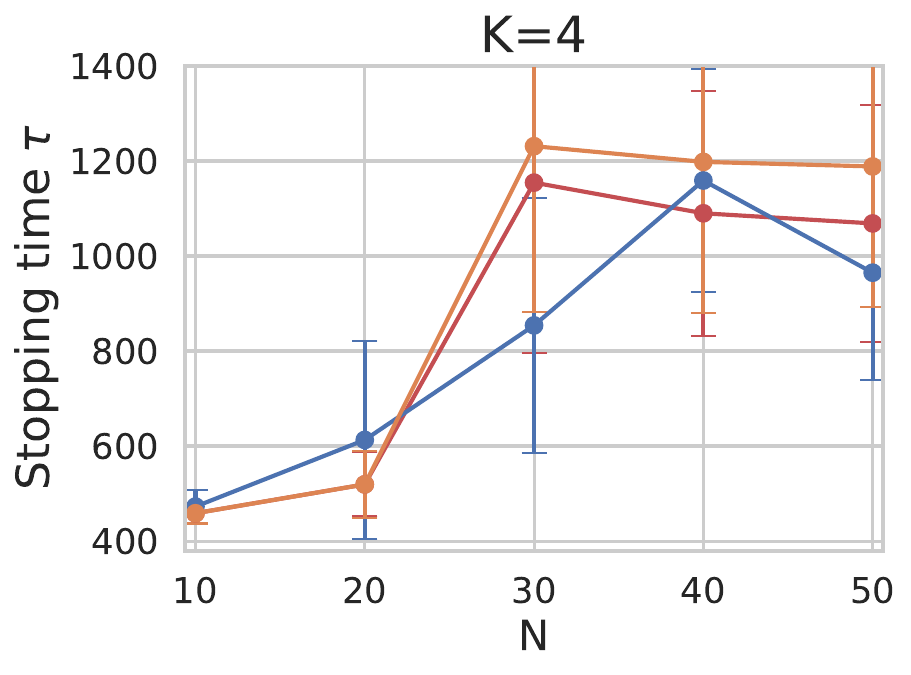}
    \end{subfigure}
    \begin{subfigure}{0.245\textwidth}
        \centering
        \includegraphics[width=\linewidth]{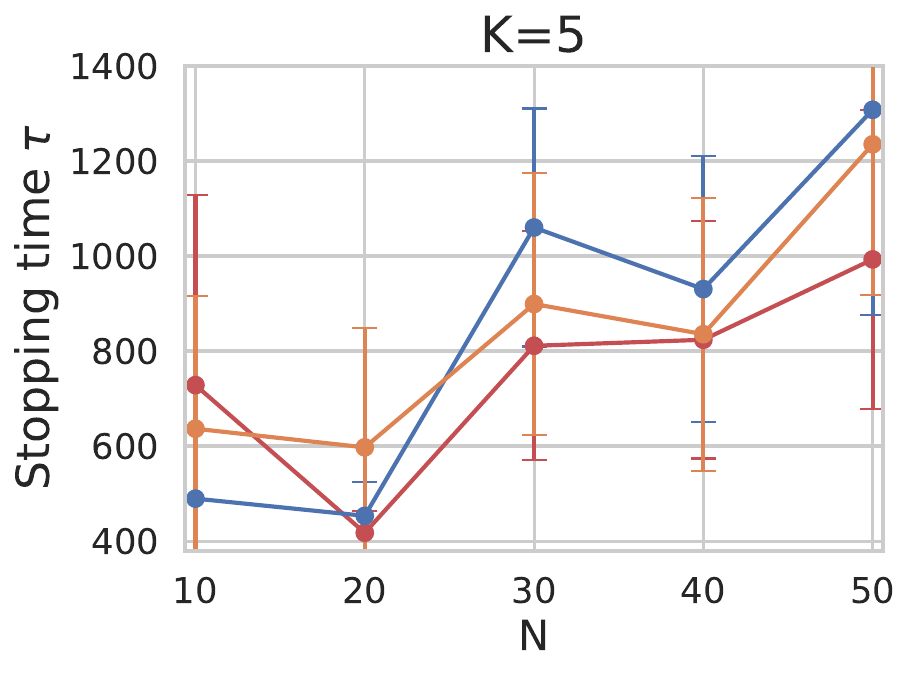}
    \end{subfigure}

    \caption{Average stopping time $\tau$ for varying $N$ and $K$ over 
     $10$ independent random seeds, with one standard deviation.}
    \label{fig:four_panels}
\end{figure*}

\textbf{Step 3: $G$-optimal design at $\thetab_0$. }\,
Given the nominal parameter $\thetab_0$, we compute a (near) $G$-optimal design $\widehat{\pi}_{\thetab_0}$
by running the Frank--Wolfe algorithm (Algorithm~\ref{alg:FW_MNL}).
The design can be computed either using a \textit{MILP solver} (Subsection~\ref{subsec:lmo_milp}) 
or via a relaxed but polynomial-time \textit{lifted G-optimal design} formulation (Subsection~\ref{subsec:lifted_lmo_mnl}).
In these design computations, include outcome $0$ in every experiment
without charging it to the capacity: in the MILP, fix $x_0=1$ and
$1\le\sum_{i=1}^N x_i\le K$; in the lifted matrix, sum over $S\cup\{0\}$.
All design scores use the outside-option Fisher information.

\textbf{Step 4: Main rounds and stopping rule. }\,
Given the design $\widehat{\pi}_{\thetab_0}$, we run the main rounds by repeatedly sampling assortments
$S_t \sim \widehat{\pi}_{\thetab_0}$ and observing the corresponding choice feedback.
At each time $t$, set
$\delta_t:=\delta/[8(t+1)^2]$,
$L_t:=\log(4(N+1)^2/\delta_t)$, and
$\beta_t:=\beta(\delta_t)$.
We form the optimistic ($u^+_{t,i}$) and pessimistic ($u^{-}_{t,i}$) utilities
\[
u^\pm_{t,i}=\ab_i^\top\widehat{\thetab}_{t} \pm \sqrt{2}\,\beta_t\,\|\ab_i\|_{\Hb_{t}(\thetab_0)^{-1}}.
\]
Before stopping, we require
\begin{equation}
    256\max_{i\in[N]}\|\ab_i\|_{\Hb_t(\thetab_0)^{-1}}
    \le\frac1{\sqrt{d\log9+L_t+\lambda B^2}}.
    \label{eq:confidence_check}
\end{equation}
If this condition fails, we continue sampling from the same fixed design.
The time-dependent multiplier and this check ensure that the fixed-design
confidence bound is valid simultaneously at every time when stopping
is allowed. We then define
the corresponding optimistic/pessimistic expected revenues for any $S\in\Scal$ as
\[
\widetilde R_t(S)
:=
\frac{\sum_{i\in S} e^{u^+_{t,i}} r_i}{1+\sum_{j\in S} e^{u^+_{t,j}}},
\quad
\widecheck R_t(S)
:=
\frac{\sum_{i\in S} e^{u^-_{t,i}} r_i}{1+\sum_{j\in S} e^{u^-_{t,j}}}.
\]
We then select the pessimistic best assortment and the optimistic best alternative as
\[
\BestS_t \in \argmax_{S \in \Scal} \widecheck R_t(S),
\qquad
\AltS_t \in \argmax_{S \in \Scal:\, S \neq \BestS_t} \widetilde R_t(S).
\]
\begin{remark}[Polynomial-time assortment selection]
    Note that each of these optimization problems can be solved in polynomial time in $N$~\citep{rusmevichientong2010dynamic, davis2014assortment}, analogously to~\eqref{eq:LMO_lifted_theta}.
    The only subtlety is the constraint $S\neq \BestS_t$ in the definition of $\AltS_t$.
    Every $S\ne\BestS_t$ either omits an item of $\BestS_t$ or includes
    an item outside $\BestS_t$.
    Thus solve the same maximization problem with $i\notin S$ for each
    $i\in\BestS_t$, and with $j\in S$ for each $j\notin\BestS_t$,
    discarding infeasible subproblems.
    Taking the best of these at most $N$ constrained solutions is exact
    and also covers strict supersets of the incumbent.
\end{remark}

Finally, we stop at the first time $\tau$ at which~\eqref{eq:confidence_check} holds and
\begin{equation}
    \widecheck R_\tau(\BestS_\tau) > \widetilde R_\tau(\AltS_{\tau}),
    \label{eq:stop_rule}
\end{equation}
and output $\BestS_\tau$.
In Lemma~\ref{lemma:S_tau_on_stopping} , we show that the above stopping rule is sufficient to guarantee $\BestS_\tau=S^\star$.


\begin{table*}[t]
    \centering
    \small
    \renewcommand{\arraystretch}{1.15}
    \caption{ Average LMO runtime (seconds) with standard deviation over $10$ independent random seeds.
    Bold entries denote the fastest and second-fastest methods for each $(N,K)$ configuration.
    }
    \begin{tabular}{cc|c|cc|c}
    \toprule
    $N$ & $K$ & $\binom{N}{K}$
    & \texttt{\AlgName{}-MILP}
    & \texttt{\AlgName{}-Lift}
    & \texttt{\AlgName{}-Oracle} \\
    \midrule
    
    \multirow{3}{*}{30}
    & 3 & 4{,}060
      & 2.06 {\scriptsize$(\pm 0.34)$} & \textbf{0.16} {\scriptsize$(\pm 0.05)$} & \textbf{0.38} {\scriptsize$(\pm 0.13)$} \\
    & 4 & 27{,}405
      & 12.76 {\scriptsize$(\pm 4.48)$} & \textbf{0.15} {\scriptsize$(\pm 0.04)$} & \textbf{2.63} {\scriptsize$(\pm 0.68)$} \\
    & 5 & 142{,}506
      & 41.02 {\scriptsize$(\pm 16.27)$} & \textbf{0.18} {\scriptsize$(\pm 0.06)$} & \textbf{17.02} {\scriptsize$(\pm 6.23)$} \\
    
    \midrule
    \multirow{3}{*}{50}
    & 3 & 19{,}600
      & 9.10 {\scriptsize$(\pm 1.07)$} & \textbf{0.14} {\scriptsize$(\pm 0.04)$} & \textbf{1.57} {\scriptsize$(\pm 0.34)$} \\
    & 4 & 230{,}300
      & 77.67 {\scriptsize$(\pm 18.16)$} & \textbf{0.17} {\scriptsize$(\pm 0.07)$} & \textbf{38.20} {\scriptsize$(\pm 7.06)$} \\
    & 5 & 2{,}118{,}760
      & \textbf{135.14} {\scriptsize$(\pm 47.82)$} & \textbf{0.16} {\scriptsize$(\pm 0.06)$} & 329.23 {\scriptsize$(\pm 56.10)$} \\
    
    \midrule
    \multirow{3}{*}{100}
    & 3 & 161{,}700
      & 121.84 {\scriptsize$(\pm 41.27)$} & \textbf{0.16} {\scriptsize$(\pm 0.07)$} & \textbf{21.06} {\scriptsize$(\pm 5.44)$} \\
    & 4 & 3{,}921{,}225
      & \textbf{135.38} {\scriptsize$(\pm 54.01)$} & \textbf{0.17} {\scriptsize$(\pm 0.08)$} & 368.94 {\scriptsize$(\pm 82.40)$} \\
    & 5 & 75{,}287{,}520
      & \textbf{214.13} {\scriptsize$(\pm 101.56)$} & \textbf{0.17} {\scriptsize$(\pm 0.08)$} & 9048.44 {\scriptsize$(\pm 1579.01)$} \\

    
    \bottomrule
    \end{tabular}
    \label{tab:lmo_runtime}
\end{table*}
\vspace{-0.2cm}
%
\subsection{Sample Complexity}
\label{subsec:sample_complexity}
We define the (instance-dependent) expected revenue gap
\[
\Gap := R(S^\star,\thetab^\star)-\max\{R(S,\thetab^\star)\!:\! S \in \Scal,\ S\neq S^\star\}.
\]
We now establish an instance-dependent upper bound on the sample complexity of
Algorithm~\ref{alg:MNL_BAI}.
\begin{theorem}[Sample complexity, MILP solver]
\label{thm:BSI_MNL}
    Let $\delta\in(0,1)$, $\lambda=1$, and fix the design accuracy
    $\epsilon\in(0,1]$ with $0\le\ErrorLMO\le\epsilon d/2$.
    Use the warm-up threshold in Step~1, and let $\WarmupRounds$
    denote the resulting number of warm-up rounds.
    Suppose that Assumption~\ref{assum:bounded_assumption} holds.
    Then,
    with probability at least $1-\delta$,
    Algorithm~\ref{alg:MNL_BAI}, when implemented with the MILP solver described in Subsection~\ref{subsec:lmo_milp}, returns the optimal assortment $S^\star$,
    and
    the stopping round $\tau$ satisfies
    \begin{align*}
        \tau
        =
        \BigOTilde\!\left(
            d \log\!\left(N/\delta\right)
            \left(
                \frac{1}{\Gap^2}
                + \frac{1}{\kappa\,\Gap}
            \right)
            + \WarmupRounds
        \right)\!.
    \end{align*}
\end{theorem}
The displayed sample-complexity rates in this section treat $B$ as fixed
and $N\ge2$. With radius dependence explicit, replace $\log(N/\delta)$
by $\log((N+1)/\delta)+B^2$.
The total number of choice observations is at most $2\tau$, so it has
the same order. The additional logarithms from $\delta_t$ are included
in $\BigOTilde(\cdot)$.
\textbf{Discussion of Theorem~\ref{thm:BSI_MNL}. }\,
To the best of our knowledge, this is the first sample complexity bound for best assortment identification in MNL bandits, particularly in the presence of non-uniform revenue parameters.
For sufficiently small $\Gap$, the sample complexity is dominated by the $\Gap^{-2}$ term, yielding the leading-order bound
$\BigOTilde\big(\frac{d \log N}{\Gap^2}\big)$.
This matches the existing lower bound $\Omega \big(\frac{d}{\Gap^2}\big)$ (Theorem~3 of~\citealt{gupta2021pure}), up to a logarithmic factor in $N$.
In particular, in this small-gap regime, the bound is independent of $1/\kappa$, which can be exponentially large, i.e., $1/\kappa = \mathcal{O}(K^2 e^{3B})$.
The proof is provided in Appendix~\ref{app:sec:BAI_MNL}.

\begin{corollary}[Sample complexity, Lifted G-optimal]
\label{cor:BSI_MNL_lifted}
    Under the setting as in Theorem~\ref{thm:BSI_MNL}, suppose we employ the lifted G-optimal design described in Subsection~\ref{subsec:lifted_lmo_mnl}.
    Then, with probability at least $1-\delta$, we have
    \begin{align*}
        \tau
        =
        \BigOTilde\!\left( (1+ \ErrorLift)
            d \log\!\left(N/\delta\right)
            \left(
                \frac{1}{\Gap^2}
                + \frac{1}{\kappa\,\Gap}
            \right)
            + \WarmupRounds
        \right)\!,
    \end{align*}
    where $\ErrorLift$ denotes the approximation error induced by the lifted G-optimal design (see Theorem~\ref{thm:liftedFW}).
\end{corollary}
\textbf{Discussion of Corollary~\ref{cor:BSI_MNL_lifted}. }\,
When we employ the lifted $G$-optimal design to improve computational efficiency, particularly in large-scale regimes with large $N$ and $K$, the sample complexity incurs an additional multiplicative factor of $(1+\ErrorLift)$.
By Proposition~\ref{prop:outside_option_bound_eps_by_ratio}, this lifting error satisfies $\ErrorLift=\BigO(Ke^B)$.
For comparison, a naive application of optimal design based on $\Vb_t$, which ignores local Fisher information and
therefore admits only polynomial-time computation (cf.~\citealt{gupta2021pure}),
yields a sample complexity of
$\BigOTilde\big( \frac{d \log N}{\kappa^2 \Gap^2} \big)$.
Since $1/\kappa^{2}=\BigO(K^{4}e^{6B})$, this bound can be substantially worse.
In contrast, Corollary~\ref{cor:BSI_MNL_lifted} achieves a better worst-case statistical efficiency, by a factor of
$K^{3}e^{5B}$, while retaining polynomial-time computation.

\begin{remark}[Relation to item-level gaps]
    Our gap $\Delta_{\min}$ is an assortment-level global gap between the optimal and second-best assortments, whereas the $\Delta_i$ quantities in~\citealt{yang2021fully} and~\citealt{karpov2022instance} are item-level separation metrics in MAB settings. 
    In particular, \citealt{yang2021fully} defines $\Delta_i$ using the best revenue among assortments constrained to include or exclude item $i$, for which one has $\Delta_i \geq \Delta_{\min}$; however, their complexity scales as $\widetilde{O}(\sum_{i=1}^N \Delta_i^{-2})$, while ours scales as $\widetilde{O}(\Delta_{\min}^{-2})$, so the two bounds are not directly comparable: theirs uses a finer item-wise hardness measure, whereas ours uses a single global assortment-level gap, and neither uniformly dominates the other based on the gap term alone. Meanwhile, \citealt{karpov2022instance} defines $\Delta_i$ through the advantage score $\eta_i=(r_i-\theta_v)v_i$ and its separation from the top-$K$ decision boundary; hence, their gap is not based on constrained assortment revenues, but on item-wise margins relative to the advantage-score boundary, capturing a different notion of hardness from $\Delta_{\min}$ and making term-by-term comparison again inappropriate.
\end{remark}

\section{Numerical Experiments}
\label{sec:experiments}

We empirically evaluate the performance of \AlgName{} by measuring the sample complexity up to the stopping time.
All results are averaged over $10$ independent problem instances, and we report the mean together with the standard deviation.
In each instance, the true parameter $\thetab^\star$ is drawn uniformly from the $d$-dimensional Euclidean ball $\BB^d(B)$ of radius $B$, and the feature vectors $\{\ab_i\}_{i=1}^N$ are sampled independently from the unit ball $\BB^d(1)$.
The revenue parameters are independently sampled as $r_i \!\sim\! \operatorname{Unif}(0,1)$ for all $i \!\in\! [N]$.

To implement the Frank–Wolfe (FW) procedure (Algorithm~\ref{alg:FW_MNL}) within \AlgName{}, we consider two variants that differ in how the linear minimization oracle (LMO) is computed.
The first variant, denoted \texttt{\AlgName{}-MILP}, solves the LMO by reformulating the original combinatorial optimization problem as a mixed-integer linear program and uses the Gurobi solver~\citep{gurobi}, as described in Subsection~\ref{subsec:lmo_milp}.
The second variant, denoted \texttt{\AlgName{}-Lift}, replaces the original LMO with a Schur-complement–based lifted surrogate, which admits polynomial-time computation; see Subsection~\ref{subsec:lifted_lmo_mnl} for details.
Since this work is, to our knowledge, the first to study best assortment identification with provable guarantees, there is no existing algorithm that serves as a direct baseline.
Instead, we include an oracle version of our algorithm that computes the LMO exactly via \textit{brute-force} enumeration, denoted \texttt{\AlgName{}-Oracle}.

In the FW procedure, we use a stopping accuracy of $\epsilon = 0.1$. For the MILP-based variant, \texttt{\AlgName{}-MILP}, the certified optimality gap of the linear minimization oracle is set to $\epsilon_{\mathrm{LMO}} = 0.1$.
The reported numerical experiments use the fixed-confidence radius
$\beta(\delta)$ and the warm-up threshold $\zeta_w$, scaled down by a
factor of $10$ from the original calibration.
These practical settings are different from the time-uniform multiplier
$\beta_t$ and confidence check used for the theoretical guarantees;
the numerical results are not recomputed under that calibration.
We set  $d = 5$, $B = 1$, and $\delta = 0.05$.

\textbf{Stopping time. }\,
Figure~\ref{fig:four_panels} reports the stopping time $\tau$ of \AlgName{} as a function of the number of items $N$ and the assortment size $K$.
Both proposed methods, \texttt{\AlgName{}-MILP} and \texttt{\AlgName{}-Lift}, achieve stopping times comparable to the oracle implementation \texttt{\AlgName{}-Oracle}.
These results demonstrate that our proposed methods retain statistical efficiency while substantially reducing the computational cost associated with brute-force LMO computations.

One interesting observation is that, for larger $K$, the sample complexity appears to be more sensitive to $N$.
We conjecture that this behavior is partly driven by the revenue gap $\Gap\! =\! \min_{S \neq S^\star}\! \left( R(S^\star, \thetab^\star)\! - \!R(S, \thetab^\star) \right)$,
which for larger $K$ may decrease more rapidly with $N$.
Indeed, by the definition of 
$R(S, \thetab)$, adding more items to an assortment tends to increase the expected revenue, which in turn can reduce the revenue gaps, especially when many items are available.

\textbf{Runtime. }\,
Table~\ref{tab:lmo_runtime} compares the computational cost of solving the LMO, the main computational bottleneck of our method.
Owing to its polynomial-time formulation, \texttt{\AlgName{}-Lift} is consistently and significantly faster than the other approaches, with nearly constant runtime across all tested $(N,K)$ configurations.
Moreover, as $\binom{N}{K}$ increases, \texttt{\AlgName{}-MILP} becomes increasingly competitive with \texttt{\AlgName{}-Oracle}, whose brute-force enumeration becomes computationally expensive.

\section{Conclusion}
\label{sec:conclusion}
In this paper, we studied optimal experimental design for MNL models and proposed two computationally efficient frameworks to address the combinatorial nature of the action space. 
Building on these design methods, we developed a best assortment identification algorithm for MNL bandits and established an instance-dependent sample complexity.

\section*{Acknowledgments}
\label{sec:acknowledgments}
This work was supported by the National Research Foundation of Korea~(NRF) grant and the Institute of Information \& communications Technology Planning \& Evaluation~(IITP) grant,
both funded by the Korea government~(MSIT) 
(No. RS-2022-NR071853, RS-2023-00222663, 
RS-2025-25463302, 
RS-2026-25507282, 
RS-2025-25420849),
and by the 2025 Global Google PhD Fellowship, with support from Google.org.

\section*{Impact Statement}

This paper presents work whose goal is to advance the field of 
Machine Learning. There are many potential societal consequences 
of our work, none which we feel must be specifically highlighted here.

\nocite{langley00}

\bibliography{main_bib}
\bibliographystyle{icml2026}

\newpage
\appendix
\onecolumn

\counterwithin{table}{section}
\counterwithin{theorem}{section}
\counterwithin{algorithm}{section}
\counterwithin{figure}{section}
\counterwithin{equation}{section}
\counterwithin{condition}{section}

\section{Justification of Assumption~\ref{assm:span} (Without Loss of Generality)}
\label{app:wlog_span}

In this appendix, we justify that Assumption~\ref{assm:span} can be imposed without loss of generality.
Recall that under the MNL model, for any assortment $S \subseteq [N]$ and parameter $\thetab \in \RR^d$,
the choice probability of item $i \in S$ is defined as follows:
\begin{align*}
p(i | S,\thetab)
=\frac{\exp(\ab_i^\top \thetab)}{\sum_{j\in S}\exp(\ab_j^\top \thetab)}.
\numberthis
\label{eq:mnl_prob_app}
\end{align*}

\paragraph{Identifiable subspace.}
Define the subspace generated by all pairwise feature differences as
\begin{align*}
\mathcal U \;:=\; \operatorname{span}\big\{\ab_i-\ab_j:\ i,j\in[N],\, i\neq j\big\}
\;\subseteq\; \RR^d.
\end{align*}
If $\mathcal U \neq \RR^d$, then there exists a nontrivial orthogonal complement
$\mathcal U^\perp := \{u\in\RR^d:\  v^\top u =0,  \forall v\in\mathcal U\}$.
For any $u\in\mathcal U^\perp$ and any $i,j\in[N]$, we have
\begin{align*}
(\ab_i-\ab_j)^\top u=0
\quad\Longrightarrow\quad
\ab_i^\top u = \ab_j^\top u.
\end{align*}
Hence, $ \ab_i^\top u$ is constant over $i\in[N]$. In particular, fixing any $i_0\in[N]$ and letting
$c(u):=\ab_{i_0}^\top u$, we obtain
\begin{align*}
\ab_i^\top u = c(u),
\qquad \forall i\in[N],\ \forall u\in\mathcal U^\perp.
\numberthis \label{eq:u_const}
\end{align*}

\paragraph{Decomposition of the parameter.}
Decompose the parameter as $\thetab = \thetab_{\mathcal U} + \thetab_{\mathcal U^\perp}$, where
$\thetab_{\mathcal U}\in\mathcal U$ and $\thetab_{\mathcal U^\perp}\in\mathcal U^\perp$.
By~\eqref{eq:u_const}, there exists a scalar $c(\thetab_{\mathcal U^\perp})\in\RR$ such that
\begin{align*}
\ab_i^\top \thetab_{\mathcal U^\perp} = c(\thetab_{\mathcal U^\perp}),
\qquad \forall i\in[N].
\end{align*}
Therefore, for any $i\in S$,
\begin{align*}
\ab_i^\top \thetab
=\ab_i^\top \thetab_{\mathcal U}+\ab_i^\top \thetab_{\mathcal U^\perp}
=\ab_i^\top \thetab_{\mathcal U}+c(\thetab_{\mathcal U^\perp}),
\numberthis
\label{eq:score_split}
\end{align*}
where the additive term $c(\thetab_{\mathcal U^\perp})$ does not depend on $i$.

\paragraph{Invariance of MNL probabilities.}
Plugging~\eqref{eq:score_split} into~\eqref{eq:mnl_prob_app}, we obtain, for any $S$ and $i\in S$,
\begin{align}
p(i | S,\thetab)
&=\frac{\exp(\ab_i^\top \thetab_{\mathcal U})\exp(c(\thetab_{\mathcal U^\perp}))}
{\sum_{j\in S}\exp(\ab_j^\top \thetab_{\mathcal U})\exp(c(\thetab_{\mathcal U^\perp}))}
\nonumber\\
&=\frac{\exp(\ab_i^\top \thetab_{\mathcal U})}{\sum_{j\in S}\exp(\ab_j^\top \thetab_{\mathcal U})}
\;=\; p(i | S,\thetab_{\mathcal U}).
\label{eq:mnl_invariant}
\end{align}
Thus, the MNL choice probabilities depend only on the projection $\thetab_{\mathcal U}$ onto the
difference-span subspace $\mathcal U$, and the component $\thetab_{\mathcal U^\perp}$ is not identifiable
from choice data.

\paragraph{Without loss of generality.}
Let $d':=\dim(\mathcal U)$ and fix any orthonormal basis of $\mathcal U$ to represent parameters
in $\RR^{d'}$. Since~\eqref{eq:mnl_invariant} shows that replacing $\thetab$ by $\thetab_{\mathcal U}$
leaves all MNL probabilities unchanged, we may restrict the parameter space to $\mathcal U$
(and redefine $d \leftarrow d'$) without loss of generality.
Under this reparameterization, the spanning condition in Assumption~\ref{assm:span} holds by construction.

\section{Supplementary Material for Section~\ref{sec:opt_design_MNL}}
\label{app:proof_thm:LMO}
\subsection{Proof of Theorem~\ref{thm:LMO_MILP}}
\label{app_subsec:proof_thm:LMO_MILP}
\begin{proof}[Proof of Theorem~\ref{thm:LMO_MILP}]
In this subsection, we provide the proof of Theorem~\ref{thm:LMO_MILP}, showing that the LMO optimization problem in~\eqref{eq:LMO} reduces to a $0$--$1$ Mixed Integer Linear Program (MILP).

Fix the current Frank--Wolfe iterate $\pi_m$ and let
$\Mb_m := \Mb_{\thetab_0}(\pi_m)\succ 0$.
Recall the definition of MNL choice probabilities (at $\thetab_0$):
\begin{align*}
    p(i| S,\thetab_0)
    := \frac{\exp(\ab_i^\top\thetab_0)}{\sum_{j\in S}\exp(\ab_j^\top\thetab_0)},
    \qquad  \forall i\in S,
\end{align*}
and the corresponding Fisher information matrix:
\begin{align*}
    \Ib_{\thetab_0}(S)
    := \sum_{i\in S} p(i| S,\thetab_0)\,\big(\ab_i-\bar \ab_S(\thetab_0)\big)\big(\ab_i-\bar \ab_S(\thetab_0)\big)^\top,
    \qquad
    \text{where } \,\bar \ab_S(\thetab_0):=\sum_{j\in S} p(j| S,\thetab_0)\,\ab_j.
\end{align*}
Then, the linear maximization oracle (LMO) selects an extreme point supported on a single assortment
\begin{align*}
    S_m \in \argmax_{S\in\Scal}\ \tr\!\Big(\Mb_m^{-1} \Ib_{\thetab_0}(S)\Big).
\end{align*}

\paragraph{Binary encoding.}
We represent an assortment $S$ by a binary vector $x\in\{0,1\}^N$, where
$x_i = \mathbf{1}\{i\in S\}$.
This encoding ensures that $\sum_{i=1}^N x_i = |S|$, where $2 \leq |S| \leq K$.
For a simple presentation, we define
\begin{align*}
    w_i := \exp(\ab_i^\top\thetab_0)>0,\qquad w:=(w_1,\ldots,w_N)^\top,
    \qquad
    W(x):=w^\top x=\sum_{i=1}^N w_i x_i.
\end{align*}
Then, for any $i\in S$, the MNL choice probability can be written as
\begin{align*}
p(i | S,\thetab_0) = \frac{w_i x_i}{W(x)} .
\end{align*}
Moreover, the mean feature vector satisfies
\begin{align*}
\bar \ab_S(\thetab_0)
= \sum_{j\in S} p(j | S,\thetab_0)\, \ab_j
= \frac{\sum_{j=1}^N w_j x_j \ab_j}{W(x)} .
\end{align*}

\paragraph{Reformulation as a $0$--$1$ QFIP.}
Using the definition of 
$\Ib_{\thetab_0}(S)$,
we obtain
\begin{align}
\tr\!\Big(\Mb_m^{-1} \Ib_{\thetab_0}(S)\Big)
&= \sum_{i\in S} p(i| S,\thetab_0)\, \underbrace{\ab_i^\top \Mb_m^{-1} \ab_i}_{=: s_i}
\;-\; \bar \ab_{S}(\thetab_0)^\top \Mb_m^{-1} \bar \ab_{S}(\thetab_0) \nonumber\\
&= \frac{\sum_{i=1}^N x_i \overbrace{w_i s_i}^{=: r_i} }{W(x)}
\;-\;
\frac{\Big(\sum_{i=1}^N w_i x_i \ab_i\Big)^\top \Mb_m^{-1} \Big(\sum_{j=1}^N w_j x_j \ab_j\Big)}{W(x)^2}\nonumber\\
&= \frac{r^\top x}{W(x)}
\;-\;
\frac{\Big(\sum_{i=1}^N w_i x_i \ab_i\Big)^\top \Mb_m^{-1} \Big(\sum_{j=1}^N w_j x_j \ab_j\Big)}{W(x)^2}
,
\label{eq:trace_expand_1}
\end{align}
where we defined
\begin{align*}
    s_i := \ab_i^\top \Mb_m^{-1} \ab_i,\qquad r_i := w_i s_i,\qquad r:=(r_1,\ldots,r_N)^\top .
\end{align*}
Define the symmetric matrix $G\in\RR^{N\times N}$ by $G_{ij} := w_i w_j\, \ab_i^\top \Mb_m^{-1} \ab_j$.
Then we have
\begin{align*}
    \Big(\sum_{i=1}^N w_i x_i \ab_i\Big)^\top \Mb_m^{-1} \Big(\sum_{j=1}^N w_j x_j \ab_j\Big)=x^\top G x.
\end{align*}
Moreover, since $W(x)^2=(w^\top x)^2=x^\top (w w^\top) x$, the right-hand side of~\eqref{eq:trace_expand_1} becomes
\begin{equation}
\label{eq:trace_qfp}
\tr\!\Big(\Mb_m^{-1} \Ib_{\thetab_0}(S)\Big)
=
\frac{(w^\top x)(r^\top x)-x^\top G x}{(w^\top x)^2}.
\end{equation}
Finally, symmetrizing the bilinear term yields a pure quadratic form:
\begin{align*}
    (w^\top x)(r^\top x)
    = x^\top\Big(\tfrac12(w r^\top + r w^\top)\Big)x.
\end{align*}
Hence, by defining
\begin{align*}
    A := -\tfrac12(w r^\top + r w^\top) + G,
    \qquad
    B := w w^\top,
\end{align*}
we can rewrite \eqref{eq:trace_qfp} as
\begin{align*}
    \tr\!\Big(\Mb_m^{-1} \Ib_{\thetab_0}(S)\Big)=- \frac{x^\top A x}{x^\top B x}.
\end{align*}
Therefore, the LMO in~\eqref{eq:LMO} is equivalently written as the following
$0$--$1$ quadratic fractional integer problem (QFIP):
\begin{equation}
\label{eq:LMO_QFP}
    \max_{x\in\{0,1\}^N}\ \,-\frac{x^\top A x}{x^\top B x}
    = 
    -\min_{x\in\{0,1\}^N}\ \,\frac{x^\top A x}{x^\top B x},
    \qquad\text{s.t.}\quad 2 \leq \sum_{i=1}^N x_i \leq K.
\end{equation}
We next reformulate this QFIP~\eqref{eq:LMO_QFP} as a linear fractional integer program (LFIP), adapting the approach used in the proof of Theorem~2.2 of~\citet{kapoor2006linearization}.

\paragraph{QFIP $\Rightarrow$ LFIP.}
Let $\mathbf{1}\in\RR^N$ denote the all-ones vector and define the row-sum bounds
\[
    M_A := \max_{i\in[N]} \sum_{j=1}^N |A_{ij}|,
    \qquad
    M_B := \max_{i\in[N]} \sum_{j=1}^N B_{ij}.
\]
For any $x\in\{0,1\}^N$, we have $(Ax)_i\in[-M_A,M_A]$, hence
\[
    \tilde A x := Ax + M_A \mathbf{1} \in [0,2M_A]^N .
\]
Introduce auxiliary vectors $y^0,y^1\in\RR_+^N$ such that
\begin{align}
\label{eq:lin_Ax_split}
\tilde A x = y^0 + y^1,\qquad
0 \le y^0 \le 2M_A(\mathbf{1}-x),\qquad
0 \le y^1 \le 2M_A x .
\end{align}
These constraints enforce $y_i^1 = (\tilde A x)_i$ when $x_i=1$ and $y_i^1=0$ when $x_i=0$; thus
\[
\mathbf{1}^\top y^1 = x^\top \tilde A x = x^\top A x + M_A \mathbf{1}^\top x
\quad\Longrightarrow\quad
x^\top A x = \mathbf{1}^\top y^1 - M_A \mathbf{1}^\top x .
\]

Similarly, since $B = ww^\top$ has nonnegative entries, we have $Bx\in[0,M_B]^N$.
Introduce $z^0,z^1\in\RR_+^N$ such that
\begin{align}
\label{eq:lin_Bx_split}
B x = z^0 + z^1,\qquad
0 \le z^0 \le M_B(\mathbf{1}-x),\qquad
0 \le z^1 \le M_B x .
\end{align}
Then $\mathbf{1}^\top z^1 = x^\top B x$ holds, because
$\mathbf{1}^\top z^1 = \sum_{i=1}^N x_i (Bx)_i = x^\top B x$.

Therefore, \eqref{eq:LMO_QFP} is equivalent to the following $0$--$1$ linear fractional integer problem (LFIP):
\begin{align*}
    \min_{x\in\{0,1\}^N,\; y^0,y^1,z^0,z^1\in\RR_+^N}
    \quad &
    \frac{\mathbf{1}^\top y^1 - M_A \mathbf{1}^\top x}{\mathbf{1}^\top z^1} 
    \numberthis \label{eq:LMO_LFIP}
    \\
    \text{s.t.}\quad
    & A x + M_A \mathbf{1} = y^0 + y^1, \\
    & 0 \le y^0 \le 2M_A(\mathbf{1}-x), \\
    & 0 \le y^1 \le 2M_A x, \\
    & B x = z^0 + z^1, \\
    & 0 \le z^0 \le M_B(\mathbf{1}-x), \\
    & 0 \le z^1 \le M_B x, \\
    & 2 \le \mathbf{1}^\top x \le K.
\end{align*}

\paragraph{LFIP $\Rightarrow$ MILP (Charnes–Cooper).}
We now reduce the LFIP \eqref{eq:LMO_LFIP} to an equivalent $0$--$1$ mixed integer linear program (MILP)
via the Charnes--Cooper transformation~\citep{cooper1962programming}.
Since $B=ww^\top$ and $w_i>0$, the cardinality constraint $\mathbf{1}^\top x\ge 2$ implies
\[
w^\top x \ge 2 w_{\min},\qquad w_{\min}:=\min_{i\in[N]} w_i,
\qquad\Rightarrow\qquad
x^\top B x = (w^\top x)^2 \ge 4 w_{\min}^2.
\]
Hence the denominator in \eqref{eq:LMO_LFIP} satisfies $\mathbf{1}^\top z^1 = x^\top B x \ge 4w_{\min}^2$.
Let
\[
\alpha := \frac{1}{\mathbf{1}^\top z^1}\in\Big(0,\ \bar \alpha\Big],\qquad
\bar \alpha := \frac{1}{4w_{\min}^2},
\]
and define the scaled variables
\[
v := \alpha x,\quad
S^0 := \alpha y^0,\quad S^1 := \alpha y^1,\quad
U^0 := \alpha z^0,\quad U^1 := \alpha z^1 .
\]
Then the objective becomes linear:
\[
\frac{\mathbf{1}^\top y^1 - M_A \mathbf{1}^\top x}{\mathbf{1}^\top z^1}
= \alpha(\mathbf{1}^\top y^1 - M_A \mathbf{1}^\top x)
= \mathbf{1}^\top S^1 - M_A \mathbf{1}^\top v .
\]
Moreover, the constraints in \eqref{eq:LMO_LFIP} transform into
\begin{align*}
& A v + M_A \alpha \mathbf{1} = S^0 + S^1, \\
& 0 \le S^0 \le 2M_A(\alpha \mathbf{1}-v), \\
& 0 \le S^1 \le 2M_A v, \\
& B v = U^0 + U^1, \\
& 0 \le U^0 \le M_B(\alpha \mathbf{1}-v), \\
& 0 \le U^1 \le M_B v, \\
& \mathbf{1}^\top U^1 = 1, \\
& 2 \le \mathbf{1}^\top x \le K.
\end{align*}
Finally, we linearize the bilinear relation $v=\alpha x$ via standard big-$M$ constraints: for all $i\in[N]$,
\[
0 \le v_i \le \alpha,\qquad
v_i \le \bar \alpha\, x_i,\qquad
v_i \ge \alpha - \bar \alpha(1-x_i),
\]
together with $x\in\{0,1\}^N$ and $0\le \alpha \le \bar\alpha$.

Collecting all constraints, we obtain the following equivalent $0$--$1$ MILP:
\begin{align*}
\min\quad & \mathbf{1}^\top S^1 - M_A \mathbf{1}^\top v
\numberthis \label{eq:LMO_MILP}
\\
\text{s.t.}\quad
& A v + M_A \alpha \mathbf{1} = S^0 + S^1, \\
& 0 \le S^0 \le 2M_A(\alpha \mathbf{1}-v), \\
& 0 \le S^1 \le 2M_A v, \\
& B v = U^0 + U^1, \\
& 0 \le U^0 \le M_B(\alpha \mathbf{1}-v), \\
& 0 \le U^1 \le M_B v, \\
& \mathbf{1}^\top U^1 = 1, \\
& 0 \le v_i \le \alpha,\qquad \forall i\in[N], \\
& v_i \le \bar \alpha\, x_i,\qquad \forall i\in[N], \\
& v_i \ge \alpha - \bar \alpha(1-x_i),\qquad \forall i\in[N], \\
& 2 \le \mathbf{1}^\top x \le K, \\
& x\in\{0,1\}^N,\qquad 0\le \alpha \le \bar\alpha,\\
& (v,S^0,S^1,U^0,U^1)\in\RR_+^{N}\times\RR_+^{N}\times\RR_+^{N}\times\RR_+^{N}\times\RR_+^{N}.
\end{align*}
This completes the reduction of the LMO subproblem to a $0$--$1$ MILP.

Consequently, since \eqref{eq:LMO_QFP} is exactly reformulated as the $0$--$1$ MILP
\eqref{eq:LMO_MILP}, any feasible solution returned by a standard MILP solver
based on techniques such as branch-and-bound and branch-and-cut, together with valid
primal and dual bounds, immediately yields a corresponding solution for
\eqref{eq:LMO_QFP} with a certified suboptimality guarantee
\citep{wolsey1999integer, wolsey2020integer}.
In particular, suppose the solver terminates with an incumbent feasible solution
$\hat x$ of objective value $\mathrm{UB}$ (upper bound) and a provable lower bound $\mathrm{LB}$ (lower bound)
on the optimal value, such that the certified (absolute) optimality gap satisfies
\[
\mathrm{UB}-\mathrm{LB}\le \epsilon_{\mathrm{LMO}} .
\]
Let $x^\star$ be an optimal solution to \eqref{eq:LMO_QFP}.
By exact equivalence of the MILP and \eqref{eq:LMO_QFP}, the MILP optimum equals the
optimal value of \eqref{eq:LMO_QFP}, and therefore
\[
\frac{(\hat x)^\top A \hat x}{(\hat x)^\top B \hat x}
-
\frac{(x^\star)^\top A x^\star}{(x^\star)^\top B x^\star}
\;\le\;
\mathrm{UB}-\mathrm{LB}
\;\le\;
\epsilon_{\mathrm{LMO}} .
\]
Hence, solving the MILP to a certified optimality gap of at most
$\epsilon_{\mathrm{LMO}}$ yields an $\epsilon_{\mathrm{LMO}}$-optimal solution to
the LMO problem \eqref{eq:LMO_QFP} (equivalently, \eqref{eq:LMO}).
This completes the proof of Theorem~\ref{thm:LMO_MILP}.
\end{proof}

\subsection{Tightening big-$M$ constants in \eqref{eq:LMO_LFIP}--\eqref{eq:LMO_MILP}}
\label{app:tight_bigM}

The proof in Appendix~\ref{app_subsec:proof_thm:LMO_MILP} provides an \emph{exact} reformulation of the LMO subproblem as a $0$--$1$ MILP, using the bounds $M_A$, $M_B$, and $\bar\alpha$, which are usually referred to as big-$M$ constants~\citep{wolsey1999integer, wolsey2020integer}.
In practice, however, the runtime of MILP solvers---through the strength of the LP relaxation and the effectiveness of branching---is often highly sensitive to how tight these bounds are.
Motivated by this, we derive tighter, problem-specific bounds that retain exact equivalence while typically leading to substantial computational speedups.

Throughout, we only require bounds that are valid over the feasible set
$\{x\in\{0,1\}^N: 2\le \mathbf{1}^\top x \le K\}$.
Replacing $(M_A,M_B,\bar\alpha)$ by any \emph{smaller} constants/vectors that remain valid upper bounds
does not remove any feasible $(x,y^0,y^1,z^0,z^1)$ that corresponds to a feasible $x$ of the original QFIP,
and therefore preserves the exact equivalence between \eqref{eq:LMO_QFP}, \eqref{eq:LMO_LFIP}, and \eqref{eq:LMO_MILP}.
We now describe concrete tighter choices.

\subsubsection{Tighter bounds for $(Ax)_i$ exploiting $\mathbf{1}^\top x\le K$}
\label{app:tight_MA}

The coarse bound $M_A:=\max_i\sum_j |A_{ij}|$ ignores sparsity.
Since any feasible $x$ has at most $K$ ones, we can tighten the bound on each coordinate $(Ax)_i$.

For each row $i\in[N]$, let $\{|A_{ij}|\}_{j=1}^N$ be sorted in nonincreasing order:
\[
|A_{i(1)}|\ge |A_{i(2)}|\ge \cdots \ge |A_{i(N)}|.
\]
Define the \emph{$K$-aware row bound}
\[
M_{A,i}^{(K)} := \sum_{\ell=1}^K |A_{i(\ell)}|
\qquad\text{and}\qquad
M_A^{(K)} := \max_{i\in[N]} M_{A,i}^{(K)}.
\]
Then for any feasible $x$ with $\mathbf{1}^\top x\le K$,
\[
|(Ax)_i|=\Big|\sum_{j=1}^N A_{ij}x_j\Big|
\le \sum_{j:x_j=1}|A_{ij}|
\le \sum_{\ell=1}^K |A_{i(\ell)}|
= M_{A,i}^{(K)}
\le M_A^{(K)}.
\]

\paragraph{Option 1: minimal-change replacement (scalar bound).}
One may simply replace $M_A$ by $M_A^{(K)}$ in \eqref{eq:lin_Ax_split} and everywhere thereafter:
\[
\tilde A x := Ax + M_A^{(K)}\mathbf{1}\in[0,2M_A^{(K)}]^N,
\qquad
0\le y^0\le 2M_A^{(K)}(\mathbf{1}-x),
\quad
0\le y^1\le 2M_A^{(K)}x,
\]
and the identity becomes
\[
x^\top A x = \mathbf{1}^\top y^1 - M_A^{(K)}\mathbf{1}^\top x.
\]
All subsequent steps remain unchanged with this substitution.

\paragraph{Option 2: Stronger replacement (row-wise bound).}
A further tightening uses row-dependent shifts.
Let $m_A\in\RR_+^N$ be given by $(m_A)_i:=M_{A,i}^{(K)}$ and define
\[
\tilde A x := Ax + m_A \in [0,\,2m_A]\quad\text{(componentwise)}.
\]
Then replace \eqref{eq:lin_Ax_split} by
\begin{align}
\label{eq:lin_Ax_split_tight}
\tilde A x = y^0 + y^1,\qquad
0 \le y^0 \le 2\,m_A\odot(\mathbf{1}-x),\qquad
0 \le y^1 \le 2\,m_A\odot x ,
\end{align}
where $\odot$ denotes elementwise product.
The same logic yields
\[
\mathbf{1}^\top y^1 = \sum_{i=1}^N x_i(\tilde Ax)_i
= x^\top(Ax+m_A)
= x^\top A x + m_A^\top x,
\quad\Longrightarrow\quad
x^\top A x = \mathbf{1}^\top y^1 - m_A^\top x .
\]
Consequently, the LFIP objective in \eqref{eq:LMO_LFIP} becomes
\[
\frac{\mathbf{1}^\top y^1 - m_A^\top x}{\mathbf{1}^\top z^1}
\quad\text{(instead of $\frac{\mathbf{1}^\top y^1 - M_A\mathbf{1}^\top x}{\mathbf{1}^\top z^1}$).}
\]
All other constraints remain identical after replacing $M_A$ by $m_A$ componentwise in the upper bounds.

\subsubsection{Tighter bounds for $(Bx)_i$ exploiting $B=ww^\top$ and $\mathbf{1}^\top x\le K$}
\label{app:tight_MB}

The coarse bound $M_B:=\max_i\sum_j B_{ij}$ is also loose.
Here we can exploit the rank-one structure $B=ww^\top$:
\[
(Bx)_i = \sum_{j=1}^N w_i w_j x_j = w_i\underbrace{\Big(\sum_{j=1}^N w_j x_j\Big)}_{=:W(x)} = w_i\,W(x).
\]
Let $w_{[1]}^\downarrow\ge \cdots \ge w_{[N]}^\downarrow$ be the weights sorted in decreasing order and define
\[
W_{\max}^{(K)} := \sum_{\ell=1}^K w_{[\ell]}^\downarrow.
\]
Then for any feasible $x$ with $\mathbf{1}^\top x\le K$,
\[
0\le W(x)\le W_{\max}^{(K)}
\qquad\Rightarrow\qquad
0\le (Bx)_i = w_iW(x)\le w_i W_{\max}^{(K)}.
\]
Define the componentwise bound vector $m_B\in\RR_+^N$ by
\[
(m_B)_i := w_i W_{\max}^{(K)}.
\]
Then one may replace \eqref{eq:lin_Bx_split} by the tighter, componentwise constraints
\begin{align}
\label{eq:lin_Bx_split_tight}
B x = z^0 + z^1,\qquad
0 \le z^0 \le m_B\odot(\mathbf{1}-x),\qquad
0 \le z^1 \le m_B\odot x .
\end{align}
Exactly as before, these enforce $\mathbf{1}^\top z^1=x^\top Bx$.
The same componentwise replacements propagate through the Charnes--Cooper step:
after scaling $(U^0,U^1)=\alpha(z^0,z^1)$ and $v=\alpha x$, we obtain
\[
0\le U^0 \le m_B\odot(\alpha\mathbf{1}-v),
\qquad
0\le U^1 \le m_B\odot v
\]
in place of the coarse bounds with $M_B$.

\subsubsection{Tighter upper bound on $\alpha$ exploiting $\mathbf{1}^\top x\ge 2$}
\label{app:tight_alpha}

In the proof above we used $w^\top x \ge 2w_{\min}$ to deduce
$\bar\alpha = 1/(4w_{\min}^2)$.
This can be tightened by using the exact minimal possible sum of two weights.

Let $w_{(1)}^\uparrow\le w_{(2)}^\uparrow\le \cdots \le w_{(N)}^\uparrow$ be the weights sorted in increasing order and define
\[
W_{\min}^{(2)} := w_{(1)}^\uparrow + w_{(2)}^\uparrow.
\]
Then the constraint $\mathbf{1}^\top x\ge 2$ implies $w^\top x \ge W_{\min}^{(2)}$, and therefore
\[
x^\top Bx = (w^\top x)^2 \ge \big(W_{\min}^{(2)}\big)^2,
\qquad\Rightarrow\qquad
\alpha=\frac{1}{\mathbf{1}^\top z^1}=\frac{1}{x^\top Bx}\le \frac{1}{\big(W_{\min}^{(2)}\big)^2}.
\]
Hence one may replace the coarse upper bound
$\bar\alpha=1/(4w_{\min}^2)$ by the tighter
\[
\bar\alpha_{\mathrm{tight}} := \frac{1}{\big(W_{\min}^{(2)}\big)^2}.
\]
This directly strengthens the big-$M$ linearization of $v=\alpha x$:
\[
v_i \le \bar\alpha_{\mathrm{tight}}\,x_i,
\qquad
v_i \ge \alpha-\bar\alpha_{\mathrm{tight}}(1-x_i),
\qquad \forall i\in[N].
\]
More generally, if the feasible set enforces a different lower cardinality $|S|\ge L_{\min}$,
one can use $W_{\min}^{(L_{\min})}:=\sum_{\ell=1}^{L_{\min}} w_{(\ell)}^\uparrow$ and obtain
$\bar\alpha_{\mathrm{tight}}=1/\big(W_{\min}^{(L_{\min})}\big)^2$.

%
\subsection{Proof of Proposition~\ref{prop:KW_MNL}}
\label{subsec:proof_prop:KW_MNL}
\begin{proof}[Proof of Proposition~\ref{prop:KW_MNL}]
    Fix $\thetab_0 \in \RR^d$.
    We apply the following Lemma~\ref{lemma:matrix_KW_thm} by indexing $[L]$ with assortments $S\in\Scal$ and choosing,
    for each $S\in\Scal$, a matrix $\Ab_S$ satisfying $\Ib_{\thetab_0}(S)=\Ab_S\Ab_S^\top$ (e.g., any matrix square-root factor).
    \begin{lemma} [Theorem~1 of~\citealt{mukherjee2024optimal}]
    \label{lemma:matrix_KW_thm}
        Let $M \geq 1$ be an integer and $\Ab_1, \dots, \Ab_L \in \RR^{d \times M}$ be $L$ matrices whose column space spans $\RR^d$. 
        For any probability distribution $\pi \in \Delta_L$ over $[L] := \{1,\ldots,L\}$, define the objective 
        \[
        g(\pi) \;\coloneqq\; \max_{i \in [L]} \operatorname{tr}\!\left( \Ab_i^\top \Vb_\pi^{-1} \Ab_i \right),
        \]
        where $\Vb_\pi \;\coloneqq\; \sum_{i=1}^L \pi_i \, \Ab_i \Ab_i^\top$ is a design matrix.
        Then, the following statements are equivalent:
        \begin{enumerate}[label=(\alph*)]
            \item $\pi^\star$ is a minimizer of $g(\pi)$.
        
            \item $\pi^\star$ is a maximizer of $f(\pi) = \log \operatorname{det}(\Vb_\pi)$.
    
            \item  $g(\pi^\star) = d$.
        \end{enumerate}
        Furthermore, there exists a minimizer $\pi^\star$ of $g(\pi)$ such that $|\operatorname{supp}(\pi^\star)|\leq d(d+1)/2$.
    \end{lemma}
    Under Assumption~\ref{assm:span}, the collection of columns of $\{\Ab_S\}_{S\in\Scal}$ spans $\RR^d$,
    so the spanning condition in Lemma~\ref{lemma:matrix_KW_thm} is satisfied.
    For any design $\pi$ over $\Scal$, the induced design matrix in Lemma~\ref{lemma:matrix_KW_thm} becomes
    \[
    \Vb_\pi
    =\; \EE_{S\sim\pi}\big[\Ab_S\Ab_S^\top\big]
    \;=\; \EE_{S\sim\pi}\big[\Ib_{\thetab_0}(S)\big]
    \;=\; \Mb_{\thetab_0}(\pi).
    \]
    Moreover, for each $S\in\Scal$,
    \[
    \tr\!\big(\Ab_S^\top \Vb_\pi^{-1}\Ab_S\big)
    = \tr\!\big(\Vb_\pi^{-1}\Ab_S\Ab_S^\top\big)
    = \tr\!\big(\Mb_{\thetab_0}(\pi)^{-1}\Ib_{\thetab_0}(S)\big),
    \]
    so the objective $g(\pi)$ in Lemma~\ref{lemma:matrix_KW_thm} coincides with $g_{\thetab_0}(\pi)$ in~\eqref{eq:g_opt}.
    Similarly, the D-optimal objective $f(\pi)=\log\det(\Vb_\pi)$ coincides with
    $f_{\thetab_0}(\pi)=\log\det(\Mb_{\thetab_0}(\pi))$.
    Therefore, applying Lemma~\ref{lemma:matrix_KW_thm} yields the equivalence of
    (a)--(c) in Proposition~\ref{prop:KW_MNL}.
    
    Finally, the support-size guarantee follows directly from the last statement of
    Lemma~\ref{lemma:matrix_KW_thm}, implying that there exists a minimizer $\pi^\star_{\thetab_0}$ of $g_{\thetab_0}(\pi)$
    such that $|\operatorname{supp}(\pi^\star_{\thetab_0})|\le d(d+1)/2$.
\end{proof}

\subsection{Proof of Proposition~\ref{prop:stopping_time_inexact_LMO_informal}}
\label{app_subsec:proof_prop_stopping}
%
\begin{algorithm}[h]
   \caption{Frank--Wolfe Algorithm for MNL at $\thetab_0$ with $\ErrorLMO$-approximate LMO}
   \label{alg:FW_MNL_approx}
    \begin{algorithmic}[1]
       \STATE {\bfseries Input:} nominal parameter $\thetab_0$,
       initial design $\pi_0$, precision $\epsilon$,
       and $0\le\ErrorLMO\le\epsilon d/2$.
       \STATE {\bfseries Initialize:}
       $\tilde\epsilon\gets\epsilon-\ErrorLMO/d\in(0,1]$,
       $\Mb_0\gets\Mb_{\thetab_0}(\pi_0)\succ0$, and $m\gets0$.
       \STATE Call the MILP oracle with certified gap at most $\ErrorLMO$
       to obtain $\hat S_0$; set
       $\hat g_0\gets\tr(\Mb_0^{-1}\Ib_{\thetab_0}(\hat S_0))$.
       \WHILE{$\hat g_m>(1+\tilde\epsilon)d$}
            \STATE $\gamma_m\gets\argmax_{\gamma\in[0,1]}
            f_{\thetab_0}((1-\gamma)\pi_m+\gamma\mathbbm{1}_{\hat S_m})$.
            \STATE $\pi_{m+1}\gets(1-\gamma_m)\pi_m+
            \gamma_m\mathbbm{1}_{\hat S_m}$;
            $\Mb_{m+1}\gets\Mb_{\thetab_0}(\pi_{m+1})$;
            $m\gets m+1$.
            \STATE Call the MILP oracle with certified gap at most $\ErrorLMO$
            to obtain $\hat S_m$; set
            $\hat g_m\gets\tr(\Mb_m^{-1}\Ib_{\thetab_0}(\hat S_m))$.
       \ENDWHILE
       \STATE Output $\widehat\pi_{\thetab_0}\gets\pi_m$.
    \end{algorithmic}
\end{algorithm}


\begin{proposition}[FW stopping-time guarantee with an approximate LMO, extended version of Proposition~\ref{prop:stopping_time_inexact_LMO_informal}]
\label{prop:stopping_time_inexact_LMO}
Fix $\thetab_0$.
Let $\pi_m$ be the Frank--Wolfe iterate and write
$\Mb_m:=\Mb_{\thetab_0}(\pi_m)$.
Define the certificate
$
g_m^\star := \max_{S\in\Scal} \tr\!\left(\Mb_m^{-1}\Ib_{\thetab_0}(S)\right),
$
and
$
\hat g_m :=  \tr\!\big(\Mb_m^{-1}\Ib_{\thetab_0}(\hat S_m)\big),
$
where $\hat S_m\in\Scal$ is returned by an approximate LMO at iteration $m$.
Assume the per-iteration LMO guarantee $g_m^\star-\hat g_m \le \ErrorLMO$.
Define 
the stopping time based on the \emph{computable} certificate
\begin{equation}
\label{eq:stop_tau}
\tau := \inf\Big\{m\ge 0:\ \hat g_m \le (1+\tilde{\epsilon})d\Big\}, \qquad \text{where }\, 
\tilde{\epsilon} := \epsilon - \frac{\ErrorLMO}{d}
.
\end{equation}
Given $\ErrorLMO$, choose $\epsilon$ such that
$0<\tilde\epsilon=\epsilon-\ErrorLMO/d\le1$ and
$\ErrorLMO\le d\tilde\epsilon$ (equivalently, $2\ErrorLMO\le\epsilon d$).
Suppose that, throughout the run up to the stopping time, the iterates satisfy
$\Mb_m \succeq \lambda_0 \Ib_d$ for some constant $ \lambda_0 >0$, 
and that $\|\Ib_{\thetab_0}(S) \|_2 \leq L$ for all $S$.
Then, we have
\begin{enumerate}[label=\textbf{(\roman*)}]
    \item \textbf{Certified accuracy at the stopping time.}
    At termination, the returned iterate $\pi_\tau$ satisfies 
    $ g_\tau^\star \le (1+\epsilon)d$.

    \item \textbf{Finite stopping and an iteration bound.}
    The algorithm terminates in finite time, and the stopping time satisfies
    \[
    \tau \;\le\;
    \BigO \!\left(
            \frac{d[1+\log(L/\lambda_0)]}{\tilde{\epsilon}}
    \right).
    \]
\end{enumerate}
\end{proposition}

\begin{proof}[Proof of Proposition~\ref{prop:stopping_time_inexact_LMO}]
    The first claim follows immediately from the computable stopping rule:
    \[
        g_\tau^\star\le\hat g_\tau+\ErrorLMO
        \le(1+\tilde\epsilon)d+\ErrorLMO=(1+\epsilon)d.
    \]

    We now prove the iteration bound.
    Write $f=f_{\thetab_0}$ and let $\pi^\star$ maximize $f$.
    Set $\delta_m:=f(\pi^\star)-f(\pi_m)$.
    For a nonterminal iteration $m$, write
    $\Ib_m:=\Ib_{\thetab_0}(\hat S_m)$ and
    $\Lb_m:=\Mb_m^{-1/2}\Ib_m\Mb_m^{-1/2}\succeq0$.
    Let $\lambda_1,\ldots,\lambda_d\ge0$ be its eigenvalues.
    For $\gamma\in[0,1]$, let
    $\pi_m(\gamma):=(1-\gamma)\pi_m+\gamma\mathbbm{1}_{\hat S_m}$.
    Then $\sum_i\lambda_i=\hat g_m$, and
    \begin{equation}
        f(\pi_m(\gamma))-f(\pi_m)
        =\sum_{i=1}^d\log\bigl(1+\gamma(\lambda_i-1)\bigr),
        \qquad 0\le\gamma<1.
        \label{eq:det_growth_key}
    \end{equation}
    Its derivative at zero is $\hat g_m-d>0$.
    For $0\le\gamma\le1/2$, its second derivative is bounded below by
    \[
        -4\sum_i(\lambda_i-1)^2\ge-4\hat g_m^2,
    \]
    since $\sum_i\lambda_i^2\le\hat g_m^2$ and $\hat g_m>d$.
    Evaluating at $\gamma=(\hat g_m-d)/(4\hat g_m^2)\le1/4$
    and using exact line search gives
    \begin{equation}
        f(\pi_{m+1})-f(\pi_m)
        \ge\frac{(\hat g_m-d)^2}{8\hat g_m^2}
        \ge\frac1{32}\min\left\{1,
        \frac{(\hat g_m-d)^2}{d^2}\right\}.
        \label{eq:progress_eps}
    \end{equation}

    Concavity of $f$ implies $\delta_m\le g_m^\star-d$.
    The balanced oracle tolerance and the failed stopping test imply
    \[
        \delta_m\le g_m^\star-d
        \le\hat g_m-d+\ErrorLMO
        \le2(\hat g_m-d),
    \]
    because $\ErrorLMO\le d\tilde\epsilon<\hat g_m-d$.
    Therefore, at every nonterminal iteration,
    \begin{align*}
        \delta_m-\delta_{m+1}
        &\ge\frac1{128}\min\{1,\delta_m^2/d^2\},\\
        \delta_m-\delta_{m+1}
        &\ge\frac{\tilde\epsilon^2}{32}.
        \numberthis\label{eq:f_lower}
    \end{align*}
    The second bound uses $0<\tilde\epsilon\le1$ in
    Equation~\eqref{eq:progress_eps}.

    There are at most $\lceil128\delta_0\rceil+1$ iterations
    with $\delta_m>d$.
    Once $0<\delta_m\le d$, the first bound in~\eqref{eq:f_lower}
    implies
    \[
        \frac1{\delta_{m+1}}-\frac1{\delta_m}
        \ge\frac1{128d^2}
    \]
    whenever $\delta_{m+1}>0$.
    Hence at most $\lceil128d/\tilde\epsilon\rceil+1$
    further iterations are needed to reach
    $\delta_m\le d\tilde\epsilon$ or to stop.
    From then on, the second bound in~\eqref{eq:f_lower}
    permits at most $\lceil32d/\tilde\epsilon\rceil+1$
    additional nonterminal iterations.
    Thus the algorithm terminates after
    $\mathcal O(\delta_0+d/\tilde\epsilon)$ iterations.

    Finally, $\Mb_0\succeq\lambda_0\Ib_d$ and
    $\Ib_{\thetab_0}(S)\preceq L\Ib_d$ give
    $\delta_0\le d\log(L/\lambda_0)$.
    Consequently,
    \[
        \tau=\mathcal O\!\left(d\log\frac L{\lambda_0}
        +\frac d{\tilde\epsilon}\right)
        =\mathcal O\!\left(
        \frac{d[1+\log(L/\lambda_0)]}{\tilde\epsilon}\right).
    \]
    This proves claim~\textbf{(ii)}.
\end{proof}

\subsection{Proof of Theorem~\ref{thm:liftedFW}}
\label{app_subsec:proof_thm:liftedFW}

\begin{algorithm}[h]
   \caption{Frank--Wolfe Algorithm for \emph{Lifted} $G$-optimal Design (MNL) at $\thetab_0$}
   \label{alg:FW_MNL_lifted}
    \begin{algorithmic}[1]
       \STATE {\bfseries Input:}
       nominal parameter $\thetab_0$,
       initial distribution $\pi_0\in\Delta(\Scal)$,
       FW precision $\varepsilon$.
       \STATE {\bfseries Initialize:}
       $\widetilde{\Mb}_0 = \widetilde{\Mb}_{\thetab_0}(\pi_0) \succ 0$,
       $m=0$.
       \STATE {\bfseries Define:}
       $\widetilde f_{\thetab_0}(\pi):=\log\det\!\big(\widetilde{\Mb}_{\thetab_0}(\pi)\big)$.

       \WHILE{$\widetilde g_{\thetab_0}(\pi_m) > (1+\varepsilon)(d+1)$}
            \STATE $S_m \leftarrow \argmax_{S\in\Scal}\ 
            \tr\!\Big(\widetilde{\Mb}_m^{-1}\,\widetilde{\Ib}_{\thetab_0}(S)\Big)$. 
            \label{algeq:lmo_lifted}

            \STATE $\gamma_m \leftarrow \argmax_{\gamma \in [0,1]}
            \widetilde f_{\thetab_0}
            \Big(
                    (1-\gamma)\pi_m + \gamma \mathbbm{1}_{S_m}
            \Big)$.
            \label{algeq:stepsize_lifted}

            \STATE $\pi_{m+1} = (1-\gamma_m)\pi_m + \gamma_m \mathbbm{1}_{S_m}$.

            \STATE Update $\widetilde{\Mb}_{m+1} = \widetilde{\Mb}_{\thetab_0}(\pi_{m+1})$
            and $m \leftarrow m+1$.
       \ENDWHILE

       \STATE Output the estimated lifted design $\widetilde{\pi}^\star_{\thetab_0} =\pi_m$.
    \end{algorithmic}
\end{algorithm}
\begin{theorem}[Lifted FW guarantee, restatement of Theorem~\ref{thm:liftedFW}]
\label{thm:liftedFW_appendix}
Fix $\thetab_0$ and run Frank--Wolfe (Algorithm~\ref{alg:FW_MNL}) to minimize
$\widetilde g_{\thetab_0}(\pi)$ using the polynomial-time lifted LMO~\eqref{eq:LMO_lifted_theta}.
Let $\hat\pi$ be the design returned at the stopping time
$\widetilde g_{\thetab_0}(\hat\pi)\ \le\ (1+\epsilon)(d+1)$, and define
\[
\ErrorLift
:=
\inf\big\{\varepsilon\ge 0:\ \Delta_{\thetab_0}(\hat\pi)\preceq \varepsilon\,\Mb_{\thetab_0}(\hat\pi)\big\}.
\]
Then
\[
g_{\thetab_0}(\hat\pi)
\ \le\ (1+\ErrorLift)\big((1+\epsilon)d+\epsilon\big)
\ \le\  2(1+\ErrorLift)(1+\epsilon)d.
\]
\end{theorem}

\begin{proof}[Proof of Theorem~\ref{thm:liftedFW}]
Throughout the proof, fix an arbitrary design $\pi\in\Delta(\Scal)$ and abbreviate
\[
\Mb := \Mb_{\thetab_0}(\pi)\in\RR^{d\times d},\quad
\widetilde{\Mb} := \widetilde{\Mb}_{\thetab_0}(\pi)\in\RR^{(d+1)\times(d+1)},\quad
\widetilde g := \widetilde g_{\thetab_0}(\pi),\quad
g := g_{\thetab_0}(\pi).
\]
Write the block decomposition of $\widetilde{\Mb}$ as
\[
\widetilde{\Mb}
=
\begin{bmatrix}
\bar{\Ab} & \bar{\bb}\\
\bar{\bb}^\top & 1
\end{bmatrix},
\qquad
\bar{\Ab}:=\EE_{S\sim\pi} \big[\bar{\Ab}_{\thetab_0}(S)\big],
\quad
\bar{\bb}:=\EE_{S\sim\pi} \big[\bar{\ab}_{\thetab_0}(S)\big].
\]
Define the Schur-complement matrix
\[
\widehat{\Mb}
:=\bar{\Ab}-\bar{\bb}\bar{\bb}^\top \in\RR^{d\times d},
\]
and recall the mismatch matrix
\[
\Delta_{\thetab_0}(\pi)
:=\widehat{\Mb}-\Mb \succeq 0.
\]
We note that $\Delta_{\thetab_0}(\pi)\succeq 0$ holds directly. 
Indeed, recall
\[
\Delta_{\thetab_0}(\pi)
= \widehat{\Mb}-\Mb
= \bar{\Ab}-\bar{\bb}\bar{\bb}^\top-\Big(\bar{\Ab}-\EE_{S\sim\pi} \big[\bar{\ab}_{\thetab_0}(S)\bar{\ab}_{\thetab_0}(S)^\top\big]\Big)
=
\EE_{S\sim\pi}\big[\bar{\ab}_{\thetab_0}(S)\bar{\ab}_{\thetab_0}(S)^\top\big]
-\bar{\bb}\bar{\bb}^\top,
\]
For any $u\in\RR^d$, we have
\begin{align*}
u^\top\Delta_{\thetab_0}(\pi)u
&=
\EE_{S\sim\pi} \big[(u^\top\bar{\ab}_{\thetab_0}(S))^2\big]
-\Big(\EE_{S\sim\pi}[u^\top\bar{\ab}_{\thetab_0}(S)]\Big)^2
=
\Var\!\!\,_{S\sim\pi} \big(u^\top\bar{\ab}_{\thetab_0}(S)\big)
\ \ge\ 0.
\end{align*}
Since $u$ is arbitrary, this implies $\Delta_{\thetab_0}(\pi)\succeq 0$.

Fix any $S\in\Scal$ and abbreviate
\[
\Ib(S):=\Ib_{\thetab_0}(S),\qquad
\widetilde{\Ib}(S):=\widetilde{\Ib}_{\thetab_0}(S).
\]
From~\eqref{eq:lifted_I_block_mnl}, we can write
\[
\widetilde{\Ib}(S)
=
\begin{bmatrix}
\bar{\Ab}_{\thetab_0}(S) & \bar{\ab}_{\thetab_0}(S)\\
\bar{\ab}_{\thetab_0}(S)^\top & 1
\end{bmatrix},
\qquad
\Ib(S)=\bar{\Ab}_{\thetab_0}(S)-\bar{\ab}_{\thetab_0}(S)\bar{\ab}_{\thetab_0}(S)^\top.
\]
Since the bottom-right block of $\widetilde{\Mb}$ equals $1$, the standard block inverse formula
yields
\begin{equation}
\label{eq:block_inverse_widetildeM}
\widetilde{\Mb}^{-1}
=
\begin{bmatrix}
\widehat{\Mb}^{-1} & -\widehat{\Mb}^{-1}\bar{\bb}\\
-\bar{\bb}^\top\widehat{\Mb}^{-1} & 1+\bar{\bb}^\top\widehat{\Mb}^{-1}\bar{\bb}
\end{bmatrix}.
\end{equation}
Now compute $\tr(\widetilde{\Mb}^{-1}\widetilde{\Ib}(S))$ using~\eqref{eq:block_inverse_widetildeM}.
Multiplying the blocks and taking trace gives
\begin{align*}
\tr\!\big(\widetilde{\Mb}^{-1}\widetilde{\Ib}(S)\big)
&=
\tr\!\big(\widehat{\Mb}^{-1}\bar{\Ab}_{\thetab_0}(S)\big)
-2\,\bar{\bb}^\top\widehat{\Mb}^{-1}\bar{\ab}_{\thetab_0}(S)
+\Big(1+\bar{\bb}^\top\widehat{\Mb}^{-1}\bar{\bb}\Big)\cdot 1\\
&=
1+\tr\!\big(\widehat{\Mb}^{-1}\bar{\Ab}_{\thetab_0}(S)\big)
-2\,\bar{\bb}^\top\widehat{\Mb}^{-1}\bar{\ab}_{\thetab_0}(S)
+\bar{\bb}^\top\widehat{\Mb}^{-1}\bar{\bb}.
\end{align*}
Add and subtract $\bar{\ab}_{\thetab_0}(S)^\top\widehat{\Mb}^{-1}\bar{\ab}_{\thetab_0}(S)$ to complete the square:
\begin{align*}
\tr\!\big(\widetilde{\Mb}^{-1}\widetilde{\Ib}(S)\big)
&=
1+\Big[\tr\!\big(\widehat{\Mb}^{-1}\bar{\Ab}_{\thetab_0}(S)\big)
-\bar{\ab}_{\thetab_0}(S)^\top\widehat{\Mb}^{-1}\bar{\ab}_{\thetab_0}(S)\Big]
+\big\|\bar{\ab}_{\thetab_0}(S)-\bar{\bb}\big\|_{\widehat{\Mb}^{-1}}^2\\
&=
1+\tr\!\big(\widehat{\Mb}^{-1}\Ib(S)\big)
+\big\|\bar{\ab}_{\thetab_0}(S)-\bar{\bb}\big\|_{\widehat{\Mb}^{-1}}^2.
\end{align*}
Therefore, for every $S\in\Scal$,
\begin{equation}
\label{eq:tildeg_minus1_ge_hatg_pointwise}
\tr\!\big(\widetilde{\Mb}^{-1}\widetilde{\Ib}(S)\big)-1
\ \ge\
\tr\!\big(\widehat{\Mb}^{-1}\Ib(S)\big),
\end{equation}
since the norm term is nonnegative.

On the other hand, we
define
\[
\widehat g_{\thetab_0}(\pi)
\;:=\;
\max_{S\in\Scal}\tr\!\big(\widehat{\Mb}^{-1}\Ib(S)\big).
\]
Then, taking maxima over $S$ in~\eqref{eq:tildeg_minus1_ge_hatg_pointwise} yields
\begin{equation}
\label{eq:tildeg_minus1_ge_hatg}
\widetilde g_{\thetab_0}(\pi)-1
\ \ge\
\widehat g_{\thetab_0}(\pi).
\end{equation}

Furthermore, 
by definition, $\widehat{\Mb}=\Mb+\Delta_{\thetab_0}(\pi)$ with $\Delta_{\thetab_0}(\pi)\succeq 0$,
so $\widehat{\Mb}\succeq \Mb$ and hence $\widehat{\Mb}^{-1}\preceq \Mb^{-1}$.
Moreover, if $\Delta_{\thetab_0}(\pi)\preceq \varepsilon\,\Mb$ for some $\varepsilon\ge 0$, then
\[
\widehat{\Mb}
=\Mb+\Delta_{\thetab_0}(\pi)
\ \preceq\ (1+\varepsilon)\Mb
\quad\Longrightarrow\quad
\Mb^{-1}\ \preceq\ (1+\varepsilon)\widehat{\Mb}^{-1}.
\]
Consequently, for every $S\in\Scal$, we get
\[
\tr\!\big(\Mb^{-1}\Ib(S)\big)
\ \le\ (1+\varepsilon)\tr\!\big(\widehat{\Mb}^{-1}\Ib(S)\big).
\]
Taking maxima over $S$ gives the sandwich
\begin{equation}
\label{eq:g_le_(1+eps)_hatg}
g_{\thetab_0}(\pi)
\ \le\ (1+\varepsilon)\,\widehat g_{\thetab_0}(\pi).
\end{equation}

Hence,
combining~\eqref{eq:tildeg_minus1_ge_hatg} and~\eqref{eq:g_le_(1+eps)_hatg} yields: if
$\Delta_{\thetab_0}(\pi)\preceq \varepsilon\,\Mb_{\thetab_0}(\pi)$, then
\begin{equation}
\label{eq:g_le_(1+eps)(tildeg-1)}
g_{\thetab_0}(\pi)
\ \le\ (1+\varepsilon)\big(\widetilde g_{\thetab_0}(\pi)-1\big).
\end{equation}

Now set $\pi=\hat\pi$ and let $\varepsilon=\ErrorLift$, where by definition
$\Delta_{\thetab_0}(\hat\pi)\preceq \ErrorLift\,\Mb_{\thetab_0}(\hat\pi)$.
Applying~\eqref{eq:g_le_(1+eps)(tildeg-1)} gives
\[
g_{\thetab_0}(\hat\pi)
\ \le\ (1+\ErrorLift)\big(\widetilde g_{\thetab_0}(\hat\pi)-1\big).
\]
Finally, the stopping rule $\widetilde g_{\thetab_0}(\hat\pi)\le (1+\epsilon)(d+1)$ implies
\[
\widetilde g_{\thetab_0}(\hat\pi)-1
\ \le\ (1+\epsilon)(d+1)-1
\ =\ (1+\epsilon)d+\epsilon,
\]
and hence
\[
g_{\thetab_0}(\hat\pi)
\ \le\ (1+\ErrorLift)\big((1+\epsilon)d+\epsilon\big)
\ \le\  2(1+\ErrorLift)(1+\epsilon)d,
\]
which complete the proof of Theorem~\ref{thm:liftedFW}.
\end{proof}

\subsection{Crude Bound for Lifted FW Error}
\label{app_subsec:bound_error_lift}
\begin{proposition}[Crude bound for $\ErrorLift$]
\label{prop:how_large_ErrorLift}
If $\Mb_{\thetab_0}(\hat\pi)\succeq \lambda_0\Ib_d$ for some $\lambda_0>0$, then
    \[
    \ErrorLift
    \|\Delta_{\thetab_0}(\hat\pi)\|_{\mathrm{op}}  \le \frac{1}{\lambda_0}.
    \]
\end{proposition}
\begin{proof}[Proof of Proposition~\ref{prop:how_large_ErrorLift}]
    The proof is direct:
    \[
        \ErrorLift
        \le \|\Mb_{\thetab_0}(\hat\pi)^{-1}\|_{\mathrm{op}}\,
        \|\Delta_{\thetab_0}(\hat\pi)\|_{\mathrm{op}}
        \le \frac{1}{\lambda_0}\,
        \|\Delta_{\thetab_0}(\hat\pi)\|_{\mathrm{op}}  \le \frac{1}{\lambda_0},
    \]
which concludes the proof.
\end{proof}
\begin{proposition}[Crude bound for $\ErrorLift$ under the outside-option MNL model]
\label{prop:outside_option_bound_eps_by_ratio}
    Under the outside-option MNL model at $\thetab_0$, the lifting error at the returned design $\hat\pi$ satisfies
    \[
    \ErrorLift\ \le\ \sup_{S\sim\hat\pi}\frac{1-p_0(S)}{p_0(S)}
    \ =\ \sup_{S\sim\hat\pi}\sum_{j\in S}\exp(\ab_j^\top\thetab_0).
    \]
    In particular, if $|S|\le K$ for all $S\in\Scal$ and $\max_{j\in[N]}\ab_j^\top\thetab_0\le B$, then
    \[
    \ErrorLift\ \le\ K e^B.
    \]
    \end{proposition}
    \begin{proof} [Proof of Proposition~\ref{prop:outside_option_bound_eps_by_ratio}]
    For any $S\in\Scal$, define $w_i:=\exp(\ab_i^\top\thetab_0)$ and the outside-option probabilities
    \[
    p_0(S):=\frac{1}{1+\sum_{j\in S}w_j},
    \qquad
    p_i(S):=\frac{w_i}{1+\sum_{j\in S}w_j}\quad(i\in S).
    \]
    Let
    \[
    \bar{\ab}(S):=\sum_{i\in S}p_i(S)\ab_i,\qquad
    \bar{\Ab}(S):=\sum_{i\in S}p_i(S)\ab_i \ab_i^\top,\qquad
    \Ib(S):=\bar{\Ab}(S)-\bar{\ab}(S) \bar{\ab}(S)^\top\succeq 0.
    \]
    Moreover, we define
    \[
    \Mb(\hat\pi):=\EE_{S\sim\hat\pi}[\Ib(S)],\qquad
    \Delta(\hat\pi):= 
    \EE_{S \sim \pi}
    \left[ \bar{\ab}(S) \bar{\ab}(S)^\top\right]
    - \EE_{S \sim \pi}\left[ \bar{\ab}(S) \right] \EE_{S \sim \pi}\left[ \bar{\ab}(S) \right]^\top
    \succeq 0,
    \]
    and
    $\ErrorLift:=\inf\big\{\varepsilon\ge 0:\ \Delta(\hat\pi)\preceq \varepsilon\,\Mb(\hat\pi)\big\}.$
    Then, by the PSD Cauchy--Schwarz inequality,
    \[
    \bar{\ab}(S) \bar{\ab}(S)^\top
    =\Big(\sum_{i\in S}p_i(S)\ab_i\Big)\Big(\sum_{i\in S}p_i(S)\ab_i\Big)^\top
    \preceq
    \Big(\sum_{i\in S}p_i(S)\Big)\Big(\sum_{i\in S}p_i(S)\ab_i \ab_i^\top\Big)
    =(1-p_0(S))\,\bar{\Ab}(S).
    \]
    Rearranging yields $\Ib(S)=\bar{\Ab}(S)-\bar{\ab}(S) \bar{\ab}(S)^\top\succeq p_0(S)\,\bar{\Ab}(S)$, hence
    $\bar{\Ab}(S)\preceq \frac{1}{p_0(S)}\Ib(S)$ and therefore
    \[
    \bar{\ab}(S) \bar{\ab}(S)^\top\ \preceq\ (1-p_0(S))\bar{\Ab}(S)\ \preceq\ \frac{1-p_0(S)}{p_0(S)}\,\Ib(S).
    \]
    Since $\Delta(\hat\pi)=\EE_{S \sim \hat{\pi}}[\bar{\ab}(S) \bar{\ab}(S)^\top]-
    \EE_{S \sim \hat{\pi}}[\bar{\ab}(S)] \Big(\EE_{S \sim \hat{\pi}}[\bar{\ab}(S)]\Big)^\top \!\!\preceq\, \EE[\bar{\ab}(S) \bar{\ab}(S)^\top]$, it follows that
    \[
    \Delta(\hat\pi)\ \preceq\ \EE_{S\sim\hat\pi}\Big[\frac{1-p_0(S)}{p_0(S)}\,\Ib(S)\Big].
    \]
    Let $\rho(S):=\frac{1-p_0(S)}{p_0(S)}$. Then
    \[
    \Delta(\hat\pi)\ \preceq\ \EE[\rho(S)\Ib(S)]
    \ \preceq\ \Big(\sup_{S\sim\hat\pi} \rho(S)\Big)\,\EE[\Ib(S)]
    =\Big(\sup_{S\sim\hat\pi} \rho(S)\Big)\,\Mb(\hat\pi),
    \]
    which implies $\ErrorLift\le \sup_{S\sim\hat\pi}\rho(S)$ by the definition of $\ErrorLift$.
    The specialization follows from $\rho(S)=\sum_{j\in S}\exp(\ab_j^\top\thetab_0)\le |S|e^B\le K e^B$.
    Combining this with Proposition~\ref{prop:how_large_ErrorLift}, we conclude the proof.
\end{proof}

\section{Supplementary Material for Section~\ref{sec:BAI_MNL}}
\label{app:sec:BAI_MNL}
\subsection{Proof of Theorem~\ref{thm:BSI_MNL}}
\label{app_subsec:proof_thm:BSI_MNL}
We begin by stating an extended version of Theorem~\ref{thm:BSI_MNL}, which additionally yields a potentially sharper bound in the uniform revenue setting ($r_i \equiv 1$), particularly when the assortment size $K$ is large.
In what follows, we prove Theorem~\ref{thm:BSI_MNL_extended} rather than Theorem~\ref{thm:BSI_MNL}, since the former is more comprehensive and subsumes the latter.
\begin{theorem}[Sample complexity, extended version of Theorem~\ref{thm:BSI_MNL}]
\label{thm:BSI_MNL_extended}
    Let $\delta\in(0,1)$, $\lambda=1$, and use the settings of
    Theorem~\ref{thm:BSI_MNL}.
    Let $\WarmupRounds$ be the warm-up length bounded in
    Lemma~\ref{lemma:exploration_phase_length}.
    Suppose that Assumption~\ref{assum:bounded_assumption} holds.
    Using the stopping criterion in~\eqref{eq:stopping_condtion_app},
    with probability at least $1-\delta$,
    Algorithm~\ref{alg:MNL_BAI}, when implemented with the MILP solver described in Subsection~\ref{subsec:lmo_milp}, returns the optimal assortment $S^\star$,
    and the stopping round $\tau$ (and total sample count up to a factor of two) satisfies
    \begin{align*}
    \tau =
        \begin{cases}
        \BigOTilde\!\left(
            d \log\!\left(N/\delta\right)
            \left(
                \frac{1}{\Gap^2}
                + \frac{1}{\kappa\,\Gap}
            \right)
            + \WarmupRounds
        \right)
        &\qquad \textup{(Non-uniform $r_i$, MILP solver)},\\[2mm]
        \BigOTilde\!\left(
            d \log(N/\delta)
            \left(
            \min \left\{ 1, \frac{e^{B}}{K} \right\}\frac{1}{\Gap^2}
            + \frac{1}{\kappa \Gap}
            \right)
            + \WarmupRounds
            \right),
        &\qquad \textup{(Uniform $r_i$, MILP solver )}
        \end{cases}
    \end{align*}
\end{theorem}
%
\paragraph{Loss and Hessian matrix.}
Let $\Dcal_t =  \{ (i_s, S_s) \}_{s=1}^t$.
Define $\Lcal_{\Dcal_t}(\thetab)$  as the negative log-likelihood of $\thetab$ with respect to data collected up to $t$, 
and the corresponding maximum likelihood estimate (MLE):
\[
    \Lcal_{\Dcal_t}(\thetab)
    :=\sum_{s=1}^{t} \ell_s(\thetab) + \frac{\lambda}{2} \|\thetab \|_2^2
    =
    - \sum_{s=1}^{t} \sum_{i \in S_s\cup\{0\}} y_{si} \log p(i | S_s, \thetab)
    + \frac{\lambda}{2} \|\thetab \|_2^2
    ,
    \quad
    \text{and }\,
    \widehat{\thetab}_t := \argmin_{\thetab \in \RR^d}  \Lcal_{\Dcal_t}(\thetab).
\]

Define the regularized Hessian matrix at $\thetab$ as follows:
\begin{align*}
    \Hb_t(\thetab) 
    &
    = \Hb_{\Dcal_t}(\thetab)
    := \sum_{s=1}^{t} \nabla^2 \ell_s(\thetab) + \lambda \Ib_d
    = \sum_{s=1}^{t}  \sum_{i \in S_s\cup\{0\}} p(i|S_s, \thetab) 
    \left( \ab_i - \bar{\ab}_{\thetab}(S_s) \right)
    \left( \ab_i -  \bar{\ab}_{\thetab}(S_s)  \right)^\top
    + \lambda \Ib_d
    \\
    &=  \sum_{s=1}^{t}
    \left( 
        \sum_{i \in S_s} p(i|S_s, \thetab) \ab_i \ab_i^\top
        - \sum_{i \in S_s} \sum_{j \in S_s} p(i|S_s, \thetab) p(j|S_s, \thetab)  \ab_i \ab_j^\top
    \right)
    + \lambda \Ib_d
    .
\end{align*}
%
\paragraph{Confidence bound and good event.}
The following proposition provides a high-probability estimation error bound for the regularized MLE and guarantees local stability of the empirical Hessian under a sufficiently small uncertainty condition.
\begin{lemma}[Fixed-design confidence for pairwise utilities]
\label{lemma:estimation_error}
    Given $\Dcal=\{(i_s,S_s)\}_{s=1}^T$, let
    $\widehat{\thetab}_{\Dcal}$ minimize the $\lambda$-regularized
    negative log-likelihood over $\RR^d$, where $\lambda>0$.
    Conditional on the assortments, assume that the choices are
    independent and follow the MNL model at $\thetab^\star$.
    Set $\ab_0=0$, and suppose that
    Assumption~\ref{assum:bounded_assumption} holds.
    For $\eta\in(0,1)$, write
    \[
        L_\eta:=\log\frac{4(N+1)^2}{\eta},\qquad
        h:=\max_{0\le i<j\le N}
        \|\ab_i-\ab_j\|_{\Hb_{\Dcal}(\thetab^\star)^{-1}}.
    \]
    If
    \begin{equation}
        h\sqrt{d\log9+L_\eta+\lambda B^2}\le\frac1{16},
        \label{eq:pairwise_burnin}
    \end{equation}
    then, with probability at least $1-\eta$, both of the following hold:
    \begin{enumerate}[label=(\roman*)]
        \item $\frac12\Hb_{\Dcal}(\thetab^\star)
        \preceq\Hb_{\Dcal}(\widehat{\thetab}_{\Dcal})
        \preceq2\Hb_{\Dcal}(\thetab^\star)$;
        \item for every $0\le i<j\le N$,
        $| (\ab_i-\ab_j)^\top
        (\widehat{\thetab}_{\Dcal}-\thetab^\star)|
        \le (36\sqrt{L_\eta}+64\sqrt\lambda B)
        \|\ab_i-\ab_j\|_{\Hb_{\Dcal}(\thetab^\star)^{-1}}$.
    \end{enumerate}
\end{lemma}
The proof is given in Appendix~\ref{app:pairwise_confidence}.
The itemwise bound follows by taking $j=0$.
In particular, this implies Lemma~\ref{lemma:estimation_error_main},
whose burn-in condition is stronger than~\eqref{eq:pairwise_burnin}.

Define the confidence multiplier by
\begin{align}
    \beta(\eta):=36\sqrt{\log\frac{4(N+1)^2}{\eta}}
    +64\sqrt\lambda B,\qquad \eta\in(0,1).
    \label{eq:beta}
\end{align}
For the main rounds, use
$\delta_t:=\delta/[8(t+1)^2]$,
$L_t:=\log(4(N+1)^2/\delta_t)$, and
$\beta_t:=\beta(\delta_t)$.
These are the same quantities used in Step~4 of the algorithm.

Let $\WarmupRounds:=|\Dcal_w|=|\Dcal_w'|$ denote the number of
warm-up rounds. Each warm-up round uses two independent choices,
so the total number of choice observations at stopping round $\tau$
is $\tau+\WarmupRounds\le2\tau$.
Define
\[
    \Ecal_{\mathrm w}:=
    \left\{\max_{i\in[N]}
    |\ab_i^\top(\thetab_0-\thetab^\star)|\le\frac14\right\}.
\]
For $t\ge\WarmupRounds$, let $\Ecal_t$ be the event that either
condition~\eqref{eq:confidence_check} fails, or
\[
    |(\ab_i-\ab_j)^\top(\widehat{\thetab}_t-\thetab^\star)|
    \le\beta_t\|\ab_i-\ab_j\|_{\Hb_t(\thetab^\star)^{-1}}
    \qquad(0\le i<j\le N).
\]
We use the simultaneous event
\begin{equation}
    \Ecal:=\Ecal_{\mathrm w}\cap
    \bigcap_{t\ge\WarmupRounds}\Ecal_t.
    \label{eq:event}
\end{equation}
Only times satisfying~\eqref{eq:confidence_check} are considered
for stopping. Thus no confidence claim is needed at the other times.

\begin{lemma}[Exploration length and simultaneous confidence]
\label{lemma:exploration_phase_length}
    Set $\lambda=1$ and use the warm-up threshold in Step~1.
    Choose all warm-up assortments without using either feedback stream.
    There is an absolute constant $c_0>0$ such that
    \[
        \WarmupRounds\le
        c_0\kappa^{-1}d
        \left(d\log\frac{16(N+1)^2}{\delta}\vee B^2\right)
        \log\!\left(1+\frac{K}{\kappa}
        \left[d\log\frac{16(N+1)^2}{\delta}\vee B^2\right]\right).
    \]
    Moreover, $\PP(\Ecal)\ge1-\delta/2$.
    On $\Ecal_{\mathrm w}$, for every feasible assortment $S$
    and every $i\in S\cup\{0\}$,
    \begin{equation}
        e^{-1/2}\le
        \frac{p(i\mid S,\thetab^\star)}{p(i\mid S,\thetab_0)}
        \le e^{1/2},
        \qquad
        e^{-1/2}\Hb_t(\thetab_0)
        \preceq\Hb_t(\thetab^\star)
        \preceq e^{1/2}\Hb_t(\thetab_0).
        \label{eq:p_star_bound}
    \end{equation}
\end{lemma}
The proof is deferred to Appendix~\ref{app:exploration_phase_length}.

\paragraph{Optimal design.}
After the warm-up rounds, we compute the warm-up parameter as
\[
    \thetab_0 \leftarrow \argmin_{\thetab \in \RR^d} \Lcal_{\Dcal_w'}(\thetab),
\]
where $\Dcal_w'$ contains the same assortments as $\Dcal_w$ (which will be continually updated during the main rounds), but with independently generated choice feedback for each assortment.

Then, given the warm-up parameter $\thetab_0$, we compute the optimal design $\widehat{\pi}_{\thetab_0}$ by running the Frank–Wolfe algorithm (Algorithm~\ref{alg:FW_MNL}).
Note that this design can be computed either using a \textit{$0$–$1$ MILP solver} (as in Subsection~\ref{subsec:lmo_milp}) or via a relaxed but polynomial-time \textit{lifted G-optimal design} formulation (as in Subsection~\ref{subsec:lifted_lmo_mnl}).

\paragraph{Optimistic and pessimistic revenue.}
Recall the warm-up parameter $\thetab_0 \leftarrow \argmin_{\thetab \in \RR^d} \Lcal_{\Dcal_w'}(\thetab)$.
Then, we define the optimistic/pessimistic utilities at time $t$ as:
\begin{align*}
    u^+_{t,i}=\ab_i^\top\widehat{\thetab}_{t} + \sqrt{2} \beta_t \|\ab_i\|_{\Hb_{t}(\thetab_0)^{-1}},
    \qquad
    u^-_{t,i}=\ab_i^\top\widehat{\thetab}_{t} - \sqrt{2} \beta_t \|\ab_i\|_{\Hb_{t}(\thetab_0)^{-1}}.
\end{align*}

Moreover, for any $S \in \Scal$, define the optimistic/pessimistic expected revenues at time $t$ as:
\begin{align*}
    \widetilde R_t(S)
    :=
    \frac{\sum_{i\in S} \exp(u^+_{t,i}) r_i }{1+\sum_{j\in S} \exp(u^+_{t,j})},
    \qquad
    \widecheck{R}_t(S)
    :=
    \frac{\sum_{i\in S} \exp(u^-_{t,i}) r_i }{1+\sum_{j\in S} \exp(u^-_{t,j})}.
\end{align*}

At each round $t$, we select the pessimistic best assortment and the optimistic best alternative as
\[
\BestS_t \in \argmax_{S \in \Scal} \widecheck R_t(S),
\qquad
\AltS_t \in \argmax_{S \in \Scal:\, S \neq \BestS_t} \widetilde R_t(S).
\]

\begin{lemma}[Optimistic and pessimistic revenue gaps]
\label{lemma:revenue_bound}
    On the event $\Ecal$, at any time $t\ge\WarmupRounds$
    satisfying~\eqref{eq:confidence_check}, both inequalities
    \begin{align*}
        |R(S,\thetab^\star)-\widecheck R_t(S)|&\le\gamma_t(S),\\
        |R(S,\thetab^\star)-\widetilde R_t(S)|&\le\gamma_t(S)
    \end{align*}
    hold simultaneously for every $S\in\Scal$, where
    \begin{align}
        \gamma_t(S):={}
        2\sqrt2 e^{9/4}\beta_t
        \sqrt{\Var(r_i\mid S,\thetab^\star)}
        \sqrt{\max_{S'\in\Scal}\tr\!\left(
        \Hb_t(\thetab_0)^{-1}\Ib_{\thetab_0}(S')\right)}
        +20\beta_t^2\max_{i\in S}
        \|\ab_i\|_{\Hb_t(\thetab_0)^{-1}}^2.
        \label{eq:revenue_radius}
    \end{align}
    For $S=\varnothing$, the maximum and both revenue gaps are zero.
\end{lemma}
The proof is deferred to Appendix~\ref{app_subsubsec:proof_of_lemma:revenue_bound}.

\begin{lemma}[Variance bound for revenue parameters]
\label{lemma:bound_var}
    For any $S \in \Scal$, the following bounds hold:
    \begin{enumerate}[label=(\roman*)]
        \item 
        (\textbf{Non-uniform revenues})
        For arbitrary revenue parameters $(r_i)_{i\in[N]}$, we have
            $\Var(r_i | S,\thetab^\star)
            \le
            R(S^\star,\thetab^\star) \leq 1.$

        \item 
        (\textbf{Uniform revenues})
        If $r_i=1$ for all $i\in[N]$ and $S\ne\varnothing$, then
            $\Var(r_i | S,\thetab^\star)
            \le
            \sum_{i \in S} p(i | S,\thetab^\star)\, p(0 | S,\thetab^\star)
            \leq \min \left\{1, \frac{e^{B}}{|S|} \right\}
            $.
    \end{enumerate}
\end{lemma}
The proof is deferred to Appendix~\ref{app_subsubsec:proof_of_lemma:bound_var}.

\begin{lemma}
\label{lemma:opt_design}
    Set $\lambda=1$ and suppose that the fixed design satisfies
    $g_{\thetab_0}(\widehat\pi_{\thetab_0})\le(1+\epsilon)d$
    for some $\epsilon\ge0$.
    Conditional on the nominal-data stream, with probability at least
    $1-\delta/4$, the following holds simultaneously for every
    $t>\WarmupRounds$:
    \begin{enumerate}[label=(\roman*)]
        \item $\displaystyle
        \max_{S\in\Scal}\tr\!\left(
        \Hb_t(\thetab_0)^{-1}\Ib_{\thetab_0}(S)\right)
        \le\frac{48(1+\epsilon)d\log(2d/\delta_t)}
        {t-\WarmupRounds}$;
        \item on $\Ecal_{\mathrm w}$, $\displaystyle
        \max_{i\in[N]}\|\ab_i\|_{\Hb_t(\thetab_0)^{-1}}^2
        \le\frac{48e(1+\epsilon)d\log(2d/\delta_t)}
        {\kappa(t-\WarmupRounds)}$.
    \end{enumerate}
    Here $\delta_t=\delta/[8(t+1)^2]$ as above.
\end{lemma}
The proof is deferred to Appendix~\ref{app_subsubsec:proof_of_lemma:opt_design}.

\paragraph{Stopping rule.}
We stop at the first time $\tau$ at which~\eqref{eq:confidence_check} holds and
\begin{align}
    \widecheck R_\tau(\BestS_\tau) \;>\; \widetilde R_\tau(\AltS_{\tau}),
    \label{eq:stopping_condtion_app}
\end{align}
and output $\BestS_\tau$.
Then, the following lemma shows that $\BestS_{\tau} = S^\star$.
\begin{lemma} 
\label{lemma:S_tau_on_stopping}
    On the event $\Ecal$, 
    if~\eqref{eq:confidence_check} holds and
    $\widecheck R_\tau(\BestS_\tau) > \widetilde R_\tau(\AltS_{\tau})$,
    then the selected assortment at $\tau$ is optimal, i.e.,  
    $\BestS_{\tau} = S^\star$.
\end{lemma}
The proof is deferred to Appendix~\ref{app_subsubsec:proof_of_lemma:S_tau_on_stopping}.

We are now ready to provide the proof of Theorem~\ref{thm:BSI_MNL_extended}.
\begin{proof}[Proof of Theorem~\ref{thm:BSI_MNL_extended}]
    Work on the event $\Ecal$ and the simultaneous matrix event in
    Lemma~\ref{lemma:opt_design}.
    Their intersection has probability at least
    $1-\delta/2-\delta/4\ge1-\delta$.
    We first verify that the confidence check does not change the
    claimed order, then prove that sufficiently accurate revenues
    force stopping.

    \textbf{Step 1: Verifying the confidence check. }\,
    On $\Ecal_{\mathrm w}$, the warm-up threshold and
    Equation~\eqref{eq:p_star_bound} give, for all $t\ge\WarmupRounds$,
    \begin{align*}
        &(d\log9+B^2)
        \max_i\|\ab_i\|_{\Hb_t(\thetab_0)^{-1}}^2\\
        &\quad\le(d\log9+B^2)e^{1/2}\kappa^{-1}\WarmupThreshold^2
        \le\frac{2\sqrt e}{1024^2}<\frac1{2\cdot256^2}.
    \end{align*}
    Here $\log(16(N+1)^2/\delta)>\log9$.
    Since revenues lie in $[0,1]$, we have $\Gap\le1$.
    Therefore, whenever
    \begin{equation}
        20\beta_t^2\max_i\|\ab_i\|_{\Hb_t(\thetab_0)^{-1}}^2
        <\Gap/8,
        \label{eq:BSI_quadratic_small}
    \end{equation}
    the inequality $\beta_t^2\ge36^2L_t$ also gives
    \[
        L_t\max_i\|\ab_i\|_{\Hb_t(\thetab_0)^{-1}}^2
        <\frac1{160\cdot36^2}<\frac1{2\cdot256^2}.
    \]
    Adding the preceding two bounds proves
    condition~\eqref{eq:confidence_check}.
    Thus the same lower-order sample requirement used below already
    ensures that the confidence lemma can be applied.

    \textbf{Step 2: Stopping from uniform revenue accuracy. }\,
    Suppose that~\eqref{eq:BSI_quadratic_small} holds and that
    $\gamma_t(S)<\Gap/4$ for $S=S^\star,\BestS_t,\AltS_t$.
    By Step~1, Lemma~\ref{lemma:revenue_bound} applies.
    If $\BestS_t\ne S^\star$, then
    \begin{align*}
        \widecheck R_t(S^\star)-\widecheck R_t(\BestS_t)
        &\ge R(S^\star,\thetab^\star)-R(\BestS_t,\thetab^\star)
        -\gamma_t(S^\star)-\gamma_t(\BestS_t)\\
        &>\Gap/2>0,
    \end{align*}
    contradicting the definition of $\BestS_t$.
    Hence $\BestS_t=S^\star$, and
    \begin{align*}
        \widecheck R_t(\BestS_t)-\widetilde R_t(\AltS_t)
        &\ge R(S^\star,\thetab^\star)-R(\AltS_t,\thetab^\star)
        -\gamma_t(S^\star)-\gamma_t(\AltS_t)\\
        &>\Gap/2>0.
    \end{align*}
    This proves stopping by time $t$, without assuming that the
    algorithm has already stopped or that its incumbent is optimal.
    Correctness at any earlier stopping time follows from
    Lemma~\ref{lemma:S_tau_on_stopping}.

    \textbf{(i) Non-uniform revenues. }\,
    Write $n=t-\WarmupRounds$.
    Lemmas~\ref{lemma:bound_var} and~\ref{lemma:opt_design} give
    \begin{align*}
        \max_{S\in\Scal}\gamma_t(S)
        &\le2\sqrt2 e^{9/4}\beta_t
        \sqrt{\frac{48(1+\epsilon)d\log(2d/\delta_t)}{n}}
        +
        \frac{960e\beta_t^2(1+\epsilon)d\log(2d/\delta_t)}{\kappa n}.
        \numberthis\label{eq:uniform_tau}
    \end{align*}
    For a sufficiently large absolute constant $C$, the condition
    \begin{equation}
        n\ge C(1+\epsilon)d\beta_t^2\log(2d/\delta_t)
        \left(\frac1{\Gap^2}+\frac1{\kappa\Gap}\right)
        \label{eq:sufficient_condition}
    \end{equation}
    makes each term in~\eqref{eq:uniform_tau} smaller than $\Gap/8$.
    It also implies~\eqref{eq:BSI_quadratic_small}.
    By Steps~1--2, the algorithm must therefore stop by this time.

    Since $\delta_t=\delta/[8(t+1)^2]$,
    $\beta_t^2=\mathcal O(\log((N+1)(t+1)/\delta)+B^2)$.
    The right-hand side of~\eqref{eq:sufficient_condition} grows
    only logarithmically in $t$, apart from fixed instance factors.
    Solving this inequality, and using fixed design accuracy
    $\epsilon\in(0,1]$, yields
    \[
        \tau=\BigOTilde\!\left(
        d\left[\log\frac{N+1}{\delta}+B^2\right]
        \left(\frac1{\Gap^2}+\frac1{\kappa\Gap}\right)
        +\WarmupRounds\right).
    \]
    For fixed $B$ and $N\ge2$, this is the displayed non-uniform
    bound in the theorem. The same order holds for the total number
    of observations, since each warm-up round uses only two choices.

    \textbf{(ii) Uniform revenues. }\,
    Suppose that $r_i=1$ for every real item.
    The true and pessimistic revenues are strictly increased by
    adding an item, so $|S^\star|=|\BestS_t|=K$.
    If $K\ge2$, every maximizing alternative has size at least $K-1$:
    an assortment of size at most $K-2$ can be enlarged by one item
    without equaling the excluded size-$K$ incumbent.
    In particular, this argument also covers $K=N$, when the
    best alternative can have size $K-1$.
    By Lemma~\ref{lemma:bound_var},
    \[
        \Var(r_i\mid S,\thetab^\star)
        \le2\min\left\{1,\frac{e^B}{K}\right\},
        \qquad S\in\{S^\star,\BestS_t,\AltS_t\}.
    \]
    When $K=1$, the same bound follows from variance at most one;
    the empty alternative, if present, has zero variance.
    Thus the preceding argument applies with the squared first-term
    coefficient multiplied by $\min\{1,e^B/K\}$, up to an absolute
    factor of two.
    The second-term requirement, which also verifies the confidence
    check, is unchanged. Consequently,
    \[
        \tau=\BigOTilde\!\left(
        d\left[\log\frac{N+1}{\delta}+B^2\right]
        \left(\min\left\{1,\frac{e^B}{K}\right\}\frac1{\Gap^2}
        +\frac1{\kappa\Gap}\right)+\WarmupRounds\right).
    \]
    This gives the displayed uniform-revenue bound for fixed $B$.
    More generally, the same proof applies to any fixed design with
    $g_{\thetab_0}(\widehat\pi_{\thetab_0})\le c d$, $c\ge1$:
    multiply the gap-dependent terms by $c$, leaving $\WarmupRounds$
    unchanged.
    The common probability event already has probability at least
    $1-\delta$, so no additional allocation of $\delta$ is needed.
    This completes the proof.
\end{proof}

\subsection{Proofs of Corollary~\ref{cor:BSI_MNL_lifted}}
\label{app_subsec:cor:BSI_MNL_lifted}
In this subsection, we provide the proof of Corollary~\ref{cor:BSI_MNL_lifted}.
As in Appendix~\ref{app_subsec:proof_thm:BSI_MNL}, we first state an extended version of Corollary~\ref{cor:BSI_MNL_lifted}, which also encompasses the uniform revenue case.
\begin{corollary}[Sample complexity, extended version of Corollary~\ref{cor:BSI_MNL_lifted}]
\label{cor:BSI_MNL_lifted_extended}
    Let $\delta\in(0,1)$, $\lambda=1$, and use the settings of
    Theorem~\ref{thm:BSI_MNL}.
    Let $\WarmupRounds$ be the warm-up length bounded in
    Lemma~\ref{lemma:exploration_phase_length}.
    Suppose that Assumption~\ref{assum:bounded_assumption} holds.
    Using the stopping criterion in~\eqref{eq:stopping_condtion_app},
    with probability at least $1-\delta$,
    Algorithm~\ref{alg:MNL_BAI}, when implemented with the lifted G-optimal design described in Subsection~\ref{subsec:lifted_lmo_mnl},
    returns the optimal assortment $S^\star$,
    and the stopping round $\tau$ (and total sample count up to a factor of two) satisfies
    \begin{align*}
    \tau =
        \begin{cases}
        \BigOTilde\!\left(
            (1 + \ErrorLift) d \log\!\left(N/\delta\right)
            \left(
                \frac{1}{\Gap^2}
                + \frac{1}{\kappa\,\Gap}
            \right)
            + \WarmupRounds
        \right)
        &\qquad \textup{(Non-uniform $r_i$, Lifted G-optimal)},\\[2mm]
        \BigOTilde\!
         \left(
            (1 + \ErrorLift) d \log(N/\delta)  
            \left( 
                \min \left\{ 1, \frac{e^{B}}{K} \right\}
                \frac{1}{\Gap^2}
                + \frac{1}{\kappa \Gap}
            \right)
            + \WarmupRounds
         \right),
        &\qquad \textup{(Uniform $r_i$, Lifted G-optimal)}.
        \end{cases}
    \end{align*}
\end{corollary}
\begin{proof} [Proof of Corollary~\ref{cor:BSI_MNL_lifted_extended}]
    The result follows immediately by replacing the Frank–Wolfe approximation factor $(1+\epsilon)$ in Theorem~\ref{thm:BSI_MNL_extended} with $2(1+\ErrorLift)(1+\epsilon)$ according to Theorem~\ref{thm:liftedFW}.
\end{proof}

\subsection{Proofs of Lemmas for Theorem~\ref{thm:BSI_MNL}}
\label{app_subsec:thm:BSI_MNL_lemma}
\subsubsection{Proof of Lemma~\ref{lemma:estimation_error}}
\label{app:pairwise_confidence}

\begin{proof}[Proof of Lemma~\ref{lemma:estimation_error}]
    Condition on the query sequence, and write
    $\Hb:=\Hb_{\Dcal}(\thetab^\star)$ and
    $\ub:=\widehat{\thetab}_{\Dcal}-\thetab^\star$.
    If $h=0$, all features equal $\ab_0=0$, and the claims are immediate.
    Suppose henceforth that $h>0$.

    \textbf{Step 1: Concentration of the score. }\,
    Let $X_s:=\Hb^{-1/2}(\ab_{i_s}-\bar\ab_{\thetab^\star}(S_s))$.
    These vectors are independent and centered, with
    \[
        \sum_{s=1}^T\EE[X_sX_s^\top]
        =\Hb^{-1/2}(\Hb-\lambda\Ib_d)\Hb^{-1/2}
        \preceq\Ib_d,\qquad \|X_s\|_2\le h.
    \]
    The last inequality follows by writing
    $\ab_{i_s}-\bar\ab_{\thetab^\star}(S_s)$ as a convex
    combination of the differences $\ab_{i_s}-\ab_j$,
    including $j=0$.
    Scalar Bernstein's inequality gives, for every fixed unit vector $u$,
    \[
        \PP\!\left(\left|u^\top\sum_sX_s\right|
        >\sqrt{2x}+\frac23hx\right)\le2e^{-x}.
    \]
    Apply this bound with $x=L_\eta$ to all nonzero normalized
    directions $\Hb^{-1/2}(\ab_i-\ab_j)$.
    Also apply it with $x=d\log9+\log(4/\eta)$ to a $1/4$-net
    of the unit sphere with at most $9^d$ points.
    The net inequality contributes a factor $4/3$.
    A union bound, together with
    $\|\lambda\thetab^\star\|_{\Hb^{-1}}\le\sqrt\lambda B$,
    shows that, with probability at least $1-\eta$,
    \begin{align*}
        \|\nabla\Lcal_{\Dcal}(\thetab^\star)\|_{\Hb^{-1}}
        &\le\frac43\left(\sqrt{2x}+\frac23hx\right)+\sqrt\lambda B,
        \\[-1pt]
        |(\ab_i-\ab_j)^\top\Hb^{-1}
        \nabla\Lcal_{\Dcal}(\thetab^\star)|
        &\le\left(\sqrt{2L_\eta}+\frac23hL_\eta+\sqrt\lambda B\right)
        \|\ab_i-\ab_j\|_{\Hb^{-1}}.
        \numberthis\label{eq:pairwise_score}
    \end{align*}
    Indeed, the two failure probabilities are at most $\eta/4$
    and $\eta/2$, respectively. We work on their intersection.

    \textbf{Step 2: Localization of the MLE. }\,
    Since $x\le d\log9+L_\eta$, condition~\eqref{eq:pairwise_burnin}
    implies
    \begin{align*}
        h\|\nabla\Lcal_{\Dcal}(\thetab^\star)\|_{\Hb^{-1}}
        &\le\frac{\sqrt2}{12}+\frac1{288}+\frac1{16}<\frac14.
        \tag{Equation~\eqref{eq:pairwise_score}}
    \end{align*}
    A utility perturbation with pairwise range at most $a$
    changes every choice probability by a factor in $[e^{-a},e^a]$.
    The variance representation
    $\Var_p(Z)=\min_c\sum_i p_i(Z_i-c)^2$ therefore gives
    the same comparison for Fisher information and its regularized sum.
    Consequently,
    \[
        \nabla^2\Lcal_{\Dcal}(\thetab^\star+s\ub)
        \succeq e^{-sh\|\ub\|_{\Hb}}\Hb,
        \qquad 0\le s\le1.
    \]
    If $\ub\ne0$, first-order optimality and integration give
    \begin{align*}
        \frac{\|\ub\|_{\Hb}}{h}
        (1-e^{-h\|\ub\|_{\Hb}})
        &\le\int_0^1\ub^\top
        \nabla^2\Lcal_{\Dcal}(\thetab^\star+s\ub)\ub\,ds\\
        &=-\ub^\top\nabla\Lcal_{\Dcal}(\thetab^\star)\\
        &\le\|\ub\|_{\Hb}
        \|\nabla\Lcal_{\Dcal}(\thetab^\star)\|_{\Hb^{-1}}.
    \end{align*}
    Hence
    \begin{equation}
        \max_{0\le i<j\le N}|(\ab_i-\ab_j)^\top\ub|
        \le h\|\ub\|_{\Hb}\le\log(4/3).
        \label{eq:pairwise_localization}
    \end{equation}
    This also holds when $\ub=0$ and already proves (i),
    since $3\Hb/4\preceq\Hb_{\Dcal}(\widehat{\thetab}_{\Dcal})
    \preceq4\Hb/3$.

    \textbf{Step 3: Directional control. }\,
    Write $m:=\max_{0\le i<j\le N}|(\ab_i-\ab_j)^\top\ub|$ and
    \[
        E:=\int_0^1\nabla^2\Lcal_{\Dcal}(\thetab^\star+s\ub)\,ds-\Hb.
    \]
    The same probability-ratio comparison yields
    \[
        \|\Hb^{-1/2}E\Hb^{-1/2}\|_{\mathrm{op}}
        \le e^m-1\le\frac43m.
    \]
    By~\eqref{eq:pairwise_localization}, it follows that
    \begin{align*}
        \|\ub\|_{\Hb}\|\Hb^{-1/2}E\Hb^{-1/2}\|_{\mathrm{op}}
        &\le\frac{4\log(4/3)}{3h}m
        \le\frac{m}{2h}.
    \end{align*}
    Since $\ub=-\Hb^{-1}\nabla\Lcal_{\Dcal}(\thetab^\star)
    -\Hb^{-1}E\ub$, Cauchy--Schwarz and~\eqref{eq:pairwise_score} give
    \begin{align*}
        |(\ab_i-\ab_j)^\top\ub|
        &\le\left[\sqrt{2L_\eta}+\frac23hL_\eta
        +\sqrt\lambda B+\frac{m}{2h}\right]
        \|\ab_i-\ab_j\|_{\Hb^{-1}}.
    \end{align*}
    Taking the maximum over pairs and rearranging bounds $m$ by
    $2h(\sqrt{2L_\eta}+\frac23hL_\eta+\sqrt\lambda B)$.
    Substituting this back gives
    \begin{align*}
        |(\ab_i-\ab_j)^\top\ub|
        &\le2\left(\sqrt{2L_\eta}+\frac23hL_\eta+\sqrt\lambda B\right)
        \|\ab_i-\ab_j\|_{\Hb^{-1}}\\
        &\le(3\sqrt{L_\eta}+2\sqrt\lambda B)
        \|\ab_i-\ab_j\|_{\Hb^{-1}}\\
        &\le(36\sqrt{L_\eta}+64\sqrt\lambda B)
        \|\ab_i-\ab_j\|_{\Hb^{-1}}.
    \end{align*}
    The second inequality uses $h\sqrt{L_\eta}\le1/16$.
    This proves (ii) and completes the proof.
\end{proof}

\subsubsection{Proof of Lemma~\ref{lemma:exploration_phase_length}}
\label{app:exploration_phase_length}

\begin{proof}[Proof of Lemma~\ref{lemma:exploration_phase_length}]
We first bound the warm-up exploration length, then establish
the localization of $\thetab_0$, and finally prove the
simultaneous confidence guarantee.

\textbf{Exploration length.}
Write $\zeta:=\WarmupThreshold$.
Each warm-up update contains an item whose squared
$\Vb_{t-1}^{-1}$-norm exceeds $\zeta^2$.
By the matrix determinant lemma, we have
\begin{align*}
\log\det\Vb_t-\log\det\Vb_{t-1}
&\ge\log(1+\zeta^2)\\
&\ge\frac{\zeta^2}{2}.
\tag{$\zeta^2<1$}
\end{align*}
Since $\Vb_0=I_d$ and $\tr\Vb_m\le d+mK$,
the arithmetic-geometric mean inequality gives
\begin{align*}
\frac{m\zeta^2}{2}
&\le\log\det\Vb_m-\log\det\Vb_0\\
&\le d\log\left(\frac{\tr\Vb_m}{d}\right)
\tag{AM--GM for eigenvalues}\\
&\le d\log\left(1+\frac{mK}{d}\right).
\end{align*}
Solving this inequality yields
\begin{align*}
m
&\le4d\zeta^{-2}
\log\left(1+\frac{4K}{\zeta^2}\right).
\end{align*}
Substituting the warm-up threshold specified in
Step~1, we obtain
\begin{align*}
T_0
&\le c_0\kappa^{-1}d
\left(
d\log\frac{16(N+1)^2}{\delta}
\vee B^2
\right)
\cdot
\log\left(
1+\frac{K}{\kappa}
\left[
d\log\frac{16(N+1)^2}{\delta}
\vee B^2
\right]
\right),
\end{align*}
for an absolute constant $c_0>0$.
This proves the exploration-length bound.

\textbf{Localization of the nominal parameter.}
The warm-up sequence and its length depend only on
features and independent randomization, not on either
choice stream.
At termination, we have
\begin{align*}
\Hb_{\Dcal_w'}(\thetab^\star)
&\succeq\kappa\Vb_{\WarmupRounds},\\
\max_{i\in[N]}
\|\ab_i\|_{\Hb_{\Dcal_w'}(\thetab^\star)^{-1}}
&\le\kappa^{-1/2}\zeta.
\tag{Warm-up stopping condition}
\end{align*}
Consequently,
\begin{align*}
\max_{0\le i<j\le N}
\|\ab_i-\ab_j\|_{
\Hb_{\Dcal_w'}(\thetab^\star)^{-1}}
&\le
2\max_{i\in[N]}
\|\ab_i\|_{
\Hb_{\Dcal_w'}(\thetab^\star)^{-1}}\\
&\le2\kappa^{-1/2}\zeta.
\tag{Triangle inequality}
\end{align*}
By the specified warm-up threshold,
the burn-in condition~\eqref{eq:pairwise_burnin}
holds with $\eta=\delta/4$.
Applying Lemma~\ref{lemma:estimation_error}
to $\Dcal_w'$ therefore gives
\begin{align*}
\max_{i\in[N]}
|\ab_i^\top(\thetab_0-\thetab^\star)|
&\le
\beta(\delta/4)
\max_{i\in[N]}
\|\ab_i\|_{
\Hb_{\Dcal_w'}(\thetab^\star)^{-1}}\\
&\le
\beta(\delta/4)\kappa^{-1/2}\zeta\\
&\le\frac{36+64}{1024}
<\frac14.
\tag{Warm-up threshold}
\end{align*}
Thus, the warm-up localization event satisfies
\[
\PP(\Ecal_{\mathrm w})\ge1-\frac{\delta}{4}.
\]

Next, for every $S\in\Scal$ and
$i\in S\cup\{0\}$, let
$p_\theta(i\mid S):=p(i\mid S,\theta).$
Since $\ab_0=0$, on $\Ecal_{\mathrm w}$, we have
\begin{align*}
&\max_{i,j\in S\cup\{0\}}
|(\ab_i-\ab_j)^\top
(\thetab^\star-\thetab_0)|
\le
2\max_{i\in[N]}
|\ab_i^\top(\thetab^\star-\thetab_0)|
\le\frac12.
\tag{Warm-up localization}
\end{align*}
By the MNL probability formula, this implies
\begin{align*}
e^{-1/2}
&\le
\frac{p(i\mid S,\thetab^\star)}
{p(i\mid S,\thetab_0)}
\le e^{1/2},
\qquad i\in S\cup\{0\}.
\tag{Bounded utility shifts}
\end{align*}
Applying Lemma~\ref{lemma:var_compare_bounded_lr}
to each feature direction gives
\begin{align*}
e^{-1/2}\Ib_{\thetab_0}(S)
&\preceq\Ib_{\thetab^\star}(S)
\preceq e^{1/2}\Ib_{\thetab_0}(S).
\end{align*}
Summing over the query sequence and using
$e^{-1/2}\lambda I_d\preceq\lambda I_d
\preceq e^{1/2}\lambda I_d$, we obtain
\begin{align*}
e^{-1/2}\Hb_t(\thetab_0)
&\preceq\Hb_t(\thetab^\star)
\preceq e^{1/2}\Hb_t(\thetab_0).
\tag{Definition of $\Hb_t$}
\end{align*}
These are precisely the two comparisons in
Equation~\eqref{eq:p_star_bound}.

\textbf{Confidence at the stopping time.}
Consider an infinite continuation of the main query
stream.
Condition on $\Dcal_w'$, the warm-up queries,
and the independent randomization generating all
subsequent queries.
This fixes $\thetab_0$, the design, and the entire
query sequence, while the unprimed responses remain
independent MNL choices.

On $\Ecal_{\mathrm w}$, the matrix comparison in
Equation~\eqref{eq:p_star_bound} gives
\begin{align*}
\max_{0\le i<j\le N}
\|\ab_i-\ab_j\|_{
\Hb_t(\thetab^\star)^{-1}}
&\le
e^{1/4}
\max_{0\le i<j\le N}
\|\ab_i-\ab_j\|_{
\Hb_t(\thetab_0)^{-1}}\\
&\le
2e^{1/4}
\max_{i\in[N]}
\|\ab_i\|_{\Hb_t(\thetab_0)^{-1}}\\
&\le
2\sqrt2
\max_{i\in[N]}
\|\ab_i\|_{\Hb_t(\thetab_0)^{-1}}.
\tag{$e^{1/4}\le\sqrt2$}
\end{align*}
At any time satisfying
Equation~\eqref{eq:confidence_check},
the preceding bound implies
\begin{align*}
\max_{0\le i<j\le N}
\|\ab_i-\ab_j\|_{
\Hb_t(\thetab^\star)^{-1}}
&\le
\frac{2\sqrt2}{256}
\frac{1}{
\sqrt{d\log9+L_t+\lambda B^2}}\\
&\le
\frac{1}{
16\sqrt{d\log9+L_t+\lambda B^2}}.
\tag{$2\sqrt2/256<1/16$}
\end{align*}
Thus, the burn-in condition
\eqref{eq:pairwise_burnin} holds with
$\eta=\delta_t$.

By Lemma~\ref{lemma:estimation_error},
the conditional failure probability of the
confidence event at each such time is at most
$\delta_t$.
Taking a union bound over all times gives
\begin{align*}
\PP(\Ecal^c)
&\le
\PP(\Ecal_{\mathrm w}^c)
+\sum_{t=0}^\infty\delta_t
\tag{Union bound}\\
&\le
\frac{\delta}{4}
+\sum_{t=0}^\infty
\frac{\delta}{8(t+1)^2}\\
&=
\frac{\delta}{4}
+\frac{\pi^2\delta}{48}\\
&<\frac{\delta}{2}.
\tag{$\pi^2<12$}
\end{align*}
Consequently, we obtain
\[
\PP(\Ecal)\ge1-\frac{\delta}{2}.
\]
Conditioning on the infinite query stream is used
only for the analysis; the algorithm observes only
the prefix before stopping.
This completes the proof.
\end{proof}

\subsubsection{Proof of Lemma~\ref{lemma:revenue_bound}}
\label{app_subsubsec:proof_of_lemma:revenue_bound}

\begin{proof}[Proof of Lemma~\ref{lemma:revenue_bound}]
Fix any nonempty $S\in\Scal$ and a time satisfying
\eqref{eq:confidence_check}.
For $\ub=(u_i)_{i\in S}$, write
\[
\Qcal_S(\ub):=
\frac{\sum_{i\in S}e^{u_i}r_i}
{1+\sum_{i\in S}e^{u_i}},
\qquad u_0=0,\quad r_0=0.
\]
We establish the bound for either choice of sign.
Let
$\Delta u_i:=u_{t,i}^{\pm}-\ab_i^\top\thetab^\star$,
including $\Delta u_0=0$, and define
$\ub_s:=\ub^\star+s(\ub_t^\pm-\ub^\star)$
for $0\le s\le1$.

\textbf{Step 1: Pairwise utility errors.}
On $\Ecal$, the pairwise confidence bound and
Equation~\eqref{eq:p_star_bound} imply
\begin{align*}
|\Delta u_i-\Delta u_j|
&\le
|(\ab_i-\ab_j)^\top
(\widehat{\thetab}_t-\thetab^\star)|
+\sqrt2\beta_t
\left|
\|\ab_i\|_{\Hb_t(\thetab_0)^{-1}}
-\|\ab_j\|_{\Hb_t(\thetab_0)^{-1}}
\right|\\
&\le
2\sqrt2\beta_t
\|\ab_i-\ab_j\|_{\Hb_t(\thetab_0)^{-1}}.
\numberthis\label{eq:revenue_pairwise_error}
\end{align*}
Here, the outside option is included by setting
$\ab_0=0$ and $\Delta u_0=0$.

In particular, taking $j=0$ gives
\begin{align*}
\max_{i\in S}|\Delta u_i|
&\le
2\sqrt2\beta_t
\max_{i\in S}
\|\ab_i\|_{\Hb_t(\thetab_0)^{-1}}\\
&\le
\frac{2\sqrt2\sqrt{36^2+64^2}}{256}
<1.
\tag{Equation~\eqref{eq:confidence_check}}
\end{align*}
Consequently, for every $s\in[0,1]$,
the utility shifts between $\ub_s$ and $\ub^\star$
have range at most $2$.
The MNL probability formula therefore gives
\begin{equation}
e^{-2}\le
\frac{p(i\mid S,\ub_s)}
{p(i\mid S,\thetab^\star)}
\le e^2,
\qquad i\in S\cup\{0\}.
\label{eq:p_bar_bound}
\end{equation}

\textbf{Step 2: Integrating the revenue derivative.}
Let $J$ have distribution $p(\cdot\mid S,\ub_s)$.
Differentiating $\Qcal_S$ along the segment gives
\begin{align*}
\frac{d}{ds}\Qcal_S(\ub_s)
&=\sum_{i\in S\cup\{0\}}
p(i\mid S,\ub_s)
\left(r_i-\EE_{\ub_s}[r_J]\right)\Delta u_i\\
&=\Cov_{\ub_s}(r_J,\Delta u_J).
\tag{MNL probability derivative}
\end{align*}
By Cauchy--Schwarz,
\begin{equation}
\left|\frac{d}{ds}\Qcal_S(\ub_s)\right|
\le
\sqrt{\Var_{\ub_s}(r_J)}
\sqrt{\Var_{\ub_s}(\Delta u_J)}.
\numberthis \label{eq:revenue_derivative_bound}
\end{equation}

Using the pairwise variance identity and
Equation~\eqref{eq:revenue_pairwise_error},
we obtain
\begin{align*}
\Var_{\ub_s}(\Delta u_J)
&=\sum_{\substack{i<j\\
i,j\in S\cup\{0\}}}
p(i\mid S,\ub_s)p(j\mid S,\ub_s)
(\Delta u_i-\Delta u_j)^2\\
&\le
8\beta_t^2
\sum_{\substack{i<j\\
i,j\in S\cup\{0\}}}
p(i\mid S,\ub_s)p(j\mid S,\ub_s)
\|\ab_i-\ab_j\|_{\Hb_t(\thetab_0)^{-1}}^2
\tag{Equation~\eqref{eq:revenue_pairwise_error}}\\
&=
8\beta_t^2
\tr\!\left(
\Hb_t(\thetab_0)^{-1}
\Cov_{\ub_s}(\ab_J)
\right).
\tag{Pairwise covariance identity}
\end{align*}

By Lemma~\ref{lemma:var_compare_bounded_lr}
and Equations~\eqref{eq:p_bar_bound}
and~\eqref{eq:p_star_bound},
\begin{align*}
\Var_{\ub_s}(r_J)
&\le e^2\Var(r_i\mid S,\thetab^\star),\\
\Cov_{\ub_s}(\ab_J)
&\preceq e^2\Ib_{\thetab^\star}(S)
\preceq e^{5/2}\Ib_{\thetab_0}(S).
\end{align*}
Substituting these bounds into
Equation~\eqref{eq:revenue_derivative_bound} gives
\begin{align*}
\left|\frac{d}{ds}\Qcal_S(\ub_s)\right|
&\le
\sqrt{e^2\Var(r_i\mid S,\thetab^\star)}
\sqrt{
8e^{5/2}\beta_t^2
\tr\!\left(
\Hb_t(\thetab_0)^{-1}
\Ib_{\thetab_0}(S)
\right)
}\\
&=
2\sqrt2 e^{9/4}\beta_t
\sqrt{\Var(r_i\mid S,\thetab^\star)}
\sqrt{
\tr\!\left(
\Hb_t(\thetab_0)^{-1}
\Ib_{\thetab_0}(S)
\right)
}.
\end{align*}

Integrating over $s\in[0,1]$ and using
$\Qcal_S(\ub^\star)=R(S,\thetab^\star)$,
we obtain
\begin{align*}
&|R(S,\thetab^\star)-\Qcal_S(\ub_t^\pm)|\\
&\quad=
\left|
\int_0^1
\frac{d}{ds}\Qcal_S(\ub_s)\,ds
\right|
\tag{Fundamental theorem of calculus}\\
&\quad\le
2\sqrt2 e^{9/4}\beta_t
\sqrt{\Var(r_i\mid S,\thetab^\star)}
\sqrt{
\tr\!\left(
\Hb_t(\thetab_0)^{-1}
\Ib_{\thetab_0}(S)
\right)
}\\
&\quad\le
2\sqrt2 e^{9/4}\beta_t
\sqrt{\Var(r_i\mid S,\thetab^\star)}
\sqrt{
\max_{S'\in\Scal}
\tr\!\left(
\Hb_t(\thetab_0)^{-1}
\Ib_{\thetab_0}(S')
\right)
}.
\end{align*}

\textbf{Step 3: Completing the revenue gap bounds.}
Since
\[
20\beta_t^2
\max_{i\in S}
\|\ab_i\|_{\Hb_t(\thetab_0)^{-1}}^2
\ge0,
\]
the preceding inequality further implies
\begin{align*}
&|R(S,\thetab^\star)-\Qcal_S(\ub_t^\pm)|\\
&\quad\le
2\sqrt2 e^{9/4}\beta_t
\sqrt{\Var(r_i\mid S,\thetab^\star)}
\sqrt{
\max_{S'\in\Scal}
\tr\!\left(
\Hb_t(\thetab_0)^{-1}
\Ib_{\thetab_0}(S')
\right)
}
+
20\beta_t^2
\max_{i\in S}
\|\ab_i\|_{\Hb_t(\thetab_0)^{-1}}^2\\
&\quad=\gamma_t(S).
\tag{Equation~\eqref{eq:revenue_radius}}
\end{align*}
Since the argument holds for either sign,
and $\Qcal_S(\ub_t^-)=\check R_t(S)$ and
$\Qcal_S(\ub_t^+)=\widetilde R_t(S)$,
we conclude that
\begin{align*}
|R(S,\thetab^\star)-\check R_t(S)|
&\le\gamma_t(S),\\
|R(S,\thetab^\star)-\widetilde R_t(S)|
&\le\gamma_t(S).
\end{align*}

No optimality property of $S$ was used, so both
inequalities hold simultaneously for every
$S\in\Scal$ on $\Ecal$.
For $S=\emptyset$, both revenues and
$\gamma_t(S)$ are zero.
This completes the proof.
\end{proof}

\subsubsection{Proof of Lemma~\ref{lemma:bound_var}}
\label{app_subsubsec:proof_of_lemma:bound_var}

\begin{proof}[Proof of Lemma~\ref{lemma:bound_var}]
The empty assortment has zero revenue variance.
Fix any nonempty $S\in\Scal$ and, for simplicity, write
$p_i=p(i\mid S,\thetab^\star)$.

For the non-uniform revenue parameters, since
$r_i\in[0,1]$ for all $i\in S\cup\{0\}$, we have
\begin{align*}
\Var(r_Y\mid S,\thetab^\star)
&=\EE[r_Y^2\mid S,\thetab^\star]
-\EE[r_Y\mid S,\thetab^\star]^2\\
&\le\EE[r_Y^2\mid S,\thetab^\star]\\
&\le\EE[r_Y\mid S,\thetab^\star]
\tag{$0\le r_i\le1$}\\
&=R(S,\thetab^\star)\\
&\le\max_{S'\in\Scal}R(S',\thetab^\star)\\
&=R(S^\star,\thetab^\star).
\tag{Definition of $S^\star$}
\end{align*}
Moreover, since $r_Y\in[0,1]$, we also have
\begin{align*}
\Var(r_Y\mid S,\thetab^\star)
&\le\frac14.
\tag{Variance bound for $[0,1]$-valued variables}
\end{align*}
Combining the two bounds gives
\[
\Var(r_Y\mid S,\thetab^\star)
\le\min\left\{
R(S^\star,\thetab^\star),\frac14
\right\},
\]
which proves claim~\textup{(i)}.

For the uniform revenue parameters, i.e.,
$r_i=1$ for all $i\in[N]$ and $r_0=0$,
we have
\begin{align*}
\Var(r_Y\mid S,\thetab^\star)
&=\sum_{i\in S\cup\{0\}}
p_i\left(r_i-\sum_{j\in S\cup\{0\}}p_jr_j\right)^2\\
&=\sum_{i\in S}p_i
\left(1-\sum_{j\in S}p_j\right)^2
+p_0\left(\sum_{j\in S}p_j\right)^2\\
&=\sum_{i\in S}p_ip_0^2
+p_0\left(\sum_{j\in S}p_j\right)^2
\tag{$\sum_{j\in S}p_j=1-p_0$}\\
&=p_0^2(1-p_0)+p_0(1-p_0)^2\\
&=p_0(1-p_0).
\end{align*}
Since $0\le p_0\le1$, we obtain
\begin{align*}
\Var(r_Y\mid S,\thetab^\star)
&=p_0(1-p_0)
\le\frac14.
\tag{Bernoulli variance bound}
\end{align*}
On the other hand,
\begin{align*}
\Var(r_Y\mid S,\thetab^\star)
&=p_0(1-p_0)\\
&\le p_0\\
&=\frac{1}{
1+\sum_{j\in S}\exp(a_j^\top\thetab^\star)}\\
&\le\frac{1}{1+|S|e^{-B}}
\tag{$\|a_j\|_2\le1$, $\|\thetab^\star\|_2\le B$}\\
&\le\frac{e^B}{|S|}.
\end{align*}
Combining these inequalities yields
\[
\Var(r_Y\mid S,\thetab^\star)
\le\min\left\{\frac14,\frac{e^B}{|S|}\right\},
\]
which proves claim~\textup{(ii)}.
\end{proof}

\subsubsection{Proof of Lemma~\ref{lemma:opt_design}}
\label{app_subsubsec:proof_of_lemma:opt_design}
\begin{proof}[Proof of Lemma~\ref{lemma:opt_design}]
    Condition on the nominal-data stream and the warm-up queries.
    The design is then fixed and the main assortments are independent.
    Put $n=t-\WarmupRounds$ and
    $\bar\lambda_t=24\log(2d/\delta_t)$.
    Since $\|\ab_i\|_2\le1$, including $\ab_0=0$, we have
    \[
        0\preceq\Ib_{\thetab_0}(S)\preceq\Ib_d,
        \qquad \tr\Ib_{\thetab_0}(S)\le1.
    \]
    If $n\ge24\log(2d/\delta_t)$,
    Lemma~\ref{lemma:psd_concentration} with $\gamma=1/2$ and $L=1$
    implies, with probability at least $1-\delta_t$,
    \begin{align*}
        \Hb_t(\thetab_0)
        &\succeq\frac1{\bar\lambda_t}
        \left[\sum_{s=\WarmupRounds+1}^t
        \Ib_{\thetab_0}(S_s)+\bar\lambda_t\Ib_d\right]\\
        &\succeq\frac1{2\bar\lambda_t}
        \left[n\Mb_{\thetab_0}(\widehat\pi_{\thetab_0})
        +\bar\lambda_t\Ib_d\right]\\
        &\succeq\frac{n}{2\bar\lambda_t}
        \Mb_{\thetab_0}(\widehat\pi_{\thetab_0}).
    \end{align*}
    Taking inverses and using the design guarantee yields (i).
    If $n<24\log(2d/\delta_t)$, (i) holds deterministically:
    its left-hand side is at most one because $\Hb_t(\thetab_0)\succeq\Ib_d$,
    whereas the stated right-hand side exceeds one.

    For (ii), on $\Ecal_{\mathrm w}$ the singleton offer $\{i\}$ gives
    \begin{align*}
        \|\ab_i\|_{\Hb_t(\thetab_0)^{-1}}^2
        &=\frac{\tr\!\left(\Hb_t(\thetab_0)^{-1}
        \Ib_{\thetab_0}(\{i\})\right)}
        {p(i\mid\{i\},\thetab_0)p(0\mid\{i\},\thetab_0)}\\
        &\le\frac e\kappa
        \max_{S\in\Scal}\tr\!\left(
        \Hb_t(\thetab_0)^{-1}\Ib_{\thetab_0}(S)\right).
        \tag{Equation~\eqref{eq:p_star_bound}}
    \end{align*}
    Apply (i) and take the maximum over $i$.
    Finally,
    $\sum_{t=\WarmupRounds+1}^\infty\delta_t
    \le\delta\pi^2/48<\delta/4$.
    A union bound proves the simultaneous statement.
\end{proof}

\subsubsection{Proof of Lemma~\ref{lemma:S_tau_on_stopping}}
\label{app_subsubsec:proof_of_lemma:S_tau_on_stopping}
\begin{proof}[Proof of Lemma~\ref{lemma:S_tau_on_stopping}]
    We prove the lemma by contradiction.
    Suppose that the algorithm stops at $\tau$ but
    $\BestS_\tau\ne S^\star$.
    The confidence check and the event $\Ecal$ imply
    $u^-_{\tau,i}\le\ab_i^\top\thetab^\star\le u^+_{\tau,i}$
    for every item.
    Thus Lemma~\ref{lemma:optimism} gives
    $R(S^\star,\thetab^\star)\le\widetilde R_\tau(S^\star)$.

    We also have
    $\widecheck R_\tau(\BestS_\tau)\le R(\BestS_\tau,\thetab^\star)$.
    To see this without assuming that revenue is globally monotone
    in utilities, let $R_-:=\widecheck R_\tau(\BestS_\tau)$.
    Optimality under the lower utilities and feasibility of deleting
    an item imply $r_i\ge R_-$ for every $i\in\BestS_\tau$.
    Since $e^{\ab_i^\top\thetab^\star}\ge e^{u^-_{\tau,i}}$,
    \begin{align*}
        &\sum_{i\in\BestS_\tau}r_i e^{\ab_i^\top\thetab^\star}
        -R_-\left(1+\sum_{i\in\BestS_\tau}
        e^{\ab_i^\top\thetab^\star}\right)\\
        &\quad=\sum_{i\in\BestS_\tau}
        (r_i-R_-)\left(e^{\ab_i^\top\thetab^\star}
        -e^{u^-_{\tau,i}}\right)\ge0.
    \end{align*}
    Consequently,
    \begin{align*}
        R(S^\star,\thetab^\star)
        &\le\widetilde R_\tau(S^\star)
        \le\widetilde R_\tau(\AltS_\tau)\\
        &<\widecheck R_\tau(\BestS_\tau)
        \le R(\BestS_\tau,\thetab^\star),
    \end{align*}
    contradicting optimality of $S^\star$.
    Notice that the item-deletion property was used only for
    the unrestricted pessimistic maximizer, not for $\AltS_\tau$.
\end{proof}

\subsection{Technical Lemmas}
\label{app_subsec:technical_lemmas_for_BAI}
%
\begin{lemma} [Lemma 4 of~\citealt{oh2021multinomial}] \label{lemma:optimism}
Let 
$\tilde{R}_{t}(S) = \frac{\sum_{i \in S} \exp( u^{+}_{t,i} ) r_i }{1 + \sum_{j \in S} \exp(u^{+}_{t,j}) }$.
If for every item $i \in S^\star$, $u^{+}_{t,i} \geq a_i^\top \thetab^\star$, then for all $t \geq 1$, the following inequalities hold:
    \begin{align*}
        R(S^\star, \thetab^\star) \leq \tilde{R}_t(S^\star).
    \end{align*}
\end{lemma}
\begin{lemma} [Lemma E.3 of~\citealt{lee2024nearly}]
\label{lemma:revenue_second_pd}
    Define $\Qcal:\RR^K \rightarrow \RR$, such that for any $\ub = (u_1, \dots, u_K) \in \RR^K$, $\Qcal(\ub) = \sum_{i=1}^K \frac{\exp(u_i)}{1 + \sum_{k=1}^K \exp(u_k)}$.
    Let $p_i(\ub) = \frac{\exp(u_i)}{1 + \sum_{k=1}^K \exp(u_k)}$.
    Then, for all $i \in [K]$, we have
    \begin{align*}
        \left| \frac{\partial^2 \Qcal}{\partial i \partial j} \right|
        \leq
        \begin{cases}
            3 p_i(\ub) & \text{if} \,\,\, i=j,
            \\
            2p_i(\ub) p_j(\ub) & \text{if} \,\,\, i \neq j.
        \end{cases}
    \end{align*}
\end{lemma}
\begin{lemma} [Proposition B.5 of~\citealt{lee2025improved}]
\label{lemma:self_hessian_norm}
For any $t \in [T]$, 
the Hessian of the multinomial logistic loss $\ell_t: \RR^{d} \rightarrow \RR$ satisfies the following for any $\thetab_1, \thetab_2 \in \RR^{d}$:
\begin{align*}
    e^{-3\sqrt{2} \max_{i \in S_t} \| \ab_i^\top (\thetab_1 - \thetab_2) \|_\infty }  \nabla^2 \ell_t  (\thetab_1)
    \preceq \nabla^2  \ell_t (\thetab_2)
    \preceq e^{3\sqrt{2}\max_{i \in S_t} \| \ab_i^\top (\thetab_1 - \thetab_2) \|_\infty } \nabla^2 \ell_t (\thetab_1).
\end{align*}
\end{lemma}

\begin{lemma} [Proposition B.6 of~\citealt{lee2025improved}]
\label{lemma:hessian_usedful}
    For any $t \in [T]$, 
    the Hessian of the multinomial logistic loss $\ell_t: \RR^{d} \rightarrow \RR$ satisfies the following for any $\ub, \thetab_1, \thetab_2 \in \RR^{d}$:
    \begin{align*}
        \ub^\top 
        \left(\int_0^1 (1-s) \nabla^2 \ell_t (\thetab_1 + s (\thetab_2 - \thetab_1) ) \dd s  \right)
        \ub
        \geq \frac{1}{ 2 + 3\sqrt{2} 
        \max_{i \in S_t} \| \ab_i^\top (\thetab_1 - \thetab_2) \|_\infty } \ub^\top \nabla^2 \ell_t (\thetab_1) \ub.
    \end{align*}
\end{lemma}

\begin{lemma}[Matrix Chernoff bound for independent PSD matrices]
\label{lemma:psd_concentration}
    Let $\{X_s\}_{s=1}^t$ be independent random matrices in $\mathbb{R}^{d\times d}$ such that
    $X_s \succeq \mathbf{0}$
    and
    $\|X_s\| \le L$
    almost surely
    for some constant $L>0$.
    Define
    $\widehat{\Hb}_t := \sum_{s=1}^t X_s$ and
    $\Hb_t := \EE [\widehat{\Hb}_t]
    $.
    Fix any $\gamma\in(0,1)$ and $\delta\in(0,1)$.
    Assume that $t \ge \frac{6}{\gamma^2}\log\frac{2d}{\delta}$ and choose
    $\bar{\lambda} = \frac{6L}{\gamma^2}\log\frac{2d}{\delta}$.
    Then, with probability at least $1-\delta$, we have 
    \[
    (1-\gamma)\big(\Hb_t+\bar{\lambda}\Ib_d\big)
    \;\preceq\;
    \widehat{\Hb}_t+\bar{\lambda}\Ib_d
    \;\preceq\;
    (1+\gamma)\big(\Hb_t+\bar{\lambda}\Ib_d\big).
    \]
\end{lemma}

\begin{proof}[Proof of Lemma~\ref{lemma:psd_concentration}]
We prove the lemma by reducing the problem to a standard matrix Chernoff bound
for sums of independent positive semidefinite (PSD) matrices with eigenvalues in $[0,1]$.

Define
\[
X_s^{\bar{\lambda}} := X_s + \frac{\bar{\lambda}}{t}\Ib_d,
\qquad
Y := \sum_{s=1}^t X_s^{\bar{\lambda}}
= \widehat{\Hb}_t+\bar{\lambda}\Ib_d,
\qquad
\mu := \mathbb{E}[Y]
= \Hb_t+\bar{\lambda}\Ib_d.
\]
Then $X_s^{\bar{\lambda}}\succeq \mathbf{0}$ and $\mu\succeq \bar{\lambda}\Ib_d$, so $\lambda_{\min}(\mu)\ge \bar{\lambda}$.
Moreover,
\[
\|X_s^{\bar{\lambda}}\| 
\le
\|X_s\| +\bar{\lambda}/t
\le
L+\bar{\lambda}/t
\quad\text{a.s.}
\]
Let
$R_{\max} := L+\bar{\lambda}/t$, a deterministic bound on all $\|X_s^{\bar{\lambda}}\|$.

Since $\mu\succeq \bar{\lambda}\Ib_d$, we have $\Ib_d \preceq \mu/\bar{\lambda}$. Hence, for every $s$,
\[
\mathbf{0}\preceq X_s^{\bar{\lambda}} \preceq \|X_s^{\bar{\lambda}}\| \Ib_d \preceq R_{\max}\Ib_d \preceq \frac{R_{\max}}{\bar{\lambda}}\,\mu.
\]
Define the normalized matrices
$Z_s
:=
\frac{\bar{\lambda}}{R_{\max}}\,\mu^{-1/2}X_s^{\bar{\lambda}}\mu^{-1/2}.$
Then $Z_s\succeq \mathbf{0}$ and, using $X_s^{\bar{\lambda}} \preceq (R_{\max}/\bar{\lambda})\mu$,
\[
Z_s
\preceq
\frac{\bar{\lambda}}{R_{\max}}\,\mu^{-1/2}\Big(\frac{R_{\max}}{\bar{\lambda}}\mu\Big)\mu^{-1/2}
=
\Ib_d,
\]
so $0\preceq Z_s\preceq \Ib_d$.
Furthermore,
\[
\mathbb{E}\left[\sum_{s=1}^t Z_s\right]
=
\frac{\bar{\lambda}}{R_{\max}}\,\mu^{-1/2}\mathbb{E}\left[\sum_{s=1}^t X_s^{\bar{\lambda}}\right]\mu^{-1/2}
=
\frac{\bar{\lambda}}{R_{\max}}\,\Ib_d.
\]
This implies 
\[
\lambda_{\min}\!\left(\mathbb{E}\left[\sum_{s=1}^t Z_s\right]\right)
=
\lambda_{\max}\!\left(\mathbb{E}\left[\sum_{s=1}^t Z_s\right]\right)
=
\frac{\bar{\lambda}}{R_{\max}}.
\]

Applying the standard matrix Chernoff bound to the independent matrices $\{Z_s\}$ with
$0\preceq Z_s\preceq \Ib_d$, for any $\gamma\in(0,1)$ we have
\begin{align}
    \mathbb{P}\!\left(\lambda_{\min}\!\left(\sum_{s=1}^t Z_s\right)\le (1-\gamma)\frac{\bar{\lambda}}{R_{\max}}\right)
    &\le
    d\cdot \exp\!\left(-\frac{\gamma^2}{2}\cdot\frac{\bar{\lambda}}{R_{\max}}\right),
    \label{eq:chernoff_lower_Z}\\
    \mathbb{P}\!\left(\lambda_{\max}\!\left(\sum_{s=1}^t Z_s\right)\ge (1+\gamma)\frac{\bar{\lambda}}{R_{\max}}\right)
    &\le
    d\cdot \exp\!\left(-\frac{\gamma^2}{3}\cdot\frac{\bar{\lambda}}{R_{\max}}\right).
    \label{eq:chernoff_upper_Z}
\end{align}
Therefore, if
\begin{equation}
    \label{eq:need_beta_over_R}
    \frac{\bar{\lambda}}{R_{\max}}\geq \frac{3}{\gamma^2}\log\frac{2d}{\delta},
\end{equation}
then each of \eqref{eq:chernoff_lower_Z}--\eqref{eq:chernoff_upper_Z} is at most $\delta/2$,
and by a union bound we obtain, with probability at least $1-\delta$,
\begin{equation}
\label{eq:Z_sandwich}
    (1-\gamma)\frac{\bar{\lambda}}{R_{\max}}\Ib_d
    \;\preceq\;
    \sum_{s=1}^t Z_s
    \;\preceq\;
    (1+\gamma)\frac{\bar{\lambda}}{R_{\max}}\Ib_d.
\end{equation}

Multiply \eqref{eq:Z_sandwich} on both sides by $(R_{\max}/\bar{\lambda})\mu^{1/2}(\cdot)\mu^{1/2}$ gives
\[
(1-\gamma)\mu
\;\preceq\;
\sum_{s=1}^t X_s^{\bar{\lambda}}
\;\preceq\;
(1+\gamma)\mu.
\]
Recalling $\sum_{s=1}^t X_s^{\bar{\lambda}}=\widehat{\Hb}_t+\bar{\lambda}\Ib_d$ and $\mu=\Hb_t+\bar{\lambda}\Ib_d$
proves the statement provided \eqref{eq:need_beta_over_R} holds.

We now derive a sufficient condition on $t$ under which \eqref{eq:need_beta_over_R} holds.
Since $R_{\max} \le L + \bar{\lambda}/t$, it is sufficient to guarantee that
\[
    \frac{\bar{\lambda}}{R_{\max}}
    =
    \frac{\bar{\lambda}}{L + \bar{\lambda}/t} \;\ge\; \frac{3}{\gamma^2}\log\frac{2d}{\delta} .
\]
Choosing $\bar{\lambda} = \frac{6L}{\gamma^2}\log\frac{2d}{\delta}$ and assuming $t \geq \frac{6}{\gamma^2}\log\frac{2d}{\delta}$, we obtain
\[
    L + \frac{\bar{\lambda}}{t}
    =
    L + \frac{\frac{6L}{\gamma^2}\log\frac{2d}{\delta}}{t}
    \;\le\;
    L + \frac{\frac{6L}{\gamma^2}\log\frac{2d}{\delta}}{\frac{6}{\gamma^2}\log\frac{2d}{\delta}}
    =
    2L .
\]
Therefore,
\[
    \frac{\bar{\lambda}}{R_{\max}}
    \;\ge\;
    \frac{\frac{6L}{\gamma^2}\log\frac{2d}{\delta}}{2L}
    =
    \frac{3}{\gamma^2}\log\frac{2d}{\delta},
\]
which establishes the desired condition.
\end{proof}
\begin{lemma}[Variance comparison under bounded likelihood ratios]
\label{lemma:var_compare_bounded_lr}
    Let $p=(p_i)_{i\in\Omega}$ and $q=(q_i)_{i\in\Omega}$ be probability mass functions on a finite set $\Omega$,
    and assume $q_i>0$ for all $i$.
    Fix $c\ge 1$ and suppose that
    \[
        \frac{1}{c}\;\le\;\frac{p_i}{q_i}\;\le\;c
        \qquad\text{for all } i\in\Omega .
    \]
    Then for any real-valued random variable $X:\Omega\to\mathbb{R}$,
    \[
        \Var\!\!\,_{p}(X)\;\le\;c\,\Var\!\!\,_{q}(X).
    \]
\end{lemma}
\begin{proof} [Proof of Lemma~\ref{lemma:var_compare_bounded_lr}]
    Write $\eta_i := \frac{p_i}{q_i}$, so that $\eta_i\in[1/c,c]$ and $p_i = q_i \eta_i$.
    Recall the variational characterization of the variance:
    for any distribution $\nu$ on $\Omega$,
    \begin{equation}
    \label{eq:var_as_min}
        \Var\!\!\,_{\nu}(X)
        \;=\;
        \min_{a\in\mathbb{R}} \; \EE_{\nu}\big[(X-a)^2\big].
    \end{equation}
    (Indeed, $\EE_{\nu}[(X-a)^2] = \Var\!\!\,_{\nu}(X) + (a-\EE_{\nu}[X])^2$, minimized at $a=\EE_{\nu}[X]$.)
    
    Fix any $a\in\mathbb{R}$. Using $p_i=q_i \eta_i$ and $\eta_i\le c$,
    \begin{align*}
        \EE_p\big[(X-a)^2\big]
        &\;=\
        \sum_{i\in\Omega} p_i (X_i-a)^2
        \;=\;
        \sum_{i\in\Omega} q_i \eta_i (X_i-a)^2
        \;\le\;
        c \sum_{i\in\Omega} q_i (X_i-a)^2
        \\
        &\;=\;
        c\,\EE_q\big[(X-a)^2\big].
    \end{align*}
    Taking $\min_{a\in\mathbb{R}}$ on both sides and using \eqref{eq:var_as_min} yields
    \begin{align*}
        \Var\!\!\,_p(X)
        \;=\;
        \min_a \EE_p[(X-a)^2]
        \;\le\;
        c \min_a \EE_q[(X-a)^2]
        \;=\;
        c\,\Var\!\!\,_q(X).
    \end{align*}
    This completes the proof.
\end{proof}


\end{document}